%% file: main.tex
\documentclass[sn-mathphys,iicol]{sn-jnl}

\pdfoutput=1
 
\usepackage{pgfplots}
\pgfplotsset{compat=1.17}
\usetikzlibrary{pgfplots.colorbrewer}

\usepackage{svg}

\usepackage{program}

\usepackage[noabbrev]{cleveref}
\usepackage{subcaption}

\renewcommand{\figwidth}{0.30\textwidth}
\newcommand{\tikzfigwidth}{1\textwidth}

\theoremstyle{thmstyleone}
\newtheorem{theorem}{Theorem}
\newtheorem{proposition}{Proposition} 
\newtheorem{lemma}{Lemma}
\newtheorem{corollary}{Corollary}

\theoremstyle{thmstyletwo}

\newtheorem{remark}{Remark}

\theoremstyle{thmstylethree}
\newtheorem{definition}{Definition}

\input{commands}

\begin{document}
    
    \title[(sub)Riemannian PDE-G-CNNs]{Analysis of (sub-)Riemannian PDE-G-CNNs}
    
    \author*[]{\fnm{Gijs} \sur{Bellaard}}\email{g.bellaard@tue.nl}
    \author{\fnm{Daan L. J.} \sur{Bon}}\email{d.l.j.bon@tue.nl}
    \author{\fnm{Gautam} \sur{Pai}}\email{g.pai@tue.nl}
    \author{\fnm{Bart M. N.} \sur{Smets}}\email{b.m.n.smets@tue.nl}
    \author{\fnm{Remco} \sur{Duits}}\email{r.duits@tue.nl}
    
    \affil{
        \orgdiv{Department of Mathematics and Computer Science}, 
        \orgname{CASA, Eindhoven University of Technology}, 
        \orgaddress{
            \city{Eindhoven}, 
            \country{The Netherlands}
        }
    }
    \keywords{Convolutional neural networks, Scale space theory, Geometric deep learning, Morphological convolutions, PDEs, Riemannian Geometry, sub-Riemannian Geometry}
    
    \abstract{
        \input{sections/abstract}
    }
    
    \maketitle           
    
    \input{sections/introduction}
    
    \input{sections/preliminaries}

    \input{sections/morphological}
    
    \input{sections/distance}
    
    \input{sections/analysis}
    
    \input{sections/experiments}
    
    \input{sections/conclusion}

    \section*{Acknowledgements}
    
        We thank Dr. Javier Oliván Bescós for pointing us to the publicly available DCA1 dataset \cite{sanchez2019segmentation}. 
    
    \section*{Declarations}
        
        \subsection*{Availability of Data and Code}
            The code of the experiments, and PDE-G-CNNs in general, can be found in the publicly available LieTorch package: \url{https://gitlab.com/bsmetsjr/lietorch}.
            
            The publicly available DCA1 dataset \cite{sanchez2019segmentation} can be found at \url{http://personal.cimat.mx:8181/~ivan.cruz/DB_Angiograms.html}.
            
            The lines dataset is available from the authors on request.
        
        \subsection*{Funding}
            We gratefully acknowledge the Dutch Foundation of Science NWO for its financial support by Talent Programme VICI 2020 Exact Sciences (Duits, Geometric learning for Image Analysis, VI.C. 202-031).

    \begin{appendices}
    
        \input{sections/appendix}
    
    \end{appendices}

    \bibliography{literature}

\end{document}

%% file: commands.tex
    \newcommand{\Bra}[1]{\left[ #1 \right]}
    \newcommand{\Par}[1]{\left( #1 \right)}
    \newcommand{\Abs}[1]{\left\vert #1 \right\vert}
    \newcommand{\Ang}[1]{\left\langle #1 \right\rangle}
    \newcommand{\At}[1]{\left. #1 \right\vert}
    \newcommand{\Norm}[1]{\left\Vert #1 \right\Vert}

\renewcommand{\iff}{\Leftrightarrow}

\let\a\alpha
\let\g\gamma
\let\e\varepsilon
\let\pd\partial

\newcommand{\bbR}{\mathbb{R}}
\newcommand{\bbM}{\mathbb{M}}

\newcommand{\cG}{\mathcal{G}}
\newcommand{\cA}{\mathcal{A}}
\newcommand{\cL}{\mathcal{L}}
\newcommand{\cH}{\mathcal{H}}

\newcommand{\cO}{\mathcal{O}}

\newcommand{\bp}{\mathbf{p}}
\newcommand{\bx}{\mathbf{x}}
\newcommand{\by}{\mathbf{y}}
\newcommand{\bR}{\mathbf{R}}
\newcommand{\bn}{\mathbf{n}}

\newcommand{\bq}{\mathbf{q}}

\DeclareMathOperator{\sinc}{sinc}
\DeclareMathOperator{\arctantwo}{arctan2}
\DeclareMathOperator{\omod}{mod}

\tikzset{            
    declare function = {
        l(\x,\y,\t,\wm,\wl,\wa) = sqrt((min(\wm,\wl)*\x)^2 + (min(\wm,\wl)*\y)^2 + (\wa*\t)^2);
        u1(\x,\y,\t,\wm,\wl,\wa) = sqrt((max(\wm,\wl)*\x)^2 + (max(\wm,\wl)*\y)^2 + (\wa*\t)^2);
        u2(\x,\y,\t,\wm,\wl,\wa) = sqrt((min(\wm,\wl)*\x)^2 + (min(\wm,\wl)*\y)^2) + \wa*pi;
    }
}

\newcommand\scalemath[2]{\scalebox{#1}{\mbox{\ensuremath{\displaystyle #2}}}}

\newcommand{\pdv}[2]{\frac{\partial #1}{\partial #2}}

%% file: sections/abstract.tex
Group equivariant convolutional neural networks (G-CNNs) have been successfully applied in geometric deep learning. Typically, G-CNNs have the advantage over CNNs that they do not waste network capacity on training symmetries that should have been hard-coded in the network. The recently introduced framework of PDE-based G-CNNs (PDE-G-CNNs) generalises G-CNNs. PDE-G-CNNs have the core advantages that they simultaneously 1) reduce network complexity, 2) increase classification performance, and 3) provide geometric interpretability. Their implementations primarily consist of linear and morphological convolutions with kernels. 
\\[2pt]
In this paper we show that the previously suggested approximative morphological kernels do not always accurately approximate the exact kernels accurately. More specifically, depending on the spatial anisotropy of the Riemannian metric, we argue that one must resort to sub-Riemannian approximations. We solve this problem by providing a new approximative kernel that works regardless of the anisotropy. We provide new theorems with better error estimates of the approximative kernels, and prove that they all carry the same reflectional symmetries as the exact ones.
\\[2pt]
We test the effectiveness of multiple approximative kernels within the PDE-G-CNN framework on two datasets, and observe an improvement with the new approximative kernels. We report that the PDE-G-CNNs again allow for a considerable reduction of network complexity while having comparable or better performance than G-CNNs and CNNs on the two datasets. Moreover, PDE-G-CNNs have the advantage of better geometric interpretability over G-CNNs, as the morphological kernels are related to association fields from neurogeometry.

%% file: sections/introduction.tex
\section{Introduction} \label{sec:introduction}

Many classification, segmentation, and tracking tasks in computer vision and digital image processing require some form of ``symmetry''. Think, for example, of image classification. If one rotates, reflects, or translates an image the classification stays the same. We say that an ideal image classification is \textit{invariant} under these symmetries. A slightly different situation is image segmentation. In this case, if the input image is in some way changed the output should change accordingly. Therefore, an ideal image segmentation is \textit{equivariant} with respect to these symmetries.

Many computer vision and image processing problems are currently being tackled with \textit{neural networks} (NNs). It is desirable to design neural networks in such a way that they respect the symmetries of the problem, i.e. make them invariant or equivariant. Think for example of a neural network that detects cancer cells. It would be disastrous if, by for example slightly translating an image, the neural network would give totally different diagnosis, even though the input is essentially the same.

One way to make the networks equivariant or invariant is to simply train them on more data. One could take the training dataset and augment it with translated, rotated and reflected versions of the original images. This approach however is undesirable: invariance or equivariance is still not guaranteed and the training takes longer. It would be better if the networks are \textit{inherently} invariant or equivariant by design. This avoids a waste of network-capacity, guarantees invariance or equivariance, and increases performances, see for example \cite{bekkers2018roto}.   

More specifically, many computer vision and image processing problems are tackled with \textit{convolutional neural networks} (CNNs) \cite{lecun1989backpropagation,krizhevsky2012imagenet,litjens2017survey}. Convolution neural networks have the property that they inherently respect, to some degree, translation symmetries. CNNs do not however take into account rotational or reflection symmetries. Cohen and Welling introduced \textit{group equivariant convolutional neural networks} (G-CNNs) in \cite{cohen2016group} and designed a classification network that is inherently invariant under 90 degree rotations, integer translations and vertical/horizontal reflections. Much work is being done on invariant/equivariant networks that exploit inherent symmetries, a non-exhaustive list is \cite{bekkers2018roto,dieleman2016exploiting,dieleman2015rotation,winkels20183d,worrall2018cubenet,oyallon2015deep,weiler2018learning,bekkers2019bspline,finzi2020generalizing,cohen2019general,worrall2017harmonic,kondor2018generalization,esteves2018learning,weiler2019general,paoletti2020rotation,weiler2021coordinate,cohen2019gauge,bogatskiy2020lorentz,sifre2013rotation,bekkers2018template,worrall2019deep,satorras2021equivariant}. The idea of including geometric priors, such as symmetries, into the design of neural networks is called `Geometric Deep Learning' in \cite{bronstein2021geometric}.


In \cite{smets2022pdebased} partial differential equation (PDE) based G-CNNs are presented, aptly called PDE-G-CNNs. In fact, G-CNNs are shown to be \textit{a special case} of PDE-G-CNNs (if one restricts the PDE-G-CNNs only to convection, using many transport vectors \cite[Sec.6]{smets2022pdebased}). With PDE-G-CNNs the usual non-linearities that are present in current networks, such as the ReLU activation function and max-pooling, are replaced by solvers for specifically chosen non-linear evolution PDEs. \Cref{fig:difference_CNNs_and_PDEGCNNS} illustrates the difference between a traditional CNN layer and a PDE-G-CNN layer.

\begin{figure*}
    \centering
    \includegraphics[width=0.8\linewidth]{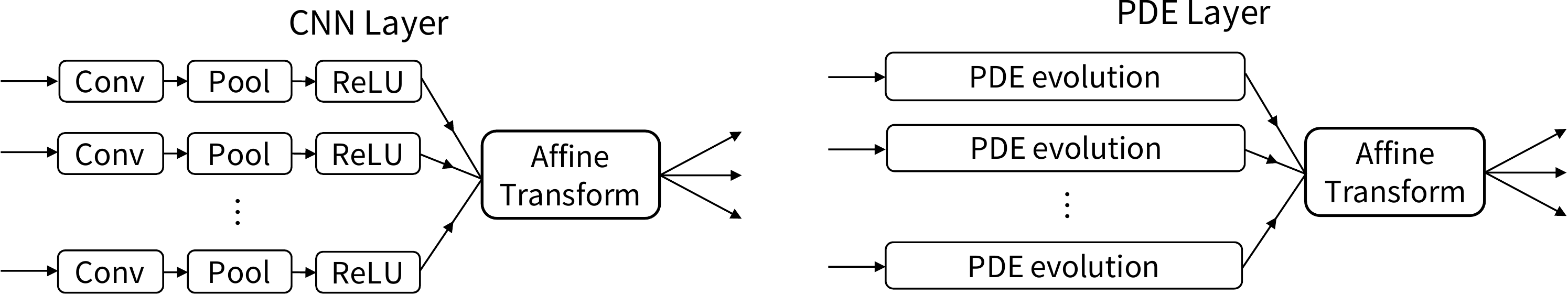}
    \caption{The difference between a traditional CNN layer and a PDE-G-CNN layer. In contrast to traditional CNNs, the layers in a PDE-G-CNN do not depend on ad-hoc non-linearities like ReLU's, and are instead implemented as solvers of (non)linear PDEs. What the PDE evolution block consists of can be seen in \Cref{fig:pde_evolution}.}
    \label{fig:difference_CNNs_and_PDEGCNNS}
\end{figure*}

The PDEs that are used in PDE-G-CNNs are not chosen arbitrarily: they come directly from the world of geometric image analysis, and thus their effects are geometrically interpretable. This makes PDE-G-CNNs more geometrically meaningful and interpretable than traditional CNNs. Specifically, the PDEs considered are diffusion, convection, dilation and erosion. These 4 PDEs correspond to the common notions of smoothing, shifting, max pooling, and min pooling. They are solved by linear convolutions, resamplings, and so-called \textit{ morphological convolutions}. \Cref{fig:pde_evolution} illustrates the basic building block of a PDE-G-CNN. 

\begin{figure}
    \centering
    \includegraphics[width=1\linewidth]{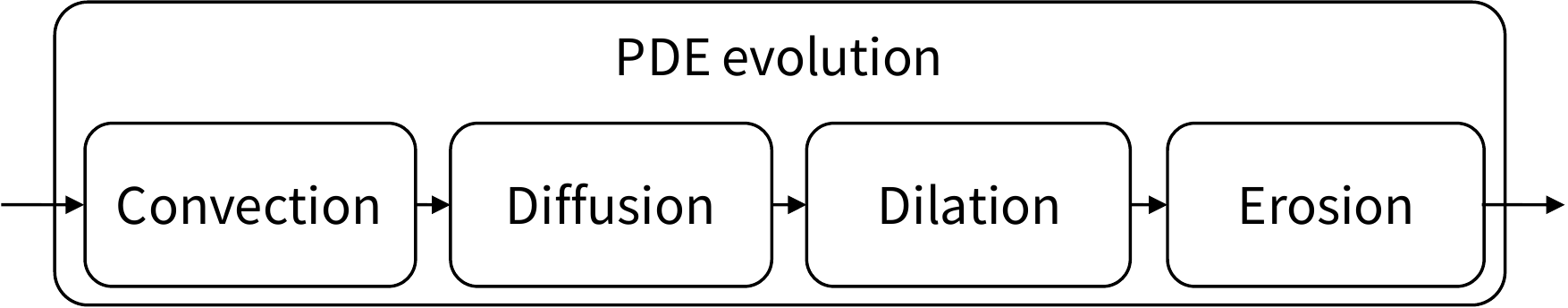}
    \caption{Overview of a PDE evolution block. Convection is solved by resampling, diffusion is solved by a linear group convolution with a certain kernel \protect\cite[Sec.5.2]{smets2022pdebased}, and dilation and erosion are solved by morphological group convolutions \eqref{eq:morphological_convolution} with a morphological kernel \eqref{eq:morphological_kernel_intro}. 
    }
    \label{fig:pde_evolution}
\end{figure}

One shared property of G-CNNs and PDE-G-CNNs is that the input data usually needs to be \textit{lifted} to a higher dimensional space. Take, for example, the case of image segmentation with a convolution neural network where we model/idealize the images as real-valued function on \(\bbR^2\). If we keep the data as functions on \(\bbR^2\) and want the convolutions within the network to be equivariant, then the only possible ones that are allowed are with isotropic kernels, \cite[p.258]{duits2010leftinvariant}. This type of short-coming generalizes to other symmetry groups as well \cite[Thm.1]{bekkers2019bspline}. One can imagine that this is a constraint too restrictive to work with, and that is why we lift the image data.

Within the PDE-G-CNN framework the input images are considered real-valued functions on \(\bbR^d\), the desired symmetries are represented by the Lie group of roto-translations \(SE(d)\), and the data is lifted to the homogeneous space of \(d\) dimensional positions and orientations \(\bbM_d\). It is on this higher dimensional space on which the evolution PDEs are defined, and the effects of diffusion, dilation, and erosion are completely determined by the Riemannian metric tensor field \(\cG\) that is chosen on \(\bbM_d\). If this Riemannian metric tensor field $\mathcal{G}$ is left-invariant, the overall processing is equivariant, this follows by combining techniques in~\cite[Thm.~21, Chpt.~4]{duits2005perceptual}, \cite[Lem.~3, Thm.~4]{duits2013morphological}. 

The Riemannian metric tensor field \(\cG\) we will use in this article is left-invariant and determined by three nonnegative parameters: \(w_1\), \(w_2\), and \(w_3\). The definition can be found in the preliminaries, \Cref{sec:preliminaries} \Cref{eq:diagonal_metric}. It is exactly these three parameters that during the training of a PDE-G-CNN are optimized. Intuitively, the parameters correspondingly regulate the cost of main spatial, lateral spatial, and angular motion. An important quantity in the analysis of this paper is the \textit{spatial anisotropy} \(\zeta := \frac{w_1}{w_2}\), as will become clear later. 

In this article we only consider the 2 dimensional case, i.e. \(d=2\). In this case, the elements of both \(\bbM_2\) and \(SE(2)\) can be represented by three real numbers: \((x,y,\theta) \in \bbR^2 \times [0,2\pi)\). In the case of \(\bbM_2\) the \(x\) and \(y\) represent a position and \(\theta\) represents an orientation. Throughout the article we take \(\bp_0 := (0,0,0) \in \bbM_2\) as our \textit{reference point} in \(\bbM_2\). In the case of \(SE(2)\) we have that \(x\) and \(y\) represent a translation and \(\theta\) a rotation. 

As already stated, within the PDE-G-CNN framework images are lifted to the higher dimensional space of positions and orientations \(\bbM_d\). There are a multitude of ways of achieving this, but there is one very natural way to do it: the \textit{orientation score transform} \cite{duits2005perceptual,duits2007scale,franken2008enhancement,bekkers2017retinal}. In this transform we pick a point \((x,y) \in \bbR^2\) in an image and determine how good a certain orientation \(\theta \in [0, 2\pi)\) fits the chosen point. In \Cref{fig:orientation_scores} an example of an orientation score is given. We refer to \cite[Sec.2.1]{bekkers2017retinal} for a summary of how an orientation score transform works.

\begin{figure}
    \centering
    \includegraphics[width=0.8\linewidth]{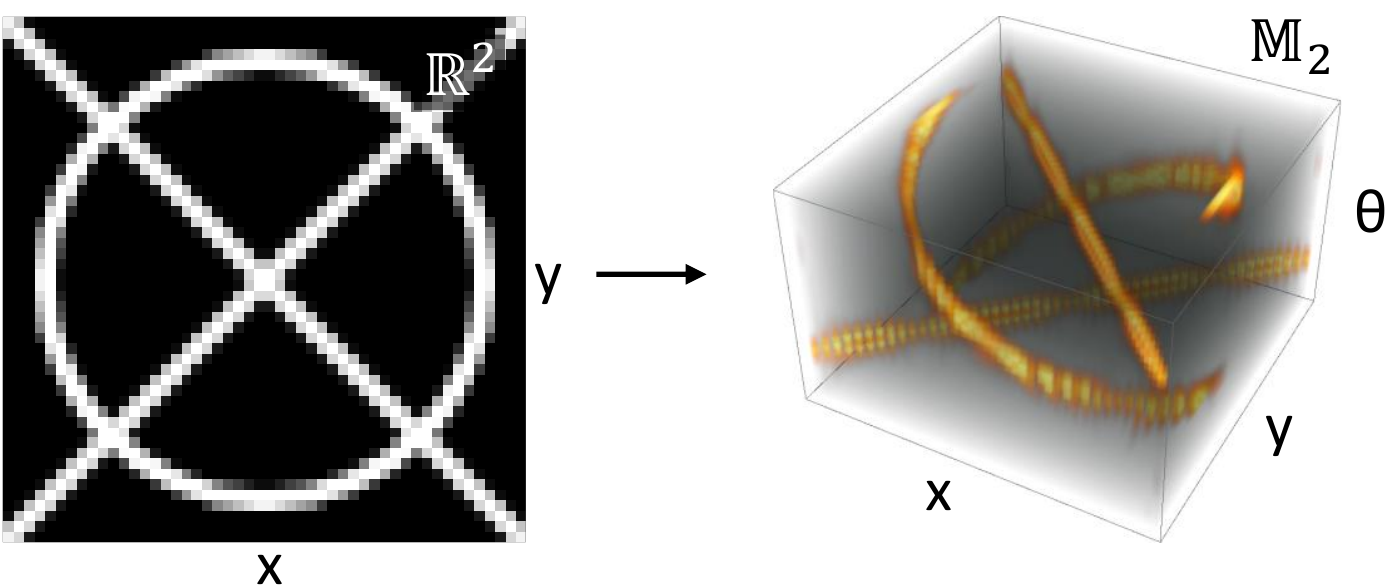}
    \caption{An example of an image together with its orientation score. We can see that the image, a real-valued function on \(\bbR^2\), is lifted to an orientation score, a real-valued function on \(\bbM_2\). Notice that the lines that are crossing in the left image are disentangled in the orientation score.}
    \label{fig:orientation_scores}
\end{figure}

Inspiration for using orientation scores comes from biology. The Nobel laureates Hubel and Wiesel found that many cells in the visual cortex of cats have a preferred orientation \cite{hubel1959receptive,bosking1997orientation}. Moreover, a neuron that fires for a specific orientation excites neighboring neurons that have an ``aligned'' orientation. Petitot and Citti-Sarti proposed a model \cite{petitot2003neurogeometry,citti2006cortical} for the distribution of the orientation preference and this excitation of neighbors based on sub-Riemannian geometry on \(\bbM_2\). They relate the phenomenon of preference of aligned orientations to the concept of \textit{association fields} \cite{field1993contour}, which model how a specific local orientation places expectations on surrounding orientations in human vision. \Cref{fig:association_field} provides an impression of such an association field.

\begin{figure}
    \centering
    \includegraphics[width=0.8\linewidth]{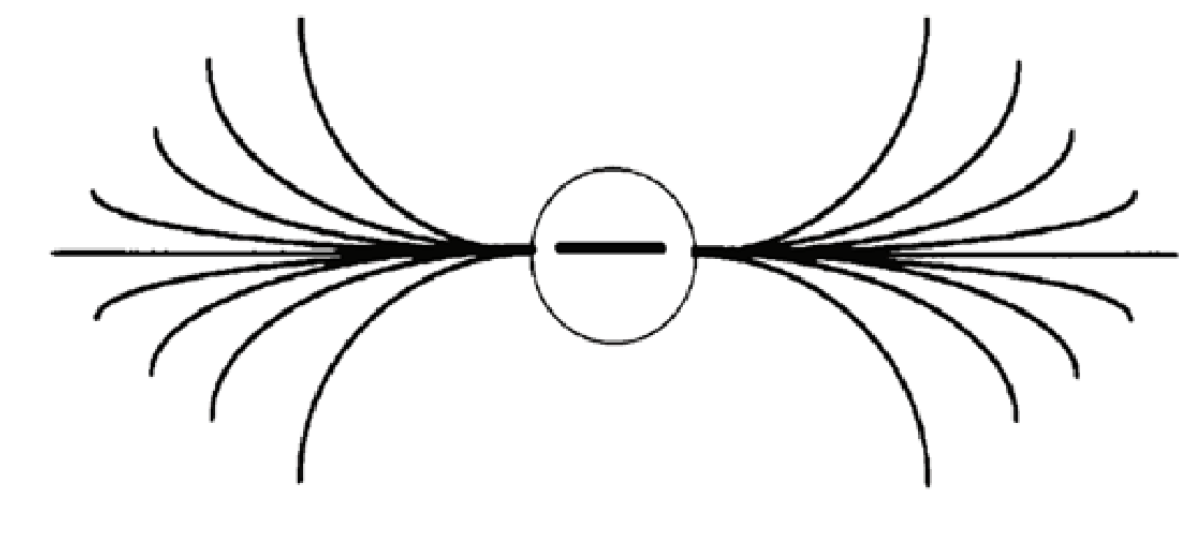}
    \caption{Association field lines from neurogeometry \protect\cite[Fig.43]{petitot2003neurogeometry}, \protect\cite[Fig.16]{field1993contour}. Such association field lines can be well approximated by spatially projected sub-Riemannian geodesics in \(\bbM_2\) \protect\cite{petitot2003neurogeometry,citti2006cortical,baspinar2021cortical,franceschiello2019geometrical}, \protect\cite[Fig.17]{duits2013association}. }
    \label{fig:association_field}
\end{figure}

As shown in \cite[Fig.17]{duits2013association} association fields are closely approximated by (projected) sub-Riemannian geodesics in \(\bbM_2\) for which optimal synthesis has been obtained by Sachkov and Moiseev \cite{sachkov2011cutlocus,moiseev2010maxwell}. Furthermore, in \cite{duits2018optimal} it is shown that the Riemannian geodesics in \(\bbM_2\) converge to the sub-Riemannian geodesics by increasing the spatial anisotropy \(\zeta\) of the metric. This shows that in practice one can approximate the sub-Riemannian model by Riemannian models. \Cref{fig:relation_distance_association_field} shows the relation between association fields and sub-Riemannian geometry in \(\bbM_2\).

\renewcommand{\figwidth}{0.49\linewidth}
\begin{figure}
   \centering
    \begin{subfigure}{\figwidth}
        \centering
        \includegraphics[width=\linewidth]{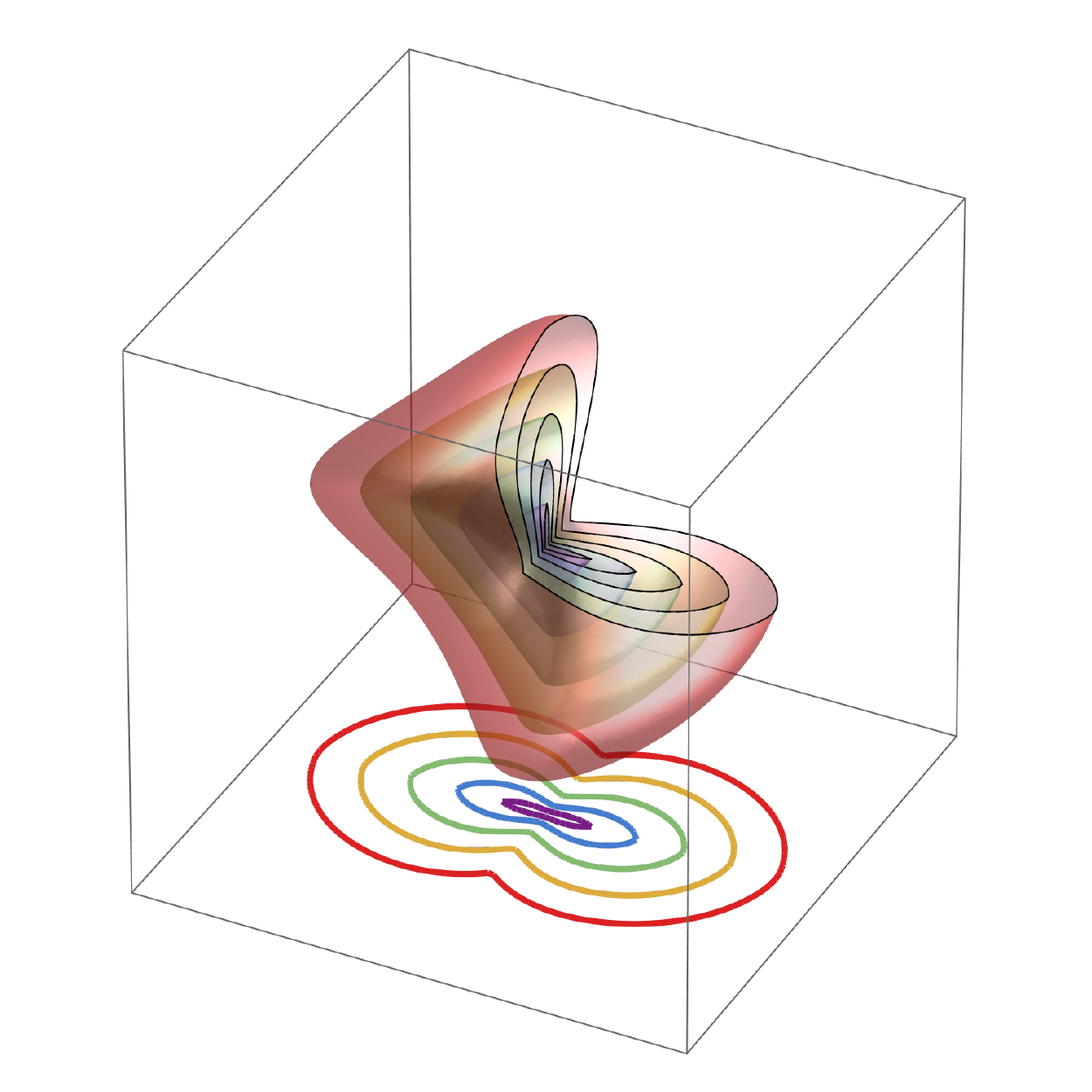}
        \caption{}
        \label{fig:multiple_contours_distance}
    \end{subfigure}
    \begin{subfigure}{\figwidth}
        \centering
        \includegraphics[width=\linewidth]{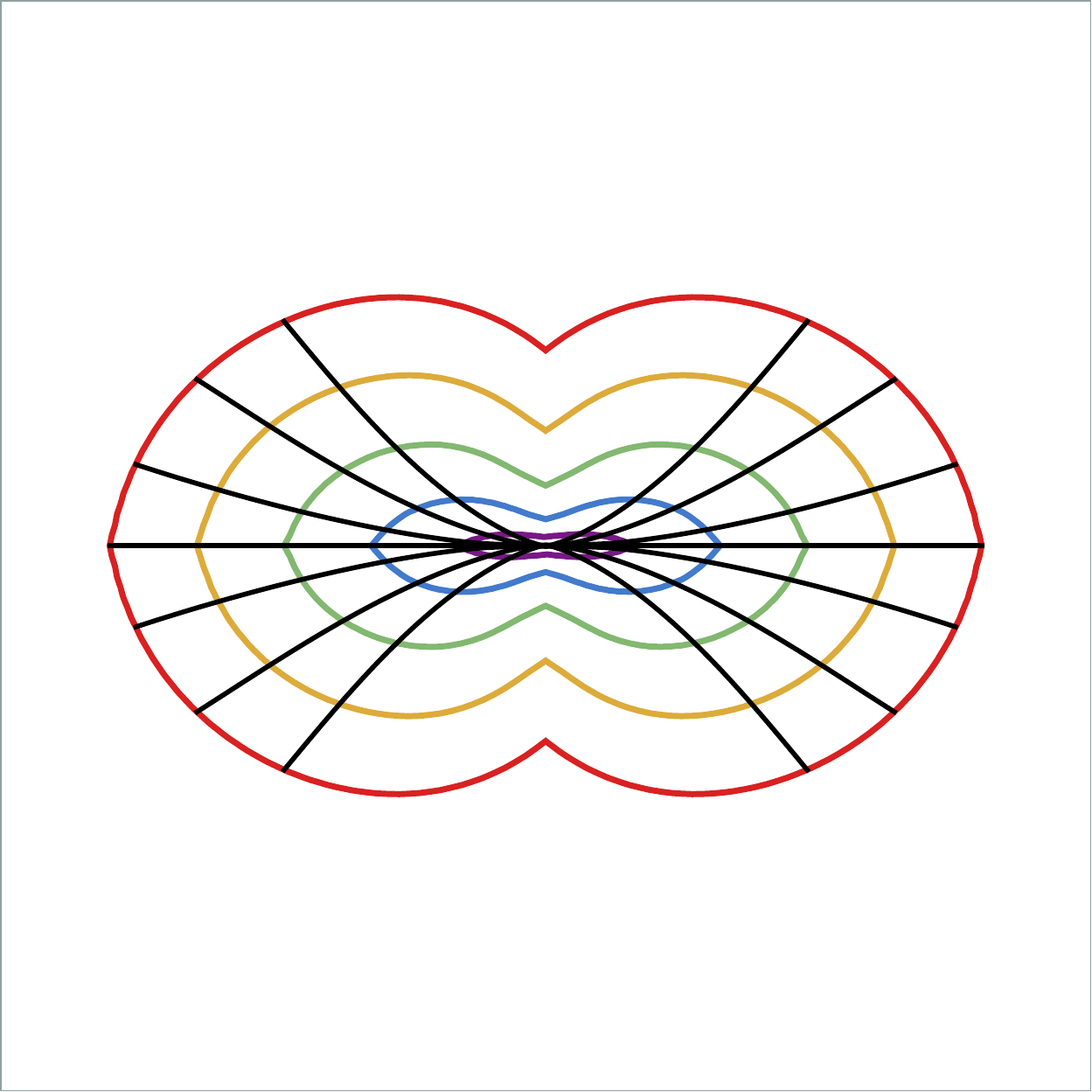}
        \caption{}
        \label{fig:geodesics_association_field}
    \end{subfigure}
    \caption{A visualization of the exact Riemannian distance \(d\), and its relation with association fields. In \ref{fig:multiple_contours_distance} we see isocontours of \(d(\bp_0, \cdot)\) in \(\bbM_2\), and on the bottom we see the min-projection over \(\theta\) of these contours (thus we selected the minimal ending angle in contrast to \Cref{fig:association_field}). The domain of the plot is \([-3,3]^2\times[-\pi,\pi) \subset \bbM_2\). The chosen contours are \(d = 0.5, 1, 1.5, 2\), and \(2.5\). The metric parameters are \((w_1,w_2,w_3)=(1,64,1)\). Due to the very high spatial anisotropy we approach the sub-Riemannian setting. In \ref{fig:geodesics_association_field} we see the same min-projection together with some corresponding spatially projected geodesics.}
    \label{fig:relation_distance_association_field}
\end{figure}

The relation between association fields and Riemannian geometry on \(\bbM_2\) directly extends to a relation between dilation/erosion and association fields. Namely, performing dilation on an orientation score in \(\bbM_2\) is similar to extending a line segment along its association field lines. Similarly, performing erosion is similar to sharpening a line segment perpendicular to its association field lines. This makes dilation/erosion the perfect candidate for a task such as \textit{line completion}. 

In the line completion problem, the input is an image containing multiple line segments, and the desired output is an image of the line that is ``hidden'' in the input image. \Cref{fig:line_completion_sample} shows such an input and desired output. This is also what David Field et al. studied in \cite{field1993contour}. We anticipate that PDE-G-CNNs outperform classical CNNs in the line completion problem due to PDE-G-CNNs being able to dilate and erode. To investigate this we made a synthetic dataset called ``Lines'' consisting of grayscale \(64\times 64\) pixel images, together with their ground-truth line completion. In \Cref{fig:full_network} a complete abstract overview of the architecture of a PDE-G-CNN performing line completion is visualized. \Cref{fig:feature_maps} illustrates how a PDE-G-CNN and CNN incrementally complete a line throughout their layers.

\renewcommand{\figwidth}{0.45\linewidth}
\begin{figure}
    \centering
    \begin{subfigure}{\figwidth}
        \centering
        \includegraphics[width=0.8\linewidth]{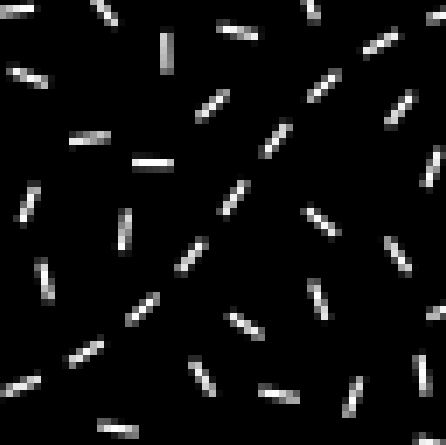}
        \caption{}
        \label{fig:line_completion_sample_input}
    \end{subfigure}
    \begin{subfigure}{\figwidth}
        \centering
        \includegraphics[width=0.8\linewidth]{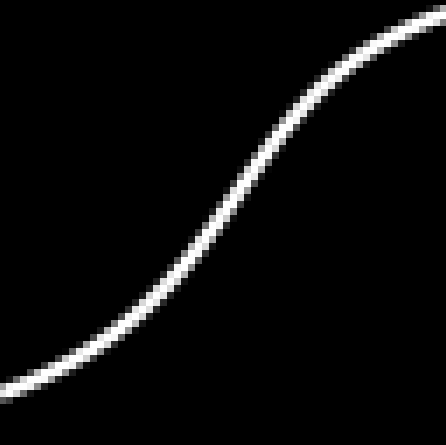}
        \caption{}
        \label{fig:line_completion_sample_gt}
    \end{subfigure}
    \caption{One sample of the Lines dataset. In \ref{fig:line_completion_sample_input} we see the input, in \ref{fig:line_completion_sample_gt} the perceived curve that we consider as ground-truth (as the input is constructed by interrupting the ground-truth line and adding random local orientations).}
    \label{fig:line_completion_sample}
\end{figure}

\begin{figure*}
    \centering
    \includegraphics[width=\linewidth]{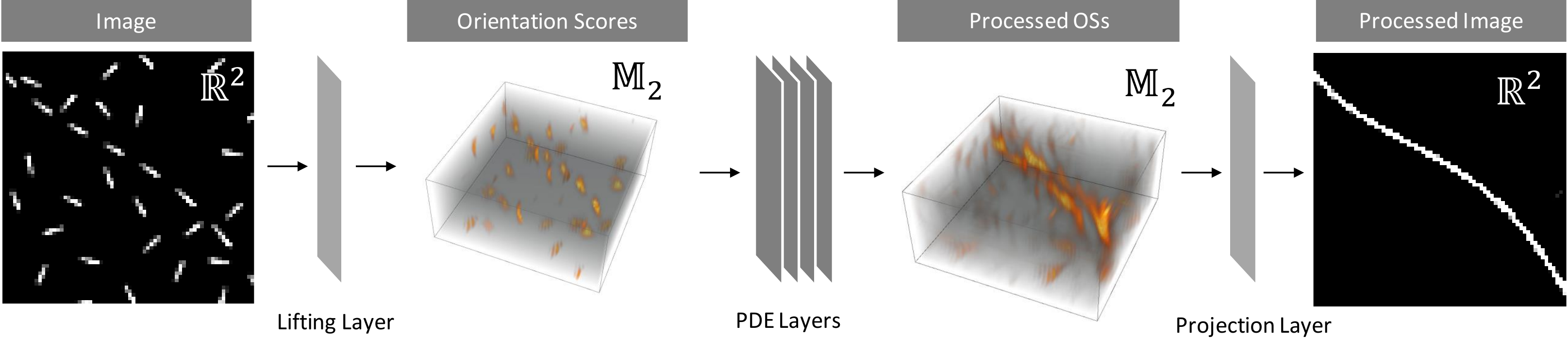}
    \caption{The overall architecture for a PDE-G-CNN performing line completion on the Lines data set. Note how the input image is lifted to an orientation score that lives in the higher dimensional space \(\bbM_2\), run through PDE-G-CNN layers(\Cref{fig:difference_CNNs_and_PDEGCNNS,fig:pde_evolution}), and afterwards projected down back to \(\bbR^2\). Usually this projection is done by taking the maximum value of a feature map over the orientations \(\theta\), for every position \((x,y) \in \bbR^2\). }
    \label{fig:full_network}
\end{figure*}

\begin{figure*}
    \centering
    \includegraphics[width=0.85\linewidth]{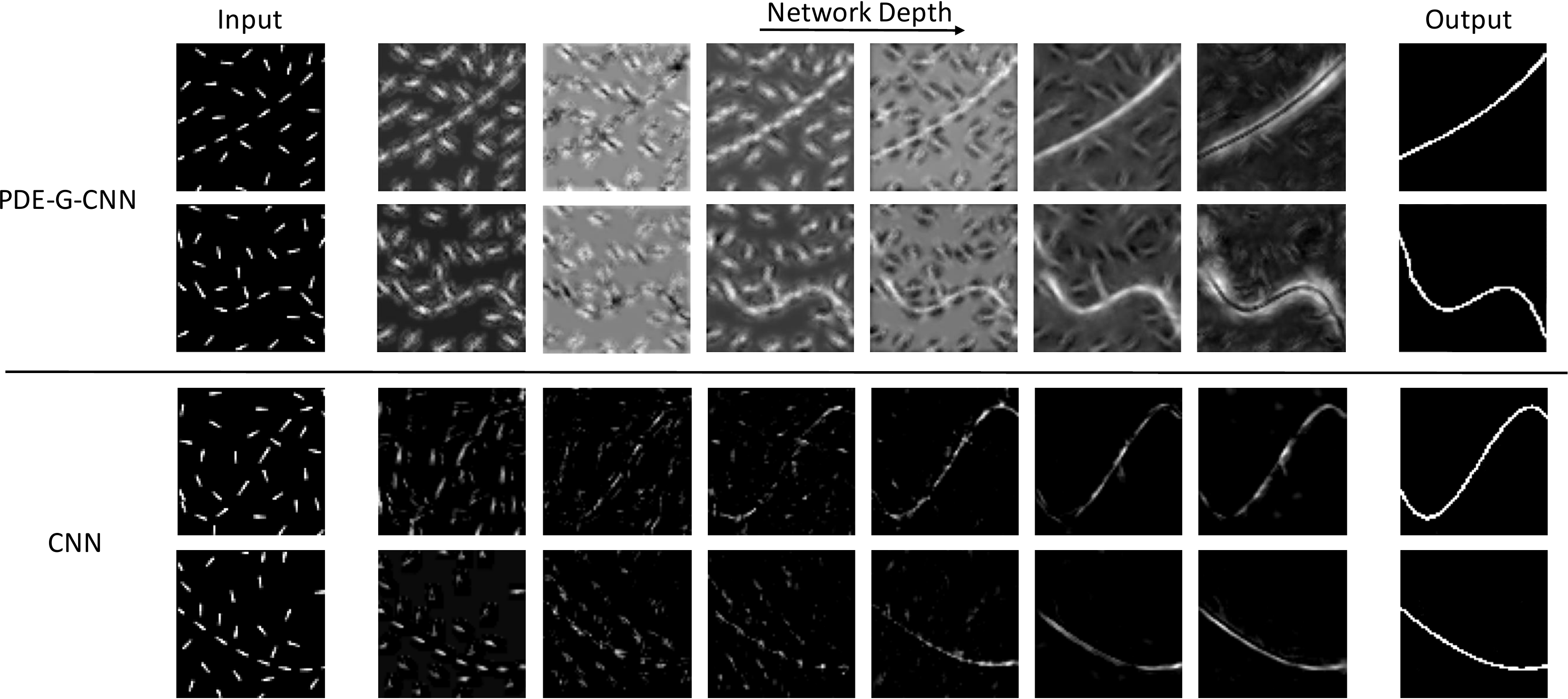}
    \caption{Visualization of how a PDE-G-CNN and CNN incrementally complete a line throughout their layers. The first two rows are of a PDE-G-CNN, the second two rows of a CNN. The first column is the input, the last column the output. The intermediate columns are a representative selection of feature maps from the output of the respective CNN or PDE layer (\Cref{fig:difference_CNNs_and_PDEGCNNS}). The feature maps of the PDE-G-CNN live in \(\bbM_2\), but for clarity we only show the max-projection over \(\theta\). Within the feature maps of the PDE-G-CNN association fields from neurogeometry \protect\cite{field1993contour,petitot2003neurogeometry,petitot2017elements} become visible as network depth increases. Such merging of association fields is not visible in the feature maps of the CNN. This observation is consistent throughout different inputs.
    }
    \label{fig:feature_maps}
\end{figure*}

In \Cref{res:kernels} we show that solving the dilation and erosion PDEs can be done by performing a morphological convolution with a \textit{morphological kernel} \(k_t^{\alpha} : \bbM_2 \to \bbR_{\geq 0}\), which is easily expressed in the Riemannian distance \(d=d_{\mathcal{G}}\) on the manifold:
\begin{equation} \label{eq:morphological_kernel_intro}
    k_t^{\alpha}(\bp)=\frac{t}{\beta} \left( \frac{d_{\cG}(\bp_0,\bp)}{t}\right)^{\beta}.
\end{equation}
Here \(\bp_0 = (0,0,0)\) is our reference point in \(\bbM_2\), and time $t>0$ controls the amount of erosion and dilation. Furthermore, $\alpha>1$ controls the ``softness'' of the max and min-pooling, with $\frac{1}{\alpha}+\frac{1}{\beta}=1$. Erosion is done through a direct morphological convolution \eqref{eq:morphological_convolution} with this specific kernel. Dilation is solved in a slightly different way but again with the same kernel (\Cref{res:kernels} in \Cref{sec:morphological} will explain the details).

And this is where a problem arises: calculating the exact distance \(d\) on \(\bbM_2\) required in (\ref{eq:morphological_kernel_intro}) is computationally expensive \cite{bekkers2015pde}. To alleviate this issue, we resort to estimating the true distance \(d\) with computationally efficient approximative distances, denoted throughout the article by \(\rho\). We then use such a distance approximation within \eqref{eq:morphological_kernel_intro} to create a corresponding approximative morphological kernel, and in turn use this to efficiently calculate the effect of dilation and erosion.

In \cite{smets2022pdebased} one such distance approximation is used: the \textit{logarithmic distance estimate} \(\rho_c\) which uses the logarithmic coordinates \(c^i\) \eqref{eq:logarithm_coordinates}. In short, \(\rho_c(\bp)\) is equal to the Riemannian length of the exponential curve that connects \(\bp_0\) to \(\bp\). The formal definition will follow in \Cref{sec:distances}. In \Cref{fig:relation_rho_c_association_field} an impression of \(\rho_c\) is given.

\renewcommand{\figwidth}{0.49\linewidth}
\begin{figure}
   \centering
    \begin{subfigure}{\figwidth}
        \centering
        \includegraphics[width=\linewidth]{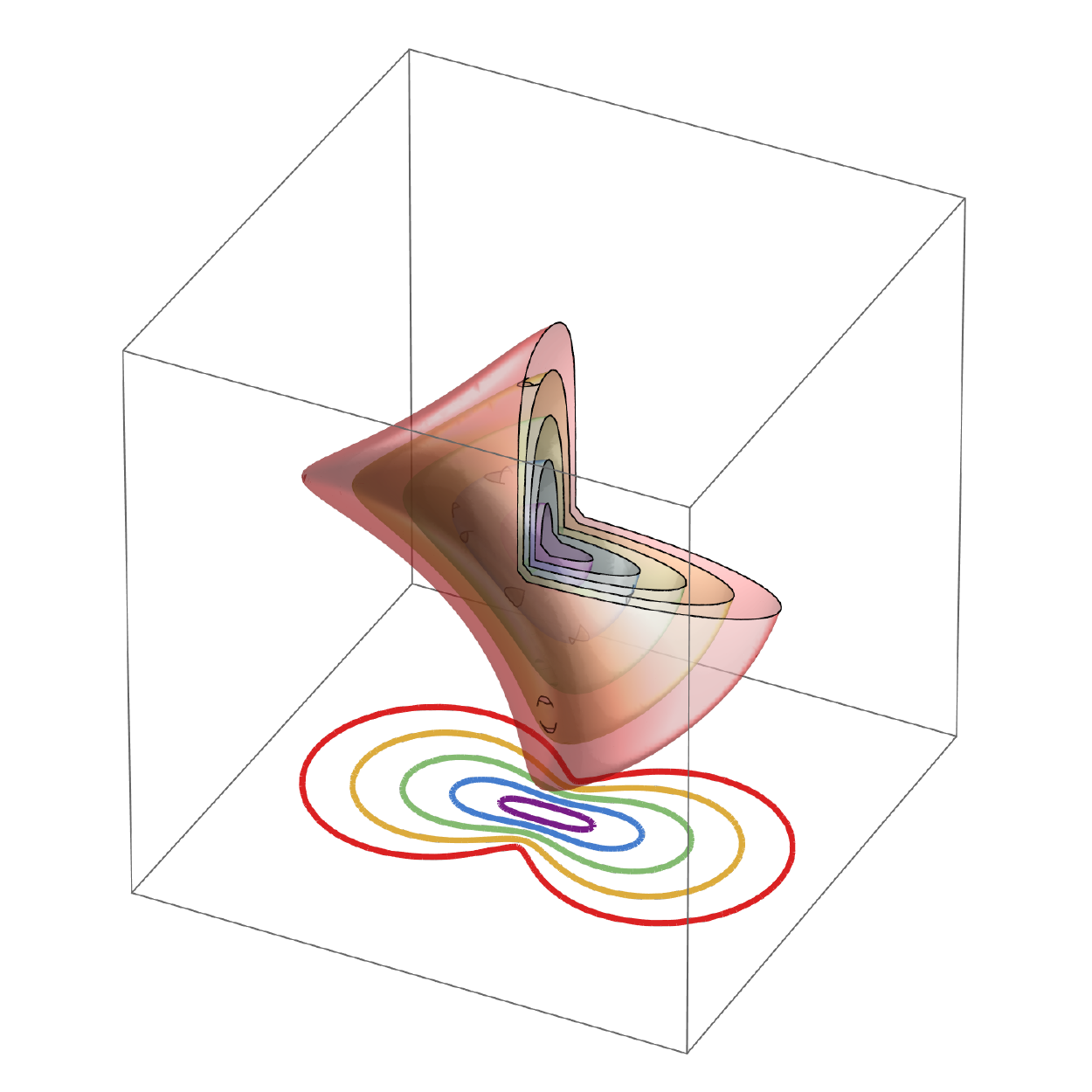}
        \caption{}
        \label{fig:multiple_contours_rho_c}
    \end{subfigure}
    \begin{subfigure}{\figwidth}
        \centering
        \includegraphics[width=\linewidth]{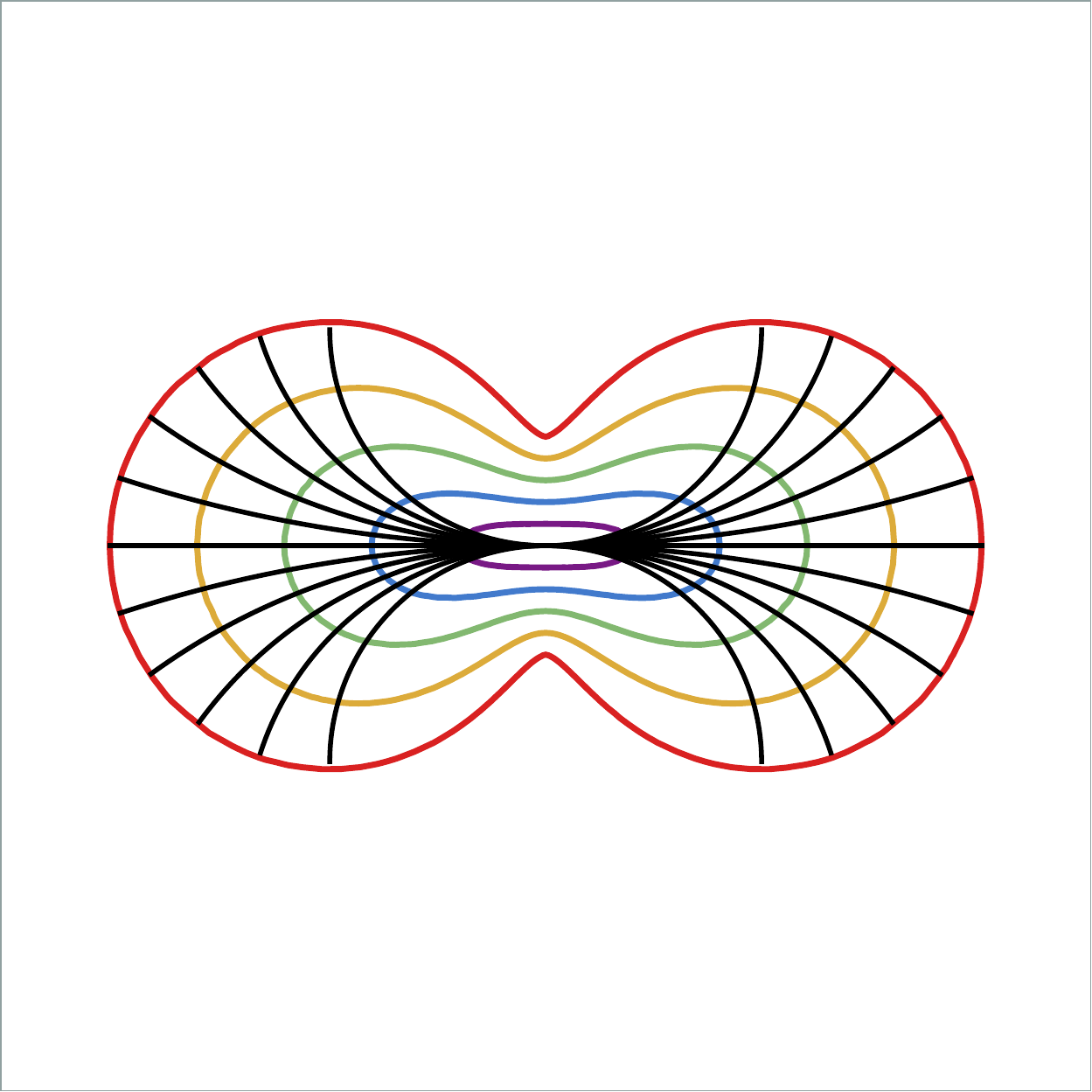}
        \caption{}
        \label{fig:exp_curves_association_field}
    \end{subfigure}
    \caption{A visualization of \(\rho_c\), similar to \Cref{fig:relation_distance_association_field}. In \ref{fig:multiple_contours_rho_c} we see multiple contours of \(\rho_c\), and on the bottom we see the min-projection over \(\theta\). The metric parameters are \((w_1,w_2,w_3)=(1,4,1)\). In \ref{fig:exp_curves_association_field} we see the same min-projection together with some corresponding spatially projected exponential curves. Note the similarity to \Cref{fig:association_field}. }
    \label{fig:relation_rho_c_association_field}
\end{figure}

Clearly, an error is made when the effect of erosion and dilation is calculated with an approximative morphological kernel. As a morphological kernel is completely determined by its corresponding (approximative) distance, it follows that one can analyse the error by analyzing the difference between the exact distance \(d\) and approximative distance \(\rho\) that is used.

Despite showing in \cite{smets2022pdebased} that $d \leq \rho_c$ no concrete bounds are given, apart from the asymptotic \( \rho_c^2 \leq d^2 + \cO(d^4) \). This motivates us to do a more in-depth analysis on the quality of the distance approximations.

We introduce a variation on the logarithmic estimate \(\rho_c\) called the \textit{half-angle distance estimate} \(\rho_b\), and analyse that. The half-angle approximation uses not the logarithmic coordinates but half-angle coordinates \(b^i\). The definition of these is also given later \eqref{eq:half_angle_coordinates}. In practice \(\rho_c\) and \(\rho_b\) do not differ much, but analysing \(\rho_b\) is much easier!

The main theorem of the paper, \Cref{res:main_results}, collects new theoretical results that describe the quality of using the half-angle distance approximation \(\rho_b\) for solving dilation and erosion in practice. It relates the approximative morphological kernel \(k_b\) corresponding with \(\rho_b\), to the exact kernel \(k\) \eqref{eq:morphological_kernel_intro}. 

Both the logarithmic estimate \(\rho_c\) and half-angle estimate \(\rho_b\) approximate the true Riemannian distance \(d\) quite well in certain cases. One of these cases is when the Riemannian metric has a low spatial anisotropy \(\zeta\). We can show this visually by comparing the isocontours of the exact and approximative distances. However, interpreting and comparing these surfaces can be difficult. This is why we have decided to additionally plot multiple \(\theta\)-isocontours of these surfaces. In \Cref{fig:true_distance_plot_intro} one such plot can be seen, and illustrates how it must be interpreted.

\begin{figure}
    \centering
    \includegraphics[width=0.8\linewidth]{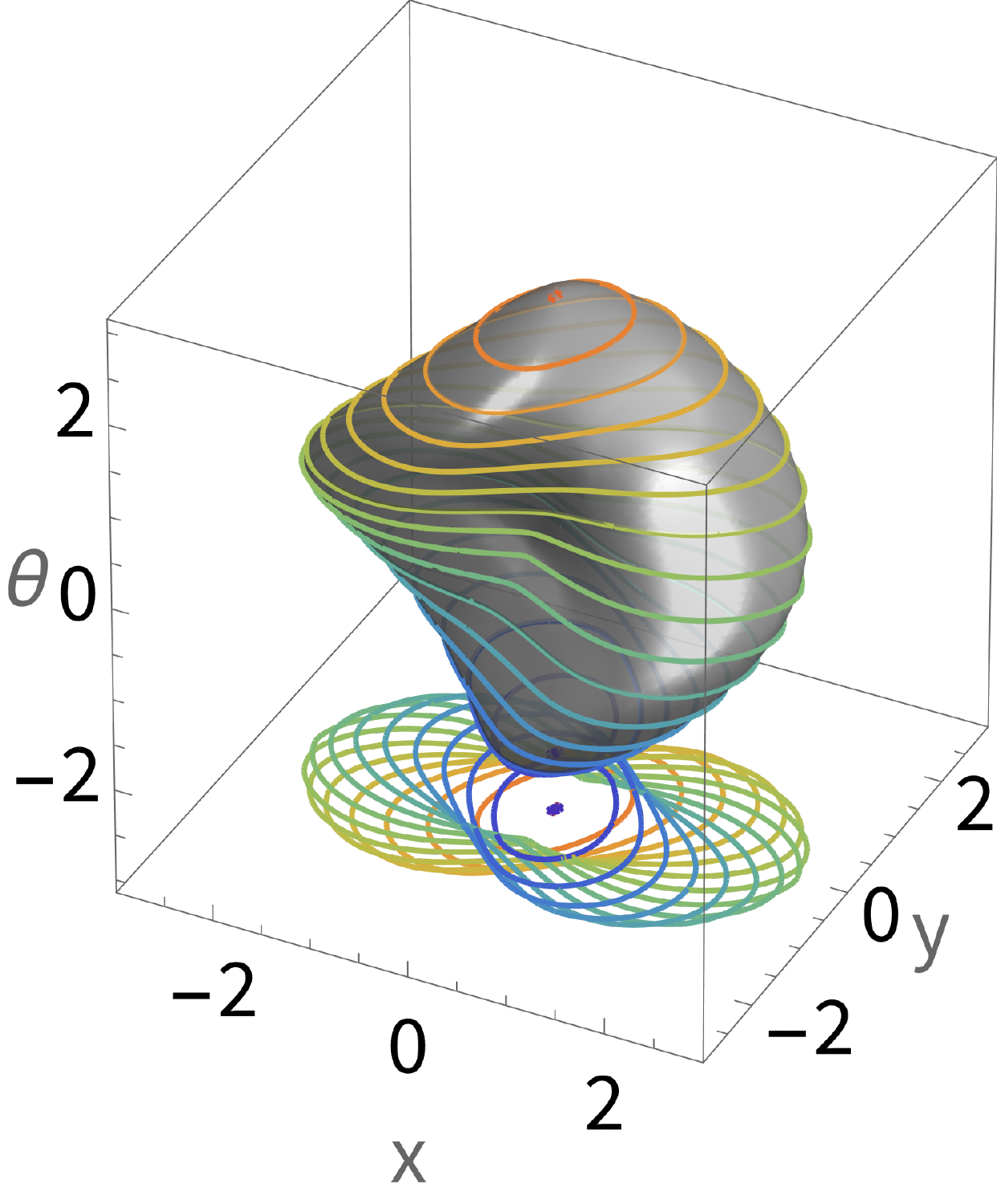}
    \caption{In grey the isocontour \(d=2.5\) is plotted. The metric parameters are \((w_1,w_2,w_3)=(1,8,1)\). For \(\theta = k\pi/10\) with \( k = -10,\dots,10 \) the isocontours are drawn and projected onto the bottom of the figure. The same kind of visualizations are used in \Cref{tab:balls,tab:balls_high_anisotropy}. }
    \label{fig:true_distance_plot_intro}
\end{figure}

In \Cref{tab:balls} a spatially isotropic \(\zeta = 1\) and low-anisotropic case \(\zeta = 2\) is visualized. Note that \(\rho_b\) approximates \(d\) well in these cases. In fact, \(\rho_b\) is exactly equal to the true distance \(d\) in the spatially isotropic case, which \textit{is not true} for \(\rho_c\).

Both the logarithm and half-angle approximation fail specifically in the high spatial anisotropy regime. For example when \(\zeta = 8\). The first two columns of \Cref{tab:balls_high_anisotropy} show that, indeed, \(\rho_b\) is no longer a good approximation of the exact distance \(d\). For this reason we introduce a novel \textit{sub-Riemannian} distance approximations \(\rho_{b, sr}\), which is visualized in the third column of \Cref{tab:balls_high_anisotropy}. 

Finally, we propose an approximative distance \(\rho_{com}\) that carefully combines the Riemannian and sub-Riemannian approximations into one. This combined approximation automatically switches to the estimate that is more appropriate depending on the spatial anisotropy, and hence covers both the low and high anisotropy regimes. Using the corresponding morphological kernel of \(\rho_{com}\) to solve erosion and dilation we obtain more accurate (and still tangible) solutions of the non-linear parts in the PDE-G-CNNs.

\renewcommand{\figwidth}{0.2\linewidth}
\begin{table*}
    \centering
    \begin{tabular}{c||c|c}
    
         & \(\zeta = 1 \) & \(\zeta = 2\)\\
         
        \hline \hline
        
        \(d\)
        & 
        \includegraphics[width=\figwidth]{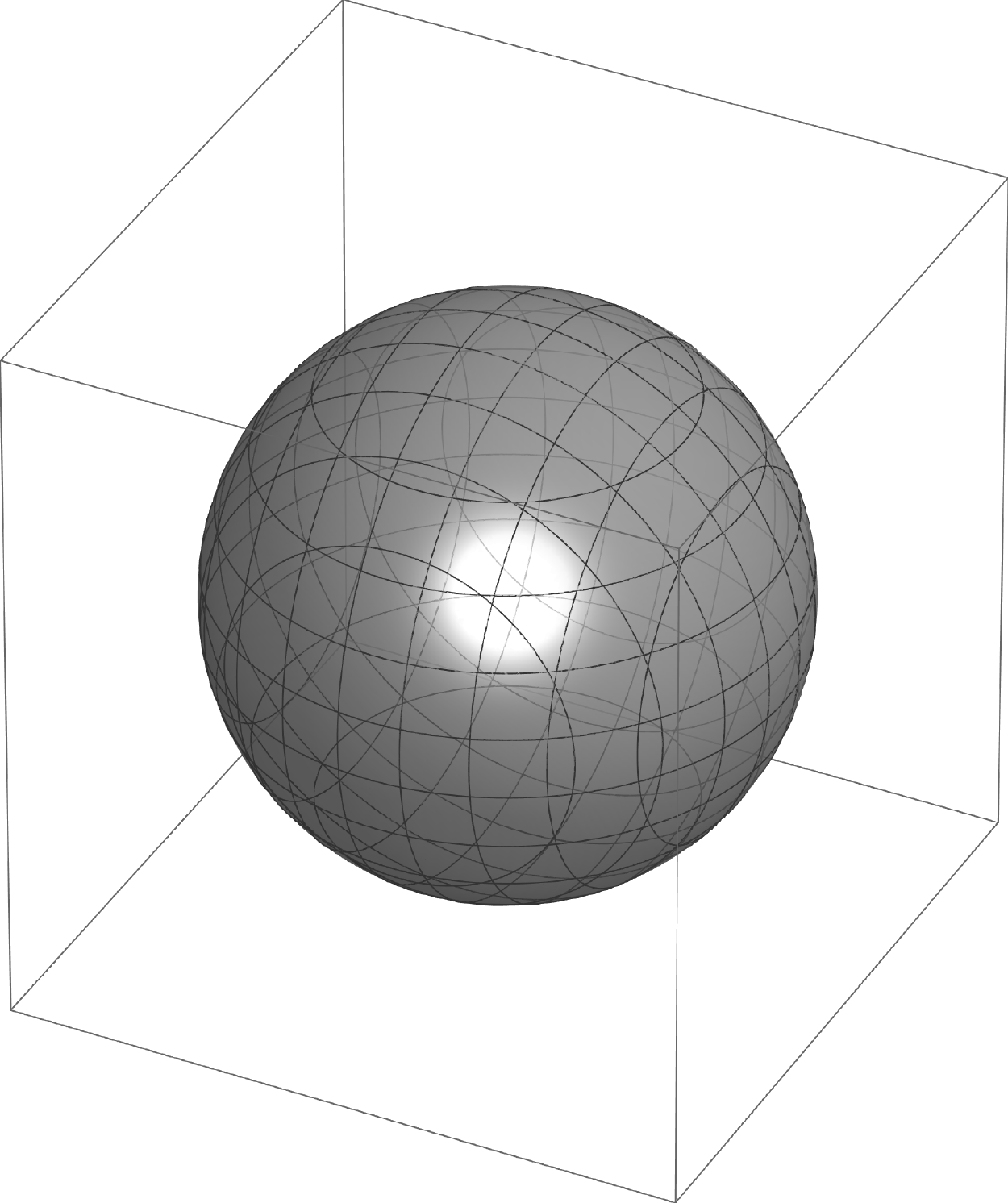}
        \includegraphics[width=\figwidth]{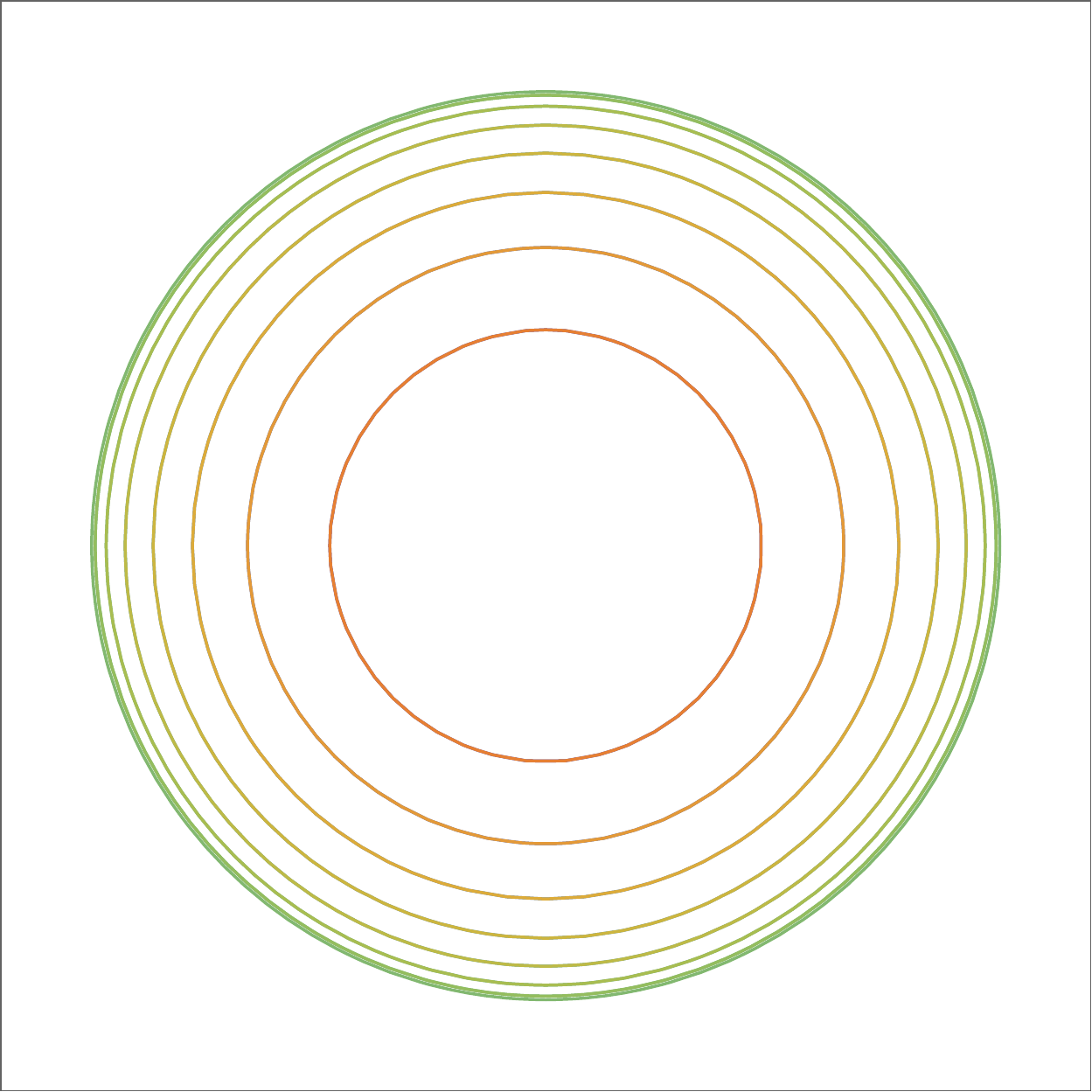} 
        & 
        \includegraphics[width=\figwidth]{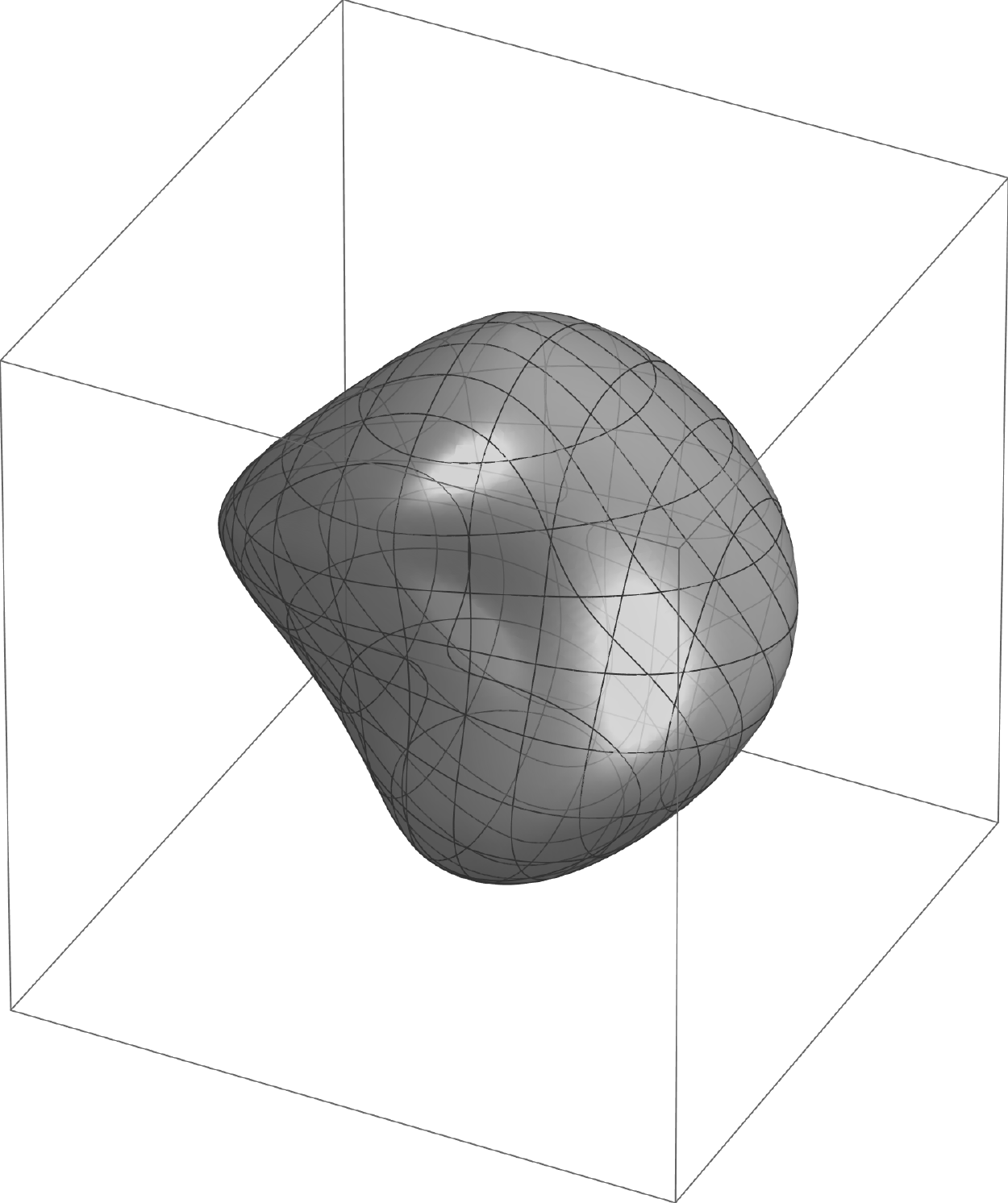} 
        \includegraphics[width=\figwidth]{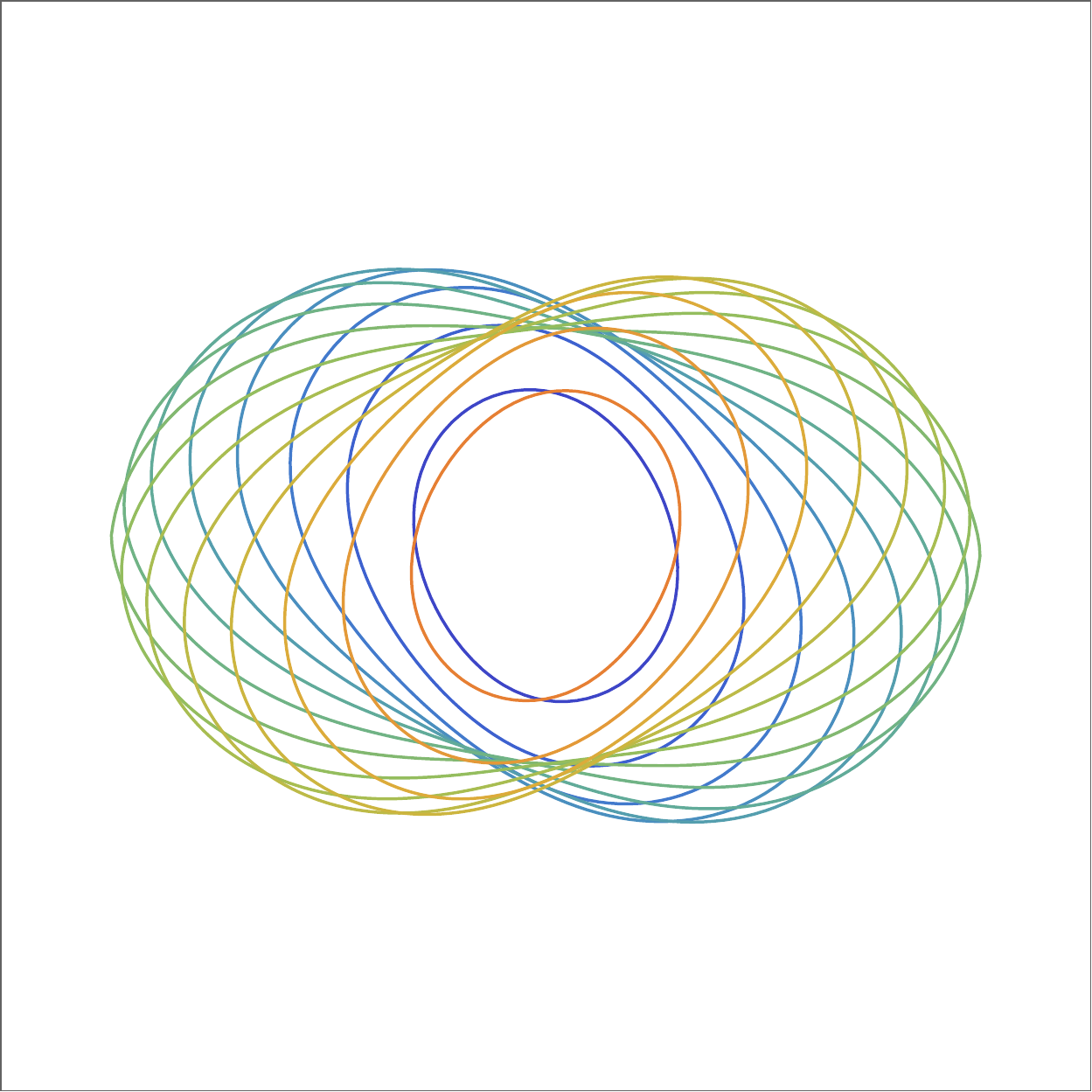} 
        \\
        
        \hline
        
        \(\rho_b\)
        & 
        \includegraphics[width=\figwidth]{figures/clean_rho_b_w=1,1,1_svg-tex.pdf}
        \includegraphics[width=\figwidth]{figures/clean_flat_rho_b_w=1,1,1_svg-tex.pdf} 
        &
        \includegraphics[width=\figwidth]{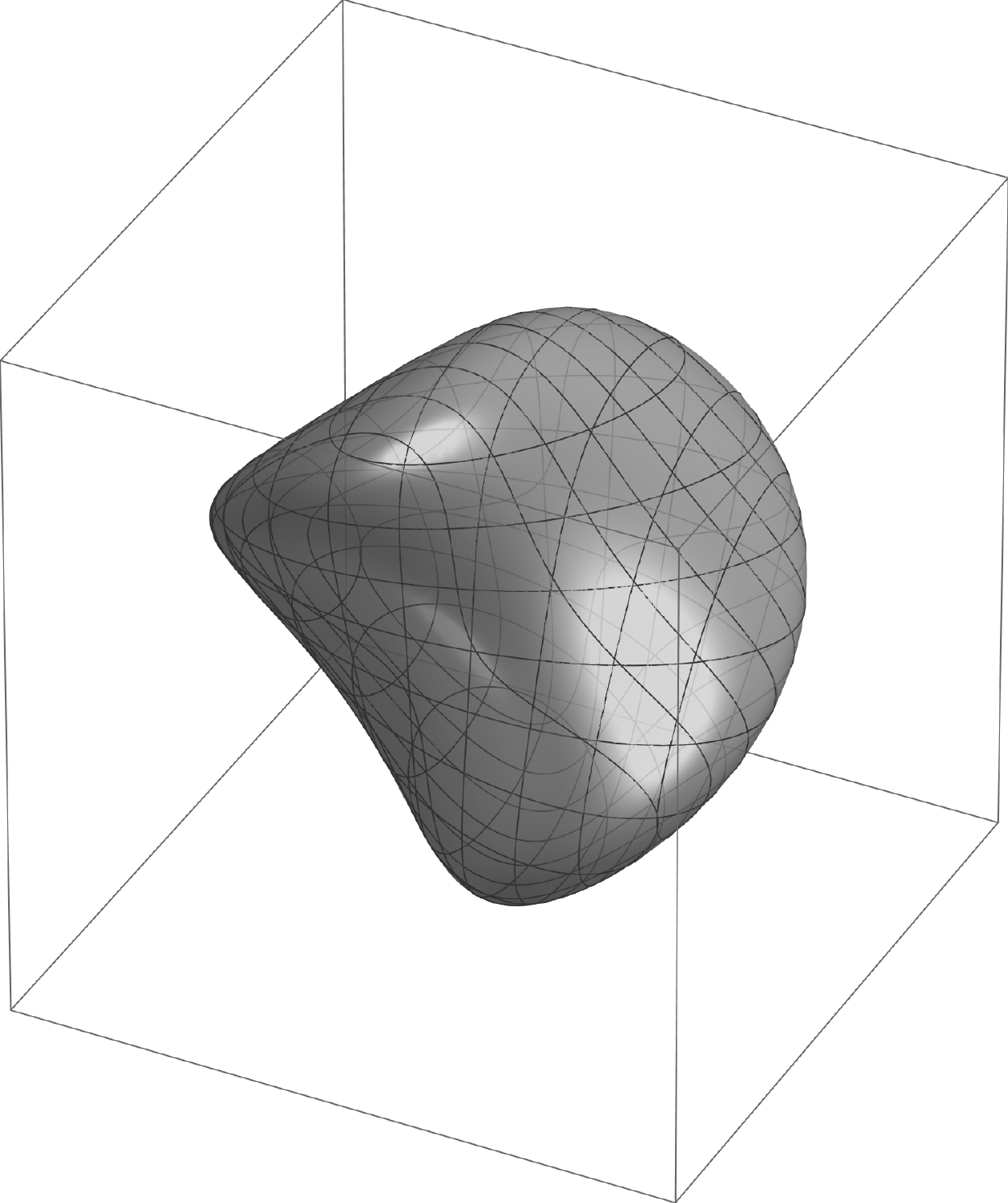}
        \includegraphics[width=\figwidth]{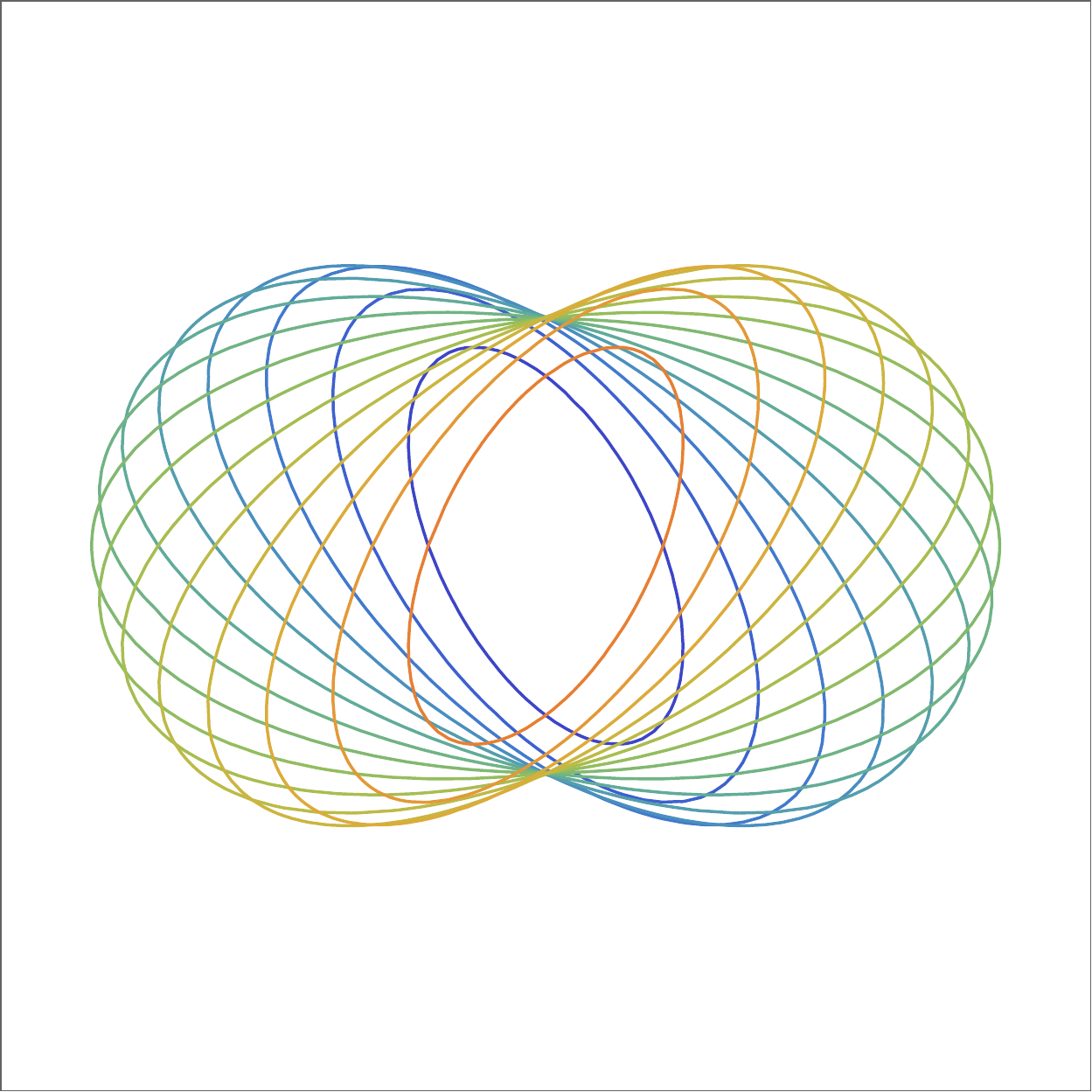}
    \end{tabular}
    \caption{The balls of the exact distance \(d\) and approximative distance \(\rho_b\) in the isotropic and low anisotropic case. The radius of the balls is set to \(r = 2.5\). The domain of the plots is \([-3,3]\times[-3,3]\times[-\pi,\pi)\). We fix \(w_1=w_3=1\) throughout the plots and vary \(w_2\). For \(\theta = k\pi/10\) with \( k = -10,\dots,10 \) the isocontours are drawn, similar to \Cref{fig:true_distance_plot_intro}.}
    \label{tab:balls}
\end{table*}

\begin{table*}
    \centering
    \begin{tabular}{c||c|c|c}
        & \(d\) & \(\rho_b\) & \(\rho_{b,sr}\) \\
        \hline \hline
        
        \(\zeta = 8\)
        & \includegraphics[width=\figwidth]{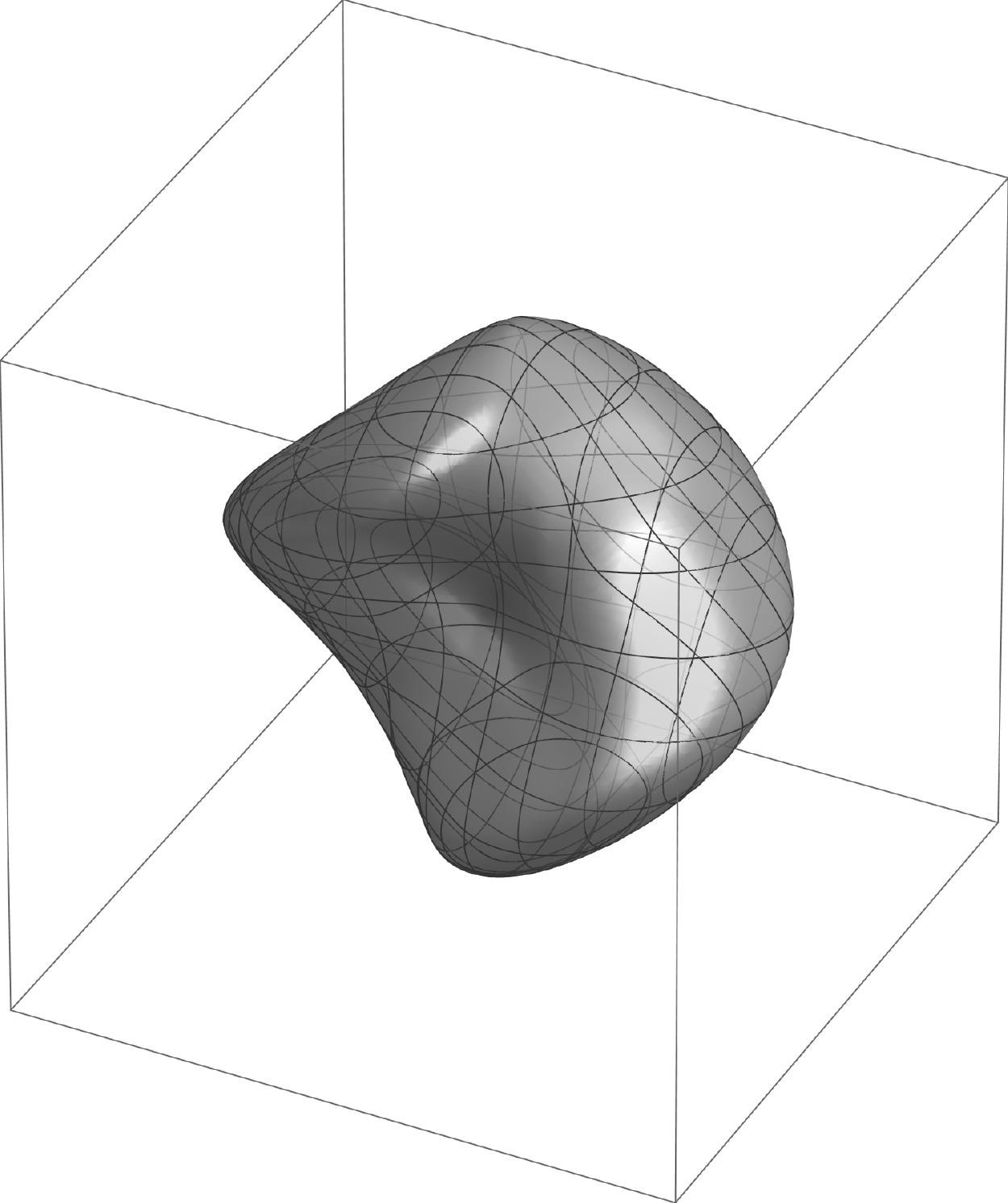} 
        & \includegraphics[width=\figwidth]{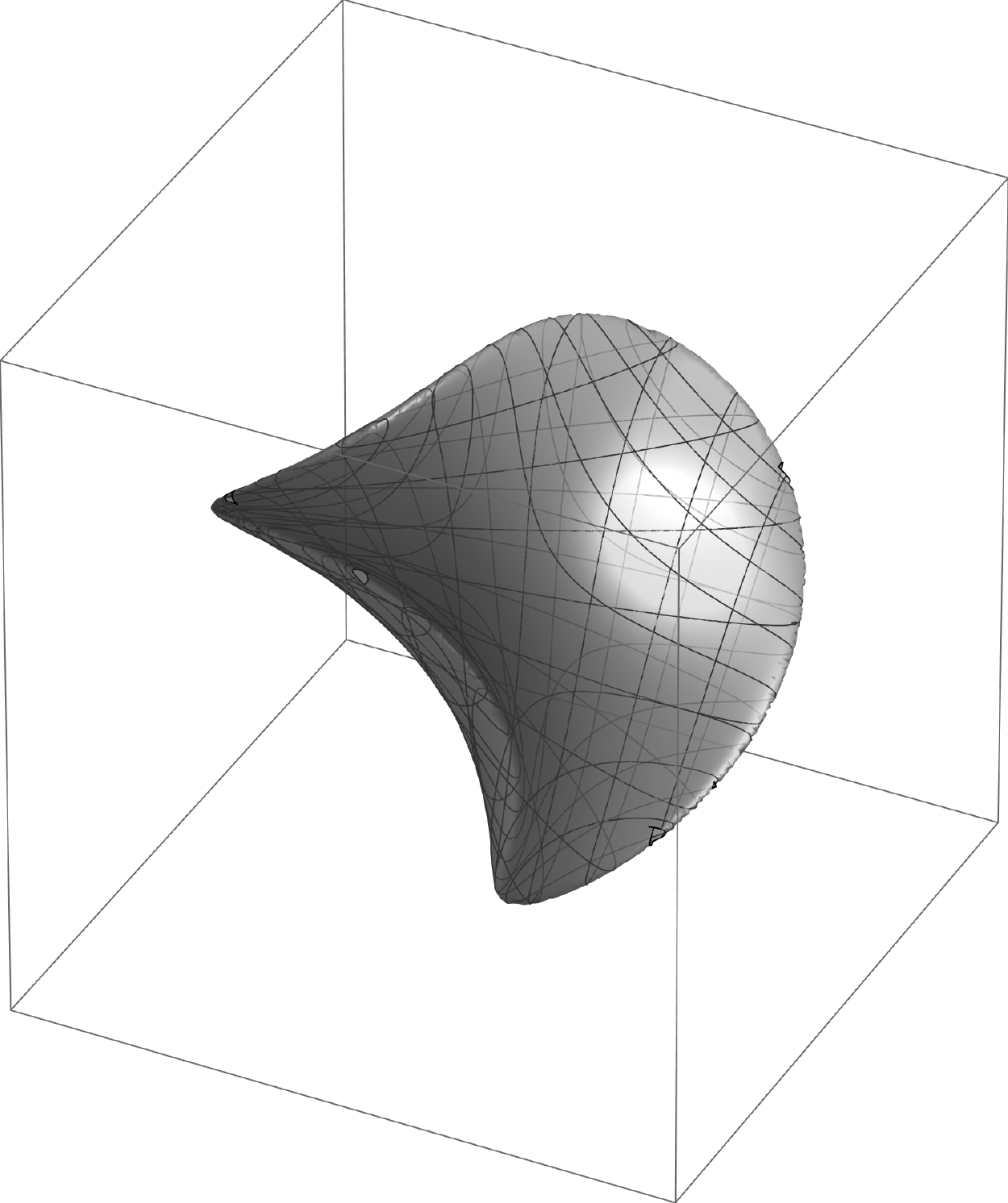}
        & \includegraphics[width=\figwidth]{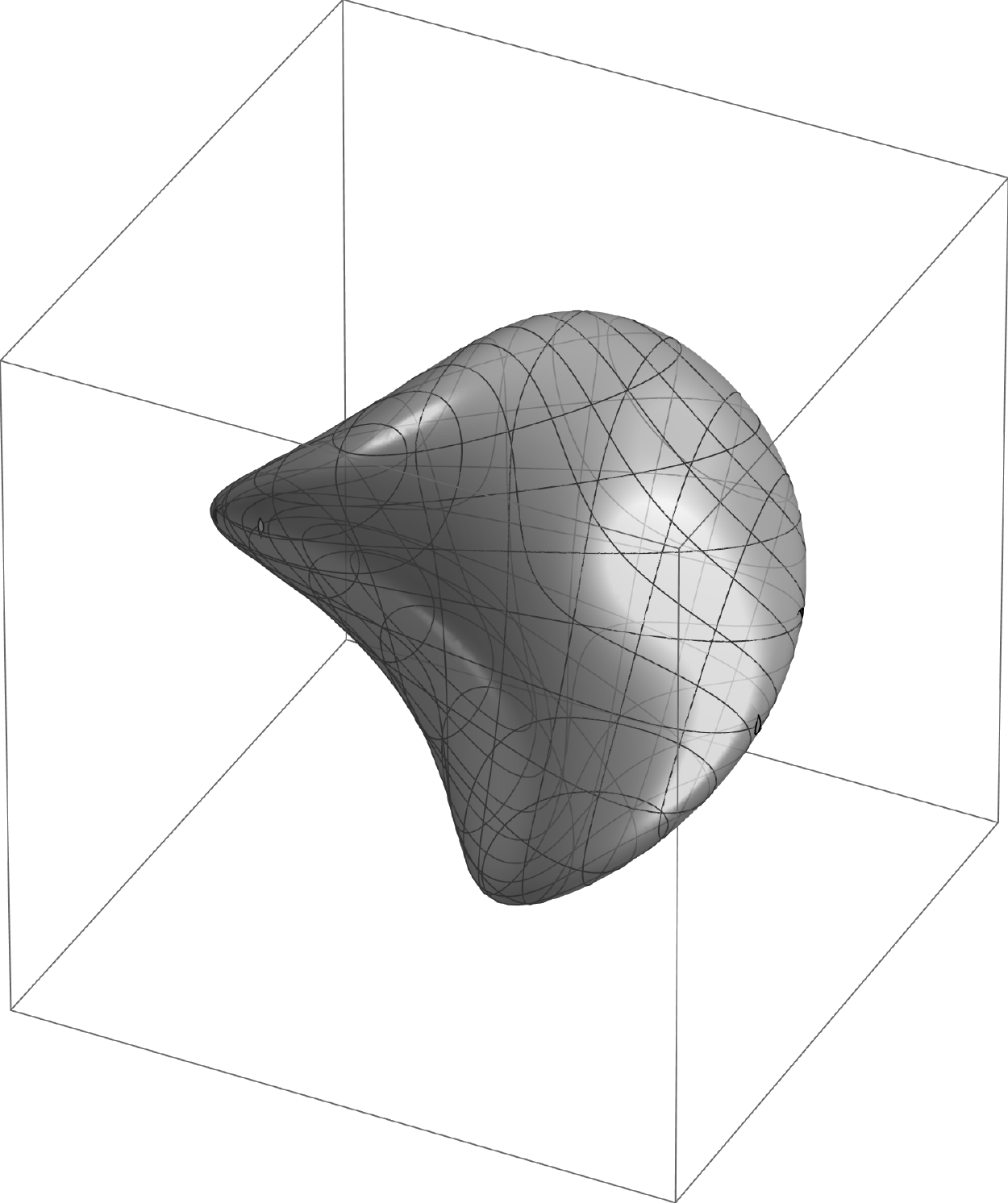}\\ 
         
        & \includegraphics[width=\figwidth]{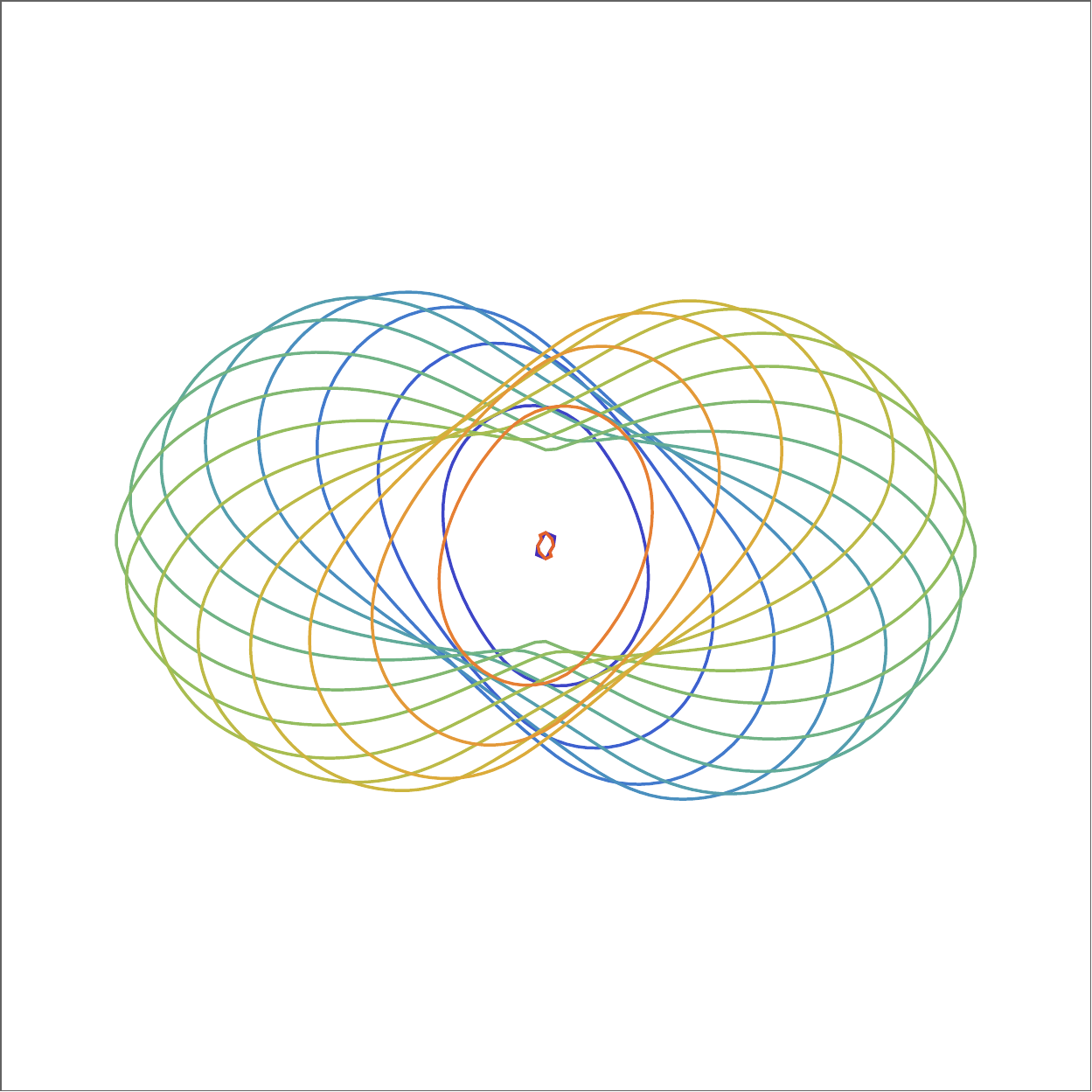} 
        & \includegraphics[width=\figwidth]{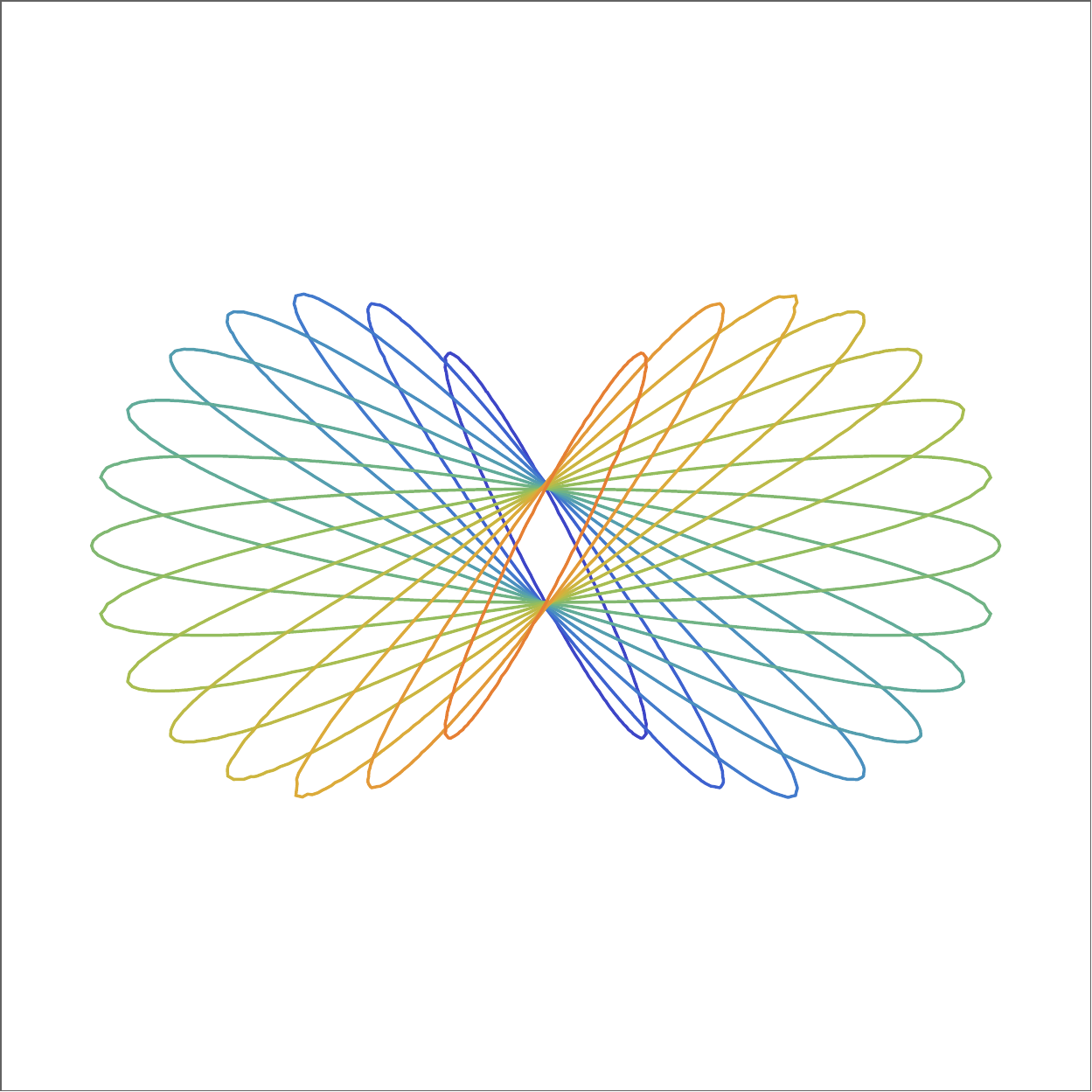} 
        & \includegraphics[width=\figwidth]{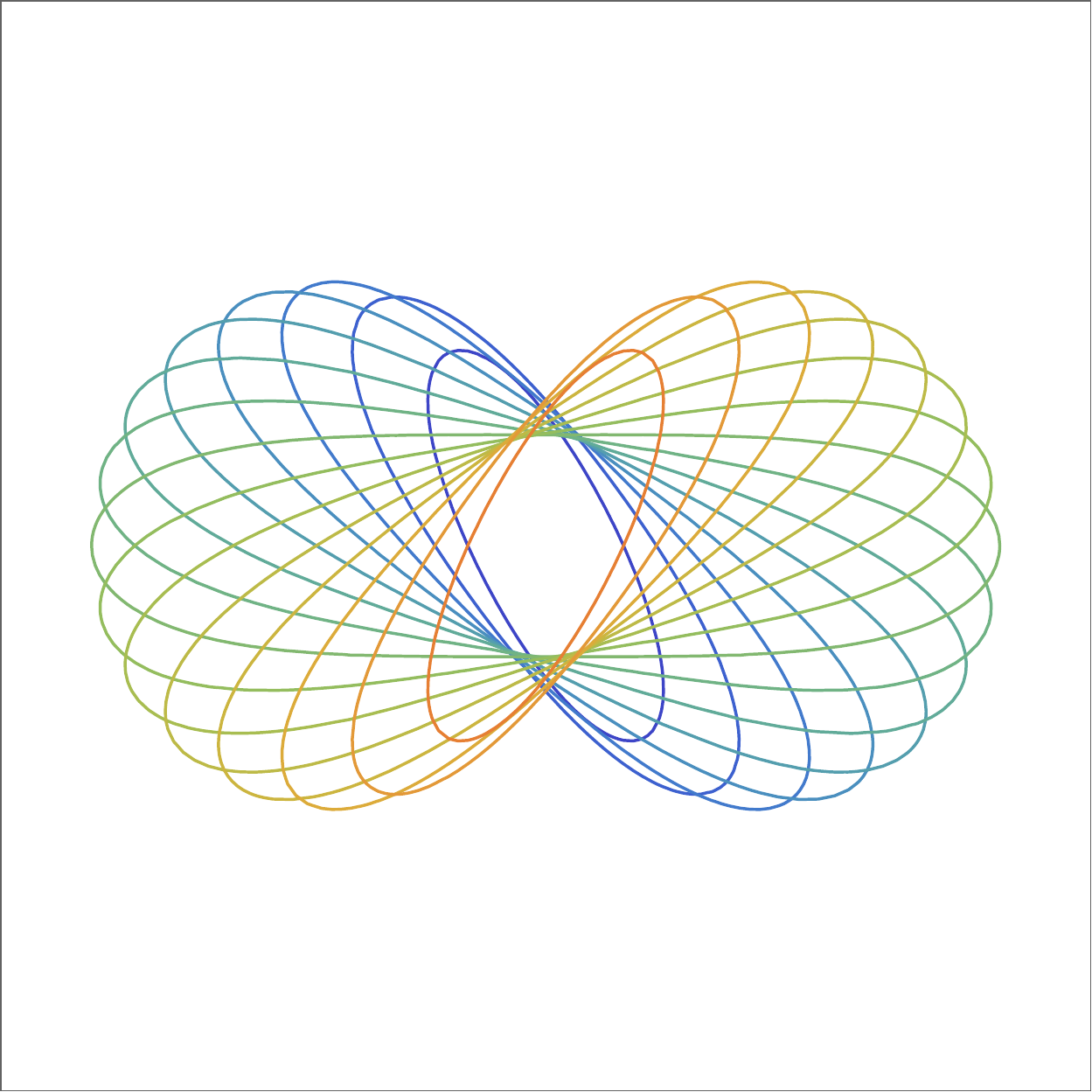}
    \end{tabular}
    \caption{The same as \Cref{tab:balls} but in the high spatially anisotropic case. Alongside the approximation \(\rho_b\) the sub-Riemannian distance approximation \(\rho_{b,sr}\) is plotted with \(\nu = 1.6\). We see that the isocontours of \(\rho_b\) are too ``thin'' compared to the isocontours of \(d\). The isocontours of \(\rho_{b,sr}\) are better in this respect. }
    \label{tab:balls_high_anisotropy}
\end{table*}

For every distance approximation (listed in \Cref{sec:distances}) we perform an empirical analysis in \Cref{sec:experiments} by seeing how the estimate changes the performance of the PDE-G-CNNs when applied to two datasets: the Lines dataset and publicly available DCA1 dataset.

\subsection{Contributions}

    In \Cref{res:kernels} we summarize how the nonlinear units in PDE-G-CNNs (described by morphological PDEs) are solved using morphological kernels and convolutions, which provides sufficient and essential background for the discussions and results in this paper. 
    
    The key contributions of this article are:
    \begin{itemize}
        \item \Cref{res:main_results} summarizes our mathematical analysis of the quality of the half-angle distance approximation \(\rho_b\) and its corresponding morphological kernel \(k_b\) in PDE-G-CNNs. We do this by comparing \(k_b\) to the exact morphological kernel \(k\). Globally, one can show that they both carry the same symmetries, and that for low spatial anisotropies \(\zeta\) they are almost indistinguishable. Furthermore, we show that locally both kernels are similar through an upper bound on the relative error. This improves upon results in \cite[Lem.20]{smets2022pdebased}.
        
        \item \Cref{tab:balls_high_anisotropy} demonstrates qualitatively that \(\rho_b\) becomes a poor approximation when the spatial anisotropy is high \(\zeta \gg 1\). In \Cref{res:rho_b_goes_certainly_bad} we underpin this theoretically and in \Cref{sec:quantitative_analysis_rho_b} we validate this observation numerically. This motivates the use of a sub-Riemannian approximation when \(\zeta\) is large.
        
        \item In \Cref{sec:distances} we introduce and derive a novel sub-Riemannian distance approximation \(\rho_{sr}\), that overcomes difficulties in previous existing sub-Riemannian kernel approximations \cite{bekkers2018nilpotent}. Subsequently, we propose our approximation \(\rho_{com}\) that combines the Riemannian and sub-Riemannian approximations into one that automatically switches to the approximation that is more appropriate depending on the metric parameters.
        
        \item \Cref{fig:baseline_dca1_scatter,fig:baseline_lines_scatter} shows that PDE-G-CNNs perform just as well as, and sometimes better than, G-CNNs and CNNs on the DCA1 and Lines dataset, while having the least amount of parameters.  \Cref{fig:lines_distance_approximations_scatter,fig:dca1_distance_approximations_scatter} depict an evaluation of the performance of PDE-G-CNNs when using the different distance approximations, again on  the DCA1 and Lines dataset. We observe that the new kernel \(\rho_{b,com}\) provides best results.
    \end{itemize}
    
    Our theoretical contributions are also relevant outside the context of geometric deep learning. Namely, it also applies to general geometric image processing \cite{bekkers2018nilpotent}, neurogeometry \cite{citti2006cortical,petitot2003neurogeometry}, and robotics \cite[Sec.6.8.4]{chirikjian2000engineering}. 
    
    In addition, \Cref{fig:association_field,fig:relation_distance_association_field,fig:relation_rho_c_association_field,fig:feature_maps} show a connection between the PDE-G-CNN framework with the theory of association fields from neurogeometry \cite{petitot2003neurogeometry,field1993contour}. Thereby, PDE-G-CNNs reveal improved geometrical interpretability, in comparison to existing convolution neural networks. In \Cref{app:geometric} we further clarify the geometrical interpretability.
    
\subsection{Outline}

    In \Cref{sec:preliminaries} a short overview of the necessary mathematical preliminaries is given. 
    \Cref{sec:morphological} collects some known results on the exact solution of erosion and dilation on the homogeneous space of two-dimensional positions and orientations \(\bbM_2\), and motivates the use of morphological kernels. 
    In \Cref{sec:distances} all approximative distances are listed. The approximative distances give rise to corresponding approximative morphological kernels. 
    The main theorem of this paper can be found in \Cref{sec:analysis} and consist of three parts, of which the proofs can be found in the relevant subsections. 
    The main theorem mostly concerns itself with the analysis of the approximative morphological kernel \(k_b\). 
    Experiments with various approximative kernels are done and the result can be found in \Cref{sec:experiments}. 
    Finally, we end the paper with a conclusion in \Cref{sec:conclusion}.

%% file: sections/preliminaries.tex
\section{Preliminaries} \label{sec:preliminaries}
    
\noindent
\textbf{Coordinates on \(SE(2)\) and \(\bbM_2\).} Let \(G = SE(2) = \bbR^2 \rtimes SO(2)\) be the two-dimensional rigid body motion group. We identify elements \(g \in G\) with \(g \equiv (x,y,\theta) \in \bbR^2 \times \bbR/(2\pi \mathbb{Z})\), via the isomorphism \(SO(2) \cong \bbR/(2\pi\mathbb{Z})\). Furthermore, we always use the small-angle identification \( \bbR/(2\pi \mathbb{Z}) = [-\pi, \pi)\). 

For \(g_1=(x_1, y_1, \theta_1)\), \(g_2 = (x_2, y_2, \theta_2) \in SE(2)\) we have the group product
\begin{equation*}
\begin{split}
    g_1 g_2 := (
        &x_1 + x_2 \cos \theta_1 - y_2 \sin \theta_1, \\
        &y_1 + x_2 \sin \theta_1 + y_2 \cos \theta_1, \\
        &
        \theta_1 + \theta_2 \omod 2\pi
        ),
\end{split}
\end{equation*}
and the identity is \(e = (0,0,0)\). The rigid body motion group acts on the homogeneous space of two-dimensional positions and orientations \(\bbM_{2} = \bbR^2 \times S^1 \subseteq \bbR^2 \times \bbR^2\) by the left-action \(\odot\):
\begin{equation*}
    (\bx,\bR) \odot (\by,\bn)= (\bx + \bR\by,\bR\bn), 
\end{equation*}
with \((\bx,\bR) \in SE(2)\) and \((\by,\bn) \in \bbM_2\). If context allows it we may omit writing \(\odot\) for conciseness. By choosing the reference element \(\bp_0 = (0,0,(1,0)) \in \bbM_2\) we have:
\begin{equation} \label{eq:one_to_one}
    (x,y,\theta) \odot \bp_0 = (x,y,(\cos \theta, \sin \theta)).
\end{equation}
This mapping is a diffeomorphism and allows us to identify \(SE(2)\) and \(\bbM_2\). Thereby we will also freely use the \((x,y,\theta)\) coordinates on \(\bbM_2\). 

\noindent
\textbf{Morphological group convolution.} Given functions \(f_1,f_2:\bbM_2 \to \bbR\)  we define their morphological convolution (or `infimal convolution') \cite{schmidt2016morphological,boomgaard1994morphological} by
\begin{equation} \label{eq:morphological_convolution}
    (f_1 \mathbin{\square} f_2)(\bp)= \inf \limits_{g \in G} \left\{f_1(g^{-1} \bp) + f_2(g \, \bp_0)\right\}
\end{equation}

\noindent
\textbf{Left-invariant (co-)vector fields on \(\bbM_2\). } Throughout this paper we shall rely on the following basis of left-invariant vector fields:
\begin{equation*}
\begin{split}
    \cA_{1} &= \cos \theta \pd_x + \sin \theta \pd_y, \\
    \cA_{2} &= -\sin \theta \pd_x + \cos \theta \pd_y, \textrm{ and }\\ \cA_{3} &= \pd_{\theta}.
\end{split}
\end{equation*}
The dual frame \(\omega^i\) is given by \(\langle \omega^i, \cA_{j}\rangle =\delta^{i}_j\), i.e:
\begin{equation*}
\begin{split}
    \omega^1 &=  \cos \theta {\rm d}x + \sin \theta {\rm d}y, \\
    \omega^2 &= -\sin \theta {\rm d}x  +\cos \theta {\rm d}y, \textrm{ and } \\
    \omega^3 &= {\rm d}\theta.
\end{split}
\end{equation*}

\noindent
\textbf{Metric tensor fields on \(\bbM_2\). } We consider the following left-invariant metric tensor fields:
\begin{equation} \label{eq:diagonal_metric} 
\begin{split}
    \cG = \sum_{i=1}^{3} w_i^2 \ \omega^{i} \otimes \omega^i
\end{split} 
\end{equation}
and write \(\Vert\dot{\bp}\Vert=\sqrt{\cG_{\bp}(\dot{\bp},\dot{\bp})}\). Here \(w_i > 0\) are the metric parameters. We also use the dual norm \(\Vert\hat \bp\Vert_* = \sup \limits_{\dot \bp \in T_\bp \bbM_2} \frac{\Ang{\dot \bp, \hat \bp}}{\Vert\dot \bp\Vert}\). We will assume, without loss of generality, that \(w_2 \geq w_1\) and introduce the ratio 
\begin{equation} \label{eq:spatial_anisotropy}
    \zeta := \frac{w_2}{w_1} \geq 1
\end{equation}
that is called the \textit{spatial anisotropy} of the metric.
\noindent
\textbf{Distances on \(\bbM_2\).} The left-invariant metric tensor field \(\cG\) on \(\bbM_2\) induces a left-invariant distance (`Riemannian metric') \(d:\mathbb{M}_{2} \times\mathbb{M}_2 \to \mathbb{R}_{\geq0}\) by
\begin{equation} \label{eq:riemannian_distance}
    d_{\cG}(\bp,\bq)=
    \inf_{\g \in \Gamma_t(\bp,\bq)}\Par{L_{\cG}(\g) := \int_0^t \Vert\dot{\g}(s)\Vert_{\cG}\, {\rm d}s},
\end{equation}
where \(\Gamma_t(\bp, \bq)\) is the set piecewise $C^1$-curves \(\g\) in \(\bbM_2\) with \(\g(0)=\bp\) and \(\g(t)=\bq\). The right-hand side does not depend on \(t>0\), and we may set \(t=1\).

If no confusion can arise we omit the subscript \(\cG\) and write \(d, L, \Vert \cdot \Vert\) for short. The distance being left-invariant means that for all \(g\in SE(2)\), \(\bp_1,\bp_2 \in \bbM_2\) one has \(d(\bp,\bq)=d(g \bp,g \bq)\). We will often use the shorthand notation \(d(\bp):=d(\bp, \bp_0)\). 

We often consider the sub-Riemannian case arising when \(w_2 \to \infty\). Then we have ``infinite cost'' for sideways motion and the only ``permissible'' curves \(\gamma\) are the ones for which \(\dot \g(t) \in H\) where \(H := \textrm{span}\{\cA_1, \cA_3\} \subset T\bbM_{2}\). This gives rise to a new notion of distance, namely the sub-Riemannian distance \(d_{sr}\):
\begin{equation} \label{eq:sub_riemannian_distance}
    d_{sr}(\bp,\bq)=
    \inf_{\substack{\g \in \Gamma_t(\bp,\bq), \\ \dot \g \in H}} L_{\cG}(\g).
\end{equation}
One can show rigorously that when \(w_2 \to \infty\) the Riemannian distance \(d\) tends to the sub-Riemannian distance \(d_{sr}\), see for example \cite[Thm.2]{duits2018optimal}. 

\noindent
\textbf{Exponential and Logarithm on \(SE(2)\).}
The exponential map $\exp(c^1 \pd_x \vert_e + c^2 \pd_y \vert_e + c^3 \pd_\theta \vert_e) = (x,y,\theta) \in SE(2)$ is given by:
\begin{equation*}
\begin{split}
    x &= \Par{c^1 \cos \tfrac{c^3}{2} - c^2 \sin \tfrac{c^3}{2}}\sinc \tfrac{c^3}{2}, \\
    y &= \Par{c^1 \sin \tfrac{c^3}{2} + c^2 \cos \tfrac{c^3}{2}}\sinc \tfrac{c^3}{2}, \\
    \theta &= c^3 \omod 2\pi.
\end{split}
\end{equation*}
And the logarithm: 
$\log (x,y,\theta) = c^1 \pd_x\vert_e + c^2 \pd_y\vert_e + c^3 \pd_\theta\vert_e \in T_eSE(2)$:
\begin{equation} \label{eq:logarithm_coordinates}
\begin{split}
    c^1 &= \frac{x\cos \tfrac{\theta}{2} + y \sin \tfrac{\theta}{2}}{\sinc \tfrac{\theta}{2}}, \\
    c^2 &= \frac{-x \sin\tfrac{\theta}{2} + y \cos \tfrac{\theta}{2}}{\sinc\tfrac{\theta}{2}}, \\
    c^3 &= \theta.
\end{split}
\end{equation}
By virtue of \cref{eq:one_to_one} we will freely use the logarithm coordinates on \(\bbM_2\).

%% file: sections/morphological.tex
\section{Erosion and Dilation} \label{sec:morphological}
    
We will be considering the following Hamilton-Jacobi equation on \(\bbM_2\):
\begin{equation} \label{eq:morphss}
\begin{cases}
    \pdv{W_\a}{t} &= \pm \frac{1}{\a} \Norm{\nabla W_{\a}}^\a = \pm \cH_{\a}(dW_\a) \\
    \At{W_\a}_{t=0} &= U, 
\end{cases}
\end{equation}
with the Hamiltonian \(\cH_\a : T^*\bbM_2 \to \bbR_{\geq 0}\):
\begin{equation*}
    \cH_{\a}(\hat{\bp}) 
    = \cH_{\a}^{1D}(\Vert\hat{\bp}\Vert) 
    = \frac{1}{\a}\Vert\hat{\bp}\Vert_*^{\a}, 
\end{equation*}
and where \(W_\a\) the viscosity solutions \cite{evans2010partial} obtained from the initial condition \(U \in C( \bbM_{2},\bbR)\). Here the \(+\)sign is a dilation scale space and the \(-\)sign is an erosion scale space \cite{schmidt2016morphological,boomgaard1994morphological}. If confusion cannot arise, we omit the superscript \(1D\). Erosion and dilation correspond to min- and max-pooling, respectively. The Lagrangian \(\cL_\a : T\bbM_2 \to \bbR_{\geq 0}\) corresponding with this Hamiltonian is obtained by taking the Fenchel transform of the Hamiltonian:
\begin{equation*}
    \cL_{\a}(\dot{\bp})
    = \cL^{1D}_{\a}(\Vert\dot{\bp}\Vert) 
    =\frac{1}{\beta} \Vert\dot{\bp}\Vert^\beta
\end{equation*}
with \(\beta\) such that \(\frac{1}{\a} + \frac{1}{\beta} = 1\). Again, if confusion cannot arise, we omit the subscript \(\a\) and/or superscript \(1D\). We deviate from our previous work by including the factor \(\frac{1}{\a}\) and working with a power of \(\a\) instead of \(2\a\). We do this because it simplifies the relation between the Hamiltonian and Lagrangian.

The following proposition collects standard results in terms of the solutions of Hamilton-Jacobi equations on manifolds \cite{diop2021extension,fathi2007weakkam,azagra2005nonsmooth}, thereby generalizing results on \(\bbR^2\) to \(\bbM_2\). 
\begin{proposition} 
    \label{res:kernels} 
    \textbf{\emph{\mbox{(Solution erosion \& dilation})}} \\[6pt]
    Let \(\a > 1\). The viscosity solution \(W_\a\) of the erosion PDE \eqref{eq:morphss} is given by
    \begin{align}
        W_\a(\bp,t) 
        &= \inf_{\substack{\bq \in \bbM_2, \\ \g \in \Gamma_t(\bp, \bq)}} U(\bq) + \int \limits_0^{t} \cL_{\a}(\dot{\g}(s))\, {\rm d}s \label{eq:weakKAM_solution}\\
        &= \inf_{\bq \in \bbM_2} U(\bq) + t \cL^{1D}_\a (d(\bp, \bq)/t) \label{eq:distance_solution}\\
        &=(k_t^{\a} \mathbin{\square} U)(\bp) \label{eq:convolution_solution}
    \end{align}
    where the morphological kernel \(k_t^{\a}: \bbM_{2} \to \bbR_{\geq 0}\) is defined as:
    \begin{equation} \label{eq:morphological_kernel}
        k_{t}^{\a}= t \cL^{1D}_\a (d/t) = \frac{t}{\beta} \Par{\frac{d(\bp_0, \cdot)}{t}}^\beta.
    \end{equation}
    Furthermore, the Riemannian distance \(d:=d(\bp_0,\cdot)\) is the viscosity solution of the eikonal PDE 
    \begin{equation} \label{eik}
        \Norm{\nabla d}^2 = \sum_{i=1}^3 (\cA_{i} d / w_i)^2=1
    \end{equation}
    with boundary condition $d(\bp_0)=0$. 
    Likewise the viscosity solution of the dilation PDE is
    \begin{equation} \label{eq:dilation_solution}
        W_{\a}(\bp,t)=-(k_t^{\a} \mathbin{\square} -U)(\bp)
    \end{equation}
\end{proposition}

\begin{proof}
    It is shown by Fathi in \cite[Prop.5.3]{fathi2007weakkam} that \eqref{eq:weakKAM_solution} is a viscosity solution of the Hamilton-Jacobi equation \eqref{eq:morphss} on a complete connected Riemannian manifold without boundary, under some (weak) conditions on the Hamiltonian and with the initial condition \(U\) being Lipschitz. In \cite[Thm.2]{diop2021extension} a similar statement is given but only for compact connected Riemannian manifolds, again under some weak conditions on the Hamiltonian but without any on the initial condition. Next, we employ these existing results and provide a self-contained proof of \eqref{eq:distance_solution} and \eqref{eq:convolution_solution}.
    
    Because we are looking at a specific class of Lagrangians, the solutions can be equivalently written as \eqref{eq:distance_solution}. In \cite[Prop.2]{diop2021extension} this form can also be found. Namely, the Lagrangian \(\cL_\a^{1D}\) is convex for \(\a > 1\), so for any curve \(\g \in \Gamma_t := \Gamma_t(\bp, \bq)\) we have by direct application of Jensen's inequality (omitting the superscript \(1D\)):
    \begin{equation*}
        \cL_\a \Par{ \frac{1}{t} \int_0^t \Vert\dot \g(s)\Vert {\rm d}s} \leq \frac{1}{t} \int_0^t \cL_\a(\Vert\dot \g(s)\Vert)\ \rm{d}s,
    \end{equation*}
    with equality if \(\Vert\dot \g\Vert\) is constant. This means that:
    \begin{equation} \label{eq:jensens2}
        \inf_{\g \in \Gamma_t} t \cL_\a \Par{ \frac{L(\g)}{t} } \leq \inf_{\g \in \Gamma_t} \int_0^t \cL_\a(\Vert\dot \g(s)\Vert)\ {\rm d}s,
    \end{equation}
    where \(L(\g):=L_{\mathcal{G}}(\gamma)\), recall (\ref{eq:riemannian_distance}), is the length of the curve \(\g\). Consider the subset of curves with constant speed \(\tilde \Gamma_t = \{ \g \in \Gamma_t \mid \Vert\dot \g\Vert = L(\g)/t\} \subset \Gamma_t\). Optimizing over a subset can never decrease the infimum so we have:
    \begin{equation*}
        \inf_{\g \in \Gamma_t} \int_0^t \cL_\a(\Vert\dot \g(s)\Vert) {\rm d}s \leq \inf_{\g \in \tilde \Gamma_t} \int_0^t \cL_\a\Par{\frac{L(\g)}{t}} {\rm d}s
    \end{equation*}
    The r.h.s of this equation is equal to the l.h.s of equation \eqref{eq:jensens2} as the length of a curve is independent of its parameterization. Thereby we have equality in \eqref{eq:jensens2}.
    By monotonicity of \(\cL_\a\) on \(\bbR_{>0}\) we may then conclude that:
    \begin{equation*}
    \begin{split}
        \inf_{\g \in \Gamma_t} t \cL_{\a} \Par{ L(\g)/t } 
        &= t \cL_{\a} \Par{ \inf_{\g \in \Gamma_t} L(\g)/t } \\
        &= t \cL_{\a} (d(\bp, \bq)/t).
    \end{split}
    \end{equation*}
    
    That we can write the solution as \eqref{eq:convolution_solution} is a consequence of the left-invariant metric on the manifold. A similar derivation can be found in \cite[Thm.30]{smets2022pdebased}:
    \begin{equation*}
    \begin{split}
        W_\a(\bp,t) 
        &= \inf_{\bq \in \bbM_2} U(\bq) + t \cL_\a(d(\bp, \bq)/t) \\
        &= \inf_{g \in G} U(g \bp_0) + t \cL_\a(d(\bp, g \bp_0)/t) \\
        &= \inf_{g \in G} U(g \bp_0) + t \cL_\a(d(g^{-1} \bp, \bp_0)/t) \\
        &= \inf_{g \in G} U(g \bp_0) + k_t^\a(g^{-1} \bp) \\
        &= (k_t^\a \mathbin{\square} U)(\bp)
    \end{split}
    \end{equation*}

    It is shown in \cite[Thm.6.24]{azagra2005nonsmooth} for complete connected Riemannian manifolds that the distance map \( d(\bp) \) is a viscosity solution of the Eikonal equation \eqref{eik}.
    
    Finally, solutions of erosion and dilation PDEs correspond to each other. If \(W_\a\) is the viscosity solution of the erosion PDE with initial condition \(U\), then \(-W_\a\) is the viscosity solution of the dilation PDE, with initial condition \(-U\). This means that the viscosity solution of the dilation PDE is given by \eqref{eq:dilation_solution}.
\end{proof}


%% file: sections/distance.tex
\section{Distance Approximations} \label{sec:distances}

To calculate the morphological kernel \(k_t^\a\) \eqref{eq:morphological_kernel} we need the the exact Riemannian distance \(d\) \eqref{eq:riemannian_distance}, but calculating this is computationally demanding. To alleviate this problem we approximate the exact distance \(d(\bp_0, \cdot)\) with approximative distances, denoted with \(\rho: \bbM^2 \to \bbR_{\geq 0}\), which are computationally cheap. To this end we define the logarithmic distance approximation \(\rho_c\), as explained in \cite[Def.19]{smets2022pdebased} and \cite[Def.6.1.2]{lupi2021kernel}, by
\begin{equation} \label{eq:logarithmic_distance_approximation}
    \rho_c := \sqrt{ (w_1 c^1)^2 + (w_2 c^2 )^2 + (w_3 c^3 )^2}.
\end{equation}

Note that all approximative distances \(\rho\) can be extended to something that looks like a metric on \(\bbM_2\). For example we can define:
\begin{equation*}
    \rho(g_1 \bp_0,\ g_2 \bp_0) := \rho(g_1^{-1} g_2 \bp_0).
\end{equation*}
But this is almost always not a true metric in the sense that it does not satisfy the triangle inequality. So in this sense an approximative distance is not necessarily a true distance. However, we will keep referring to them as approximative distances as we only require them to look like the exact Riemannian distance \(d(\bp_0, \cdot)\).

As already stated in the introduction, Riemannian distance approximations such as \(\rho_c\) begin to fail in the high spatial anisotropy cases \(\zeta \gg 1\). For these situations we need sub-Riemannian distance approximations. In previous literature two such sub-Riemannian approximations are suggested. The first one is standard \cite[Sec.6]{terelst1998weighted}, the second one is a modified smooth version \cite[p.284]{duits2010leftinvariant}, also seen in \cite[eq.14]{bekkers2018nilpotent}:
\begin{align} 
    &\sqrt{ \sqrt{\nu w_1^2w_3^2}\Abs{ c^2}  + (w_1 c^1)^2 + (w_3 c^3)^2 } \label{eq:approximative_sub_riemannian_distance_1}\\
    &\sqrt[4]{\nu w_1^2w_3^2 \Abs{ c^2} ^2 + ((w_1 c^1)^2 + (w_3 c^3)^2)^2} \label{eq:approximative_sub_riemannian_distance_2}
\end{align}
In \cite{bekkers2018nilpotent} \(\nu \approx 44\) is empirically suggested. Note that the sub-Riemannian approximations rely on the assumption that \(w_2 \geq w_1\).

However, they both suffer from a major shortcoming in the
interaction between \(w_3\) and \(c^2\). When we let \(w_3 \to 0\) both approximations suggest that it becomes arbitrarily cheap to move in the \(c^2\) direction which is undesirable as this deviates from the exact distance \(d\): moving spatially will always have a cost associated with it determined by at least \(w_1\).

To make a proper sub-Riemannian distance estimate we will use the Zassenhaus formula, which is related to the Baker–Campbell–Hausdorff formula:
\begin{equation}
    e^{t(X + Y)} = e^{tX} e^{tY} e^{-\frac{t^2}{2} \Bra{X,Y}} e^{\cO(t^3)} \dots,
\end{equation}
where we have used the shorthand \(e^x := \exp(x)\). Filling in \(X = A_1\) and \(Y = A_3\) and neglecting the higher order terms gives:
\begin{equation}
    e^{t(A_1 + A_3)} \approx e^{tA_1} e^{tA_3} e^{\frac{t^2}{2} A_2},
\end{equation}
or equivalently:
\begin{equation}
    e^{\frac{t^2}{2} A_2} \approx e^{-tA_3} e^{-tA_1} e^{t(A_1 + A_3)}.
\end{equation}
This formula says that one can successively follow exponential curves in the ``legal'' directions \(\cA_1\) and \(\cA_3\) to effectively move in the ``illegal'' direction of \(\cA_2\). Taking the lengths of these curves and adding them up gives an approximative upper bound on the sub-Riemannian distance:
\begin{equation}
\begin{split}
    d_{sr}(e^{\frac{t^2}{2} A_2}) 
    &\lessapprox \Par{w_1 + w_3 +  \sqrt{w_1^2 + w_3^2}} \Abs{t} \\
    &\leq 2\Par{w_1 + w_3} \Abs{t}.
\end{split}
\end{equation}
Substituting \(t \to \sqrt{2\Abs{t}}\) gives:
\begin{equation}
    d_{sr}(e^{tA_2}) \lessapprox 2\sqrt{2}\Par{w_1 + w_3} \sqrt{\Abs{t}}.
\end{equation}
This inequality, together with the smoothing trick to go from \eqref{eq:approximative_sub_riemannian_distance_1} to \eqref{eq:approximative_sub_riemannian_distance_2}, inspires then the following sub-Riemannian distance approximation:
\begin{equation} \label{eq:new_sub_riemannian_distance}
    \scalemath{0.8}{ 
        \rho_{c, sr} := \sqrt[4]{ \Par{\nu (w_1 + w_3)}^4 \Abs{ c^2} ^2 + ((w_1 c^1)^2 + (w_3 c^3)^2)^2} 
    },
\end{equation}
for some \(0<\nu<2\sqrt{2}\) s.t. the approximation is tight. We empirically suggest \(\nu \approx 1.6\), based on a numerical analysis that is tangential to \cite[Fig.3]{bekkers2018nilpotent}. Notice that this approximation does not break down when we let \(w_3 \to 0\).

Furthermore, in view of contraction of  
$SE(2)$ to the Heisenberg group \(H_3\) \cite[Sec.5.2]{duits2010leftinvariant}, and the exact fundamental solution \cite[Eq.27]{duits2007scale} of the Laplacian on $H_3$  (where the norm $\rho_{c,sr}$  appears squared in the numerator with $1=w_1=w_3=\nu$) we expect $\nu \geq 1$. 

\Cref{tab:old_subriemannian_appropriate} shows that both the old sub-Riemannian approximation \eqref{eq:approximative_sub_riemannian_distance_2} and new approximation \eqref{eq:new_sub_riemannian_distance} are appropriate in cases such as \(w_3=1\). \Cref{tab:old_subriemannian_not_appropriate} shows that the old approximation breaks down when we take \(w_3 = 0.5\), and that the new approximation behaves more appropriate.

\renewcommand{\figwidth}{0.2\linewidth}
\begin{table*}
    \centering
    \begin{tabular}{c|c|c}
        \(d\) & \(\rho_{b,sr,old}\) & \(\rho_{b,sr}\) \\
        \hline \hline
        
        \includegraphics[width=\figwidth]{figures/clean_exact_w=1,8,1_svg-tex.pdf} 
        & \includegraphics[width=\figwidth]{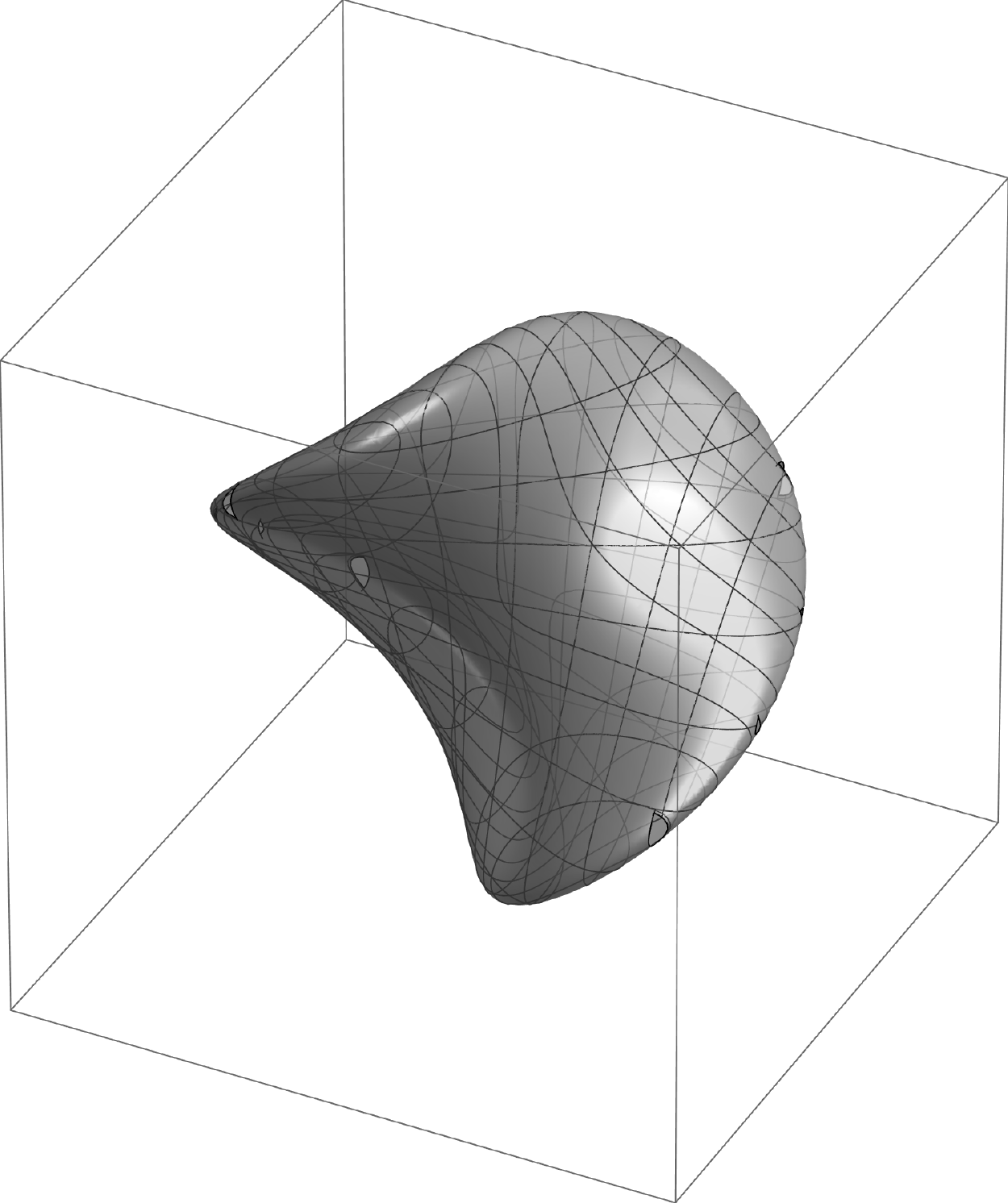}
        & \includegraphics[width=\figwidth]{figures/clean_rho_b_sr_w=1,8,1_svg-tex.pdf}\\ 
         
        \includegraphics[width=\figwidth]{figures/clean_flat_exact_w=1,8,1_svg-tex.pdf} 
        & \includegraphics[width=\figwidth]{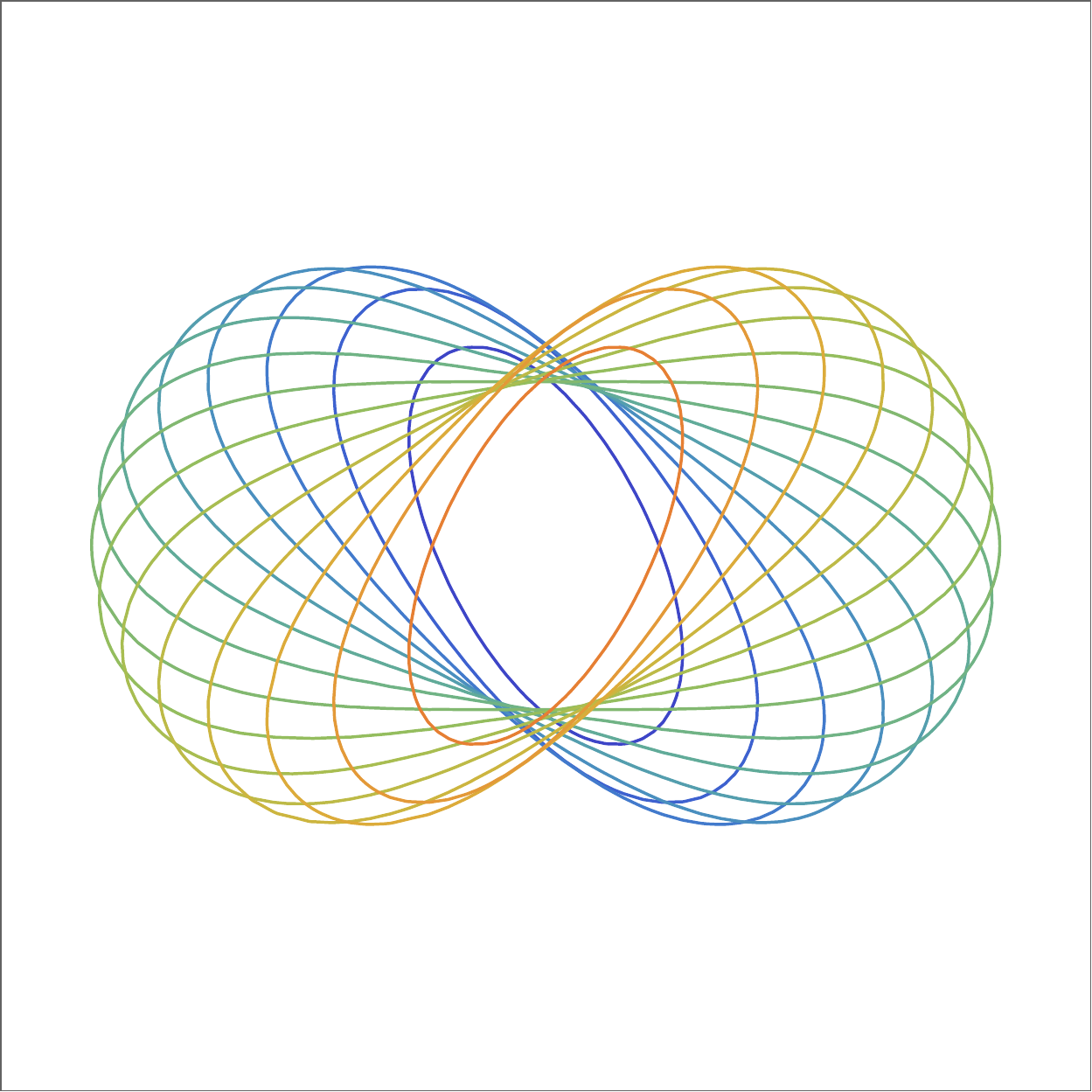} 
        & \includegraphics[width=\figwidth]{figures/clean_flat_rho_b_sr_w=1,8,1_svg-tex.pdf}
    \end{tabular}
    \caption{Same situation and metric parameters as \Cref{tab:balls_high_anisotropy}, i.e. \(w_1 = w_3 = 1\) and \(w_2 = 8\). We see the exact distance \(d\) alongside the old sub-Riemannian approximation \(\rho_{b,sr,old}\) \eqref{eq:approximative_sub_riemannian_distance_2} and new approximation \(\rho_{b,sr}\) \eqref{eq:new_sub_riemannian_distance}. For the old approximation we chose \(\nu=44\), as suggested in \protect\cite{bekkers2018nilpotent}, and for the new one \(\nu = 1.6\). We see that in this case both approximations are appropriate.}
    \label{tab:old_subriemannian_appropriate}
\end{table*}

\renewcommand{\figwidth}{0.2\linewidth}
\begin{table*}
    \centering
    \begin{tabular}{c|c|c}
        \(d\) & \(\rho_{b,sr,old}\) & \(\rho_{b,sr}\) \\
        \hline \hline
        
        \includegraphics[width=\figwidth]{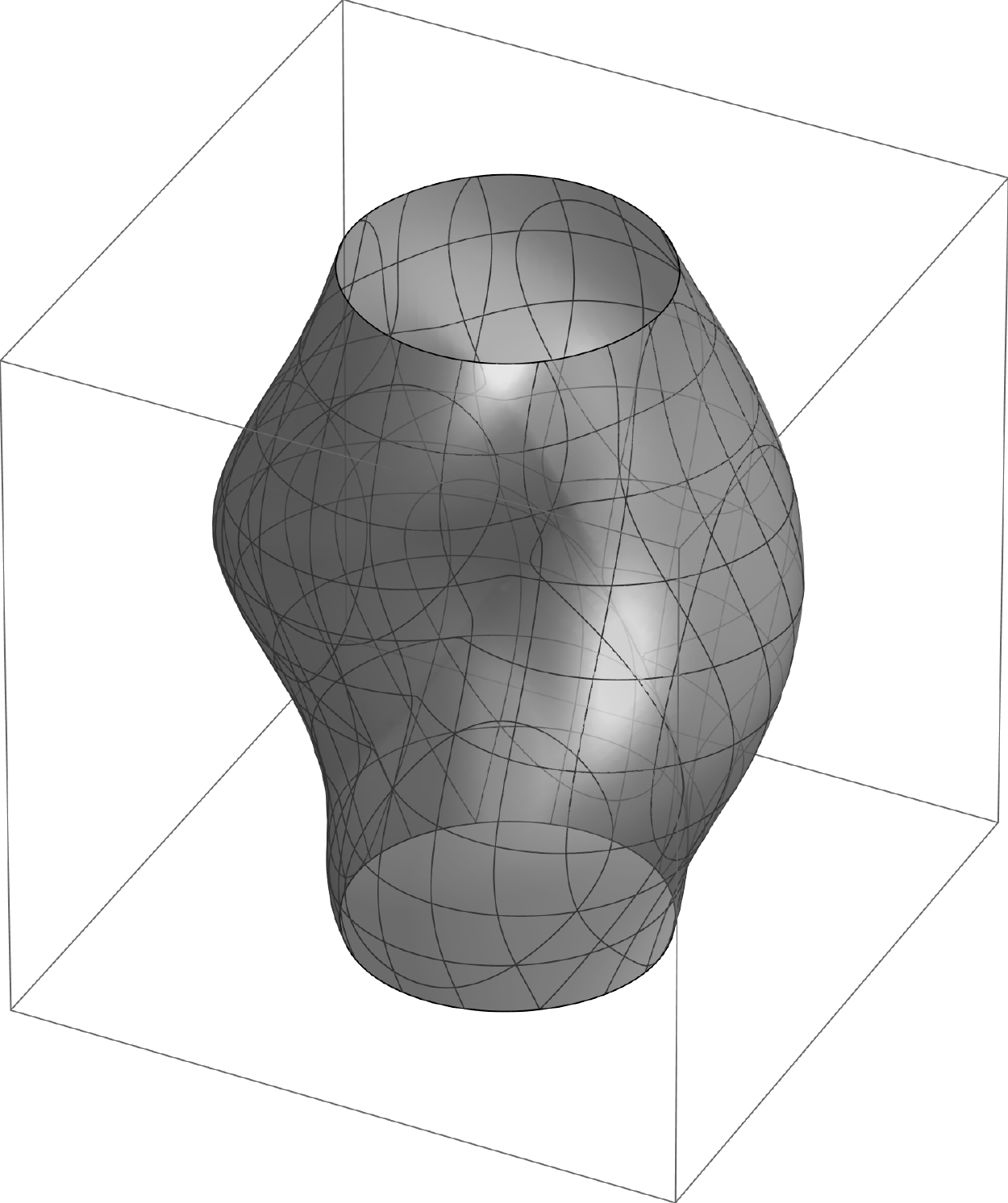} 
        & \includegraphics[width=\figwidth]{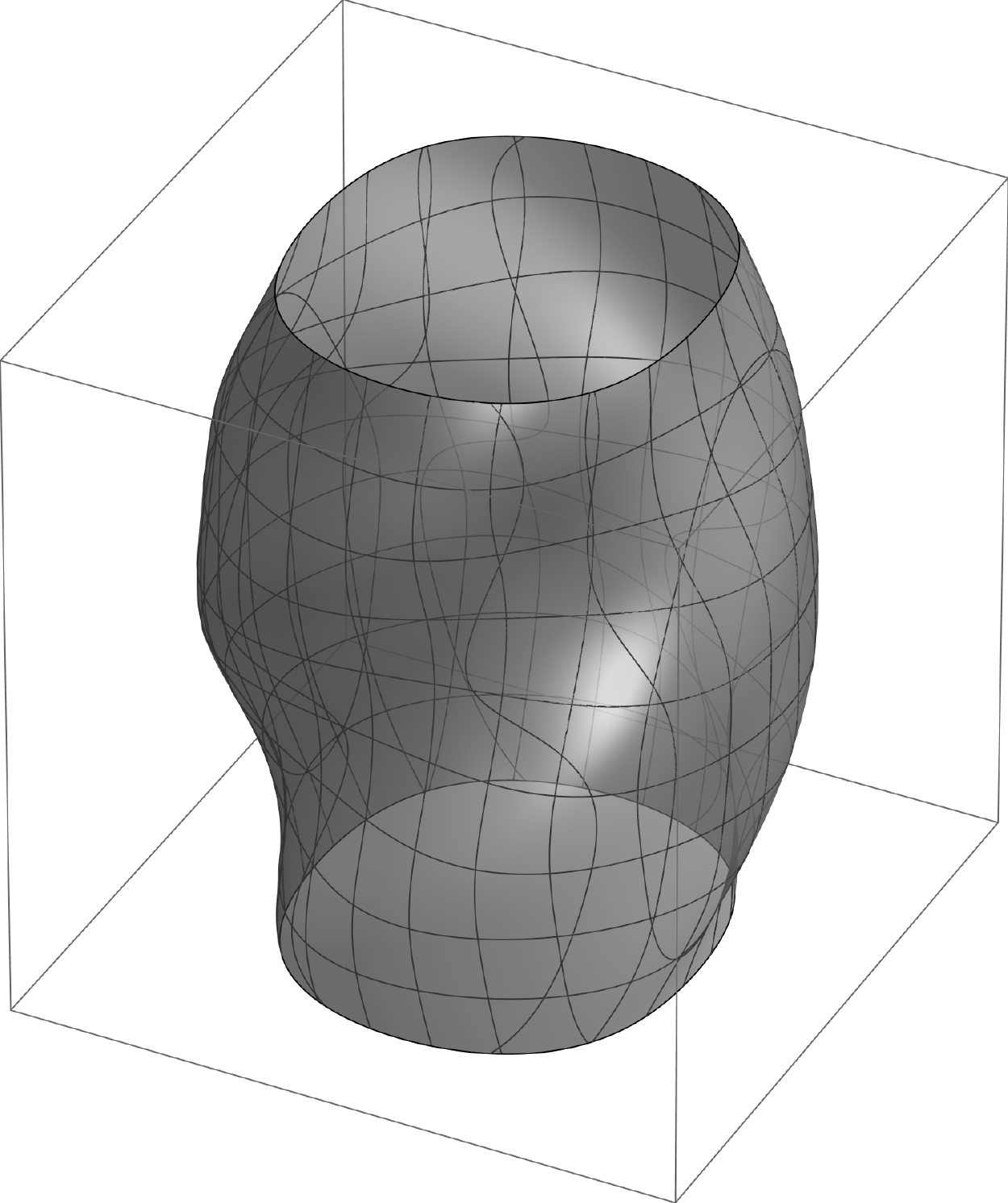}
        & \includegraphics[width=\figwidth]{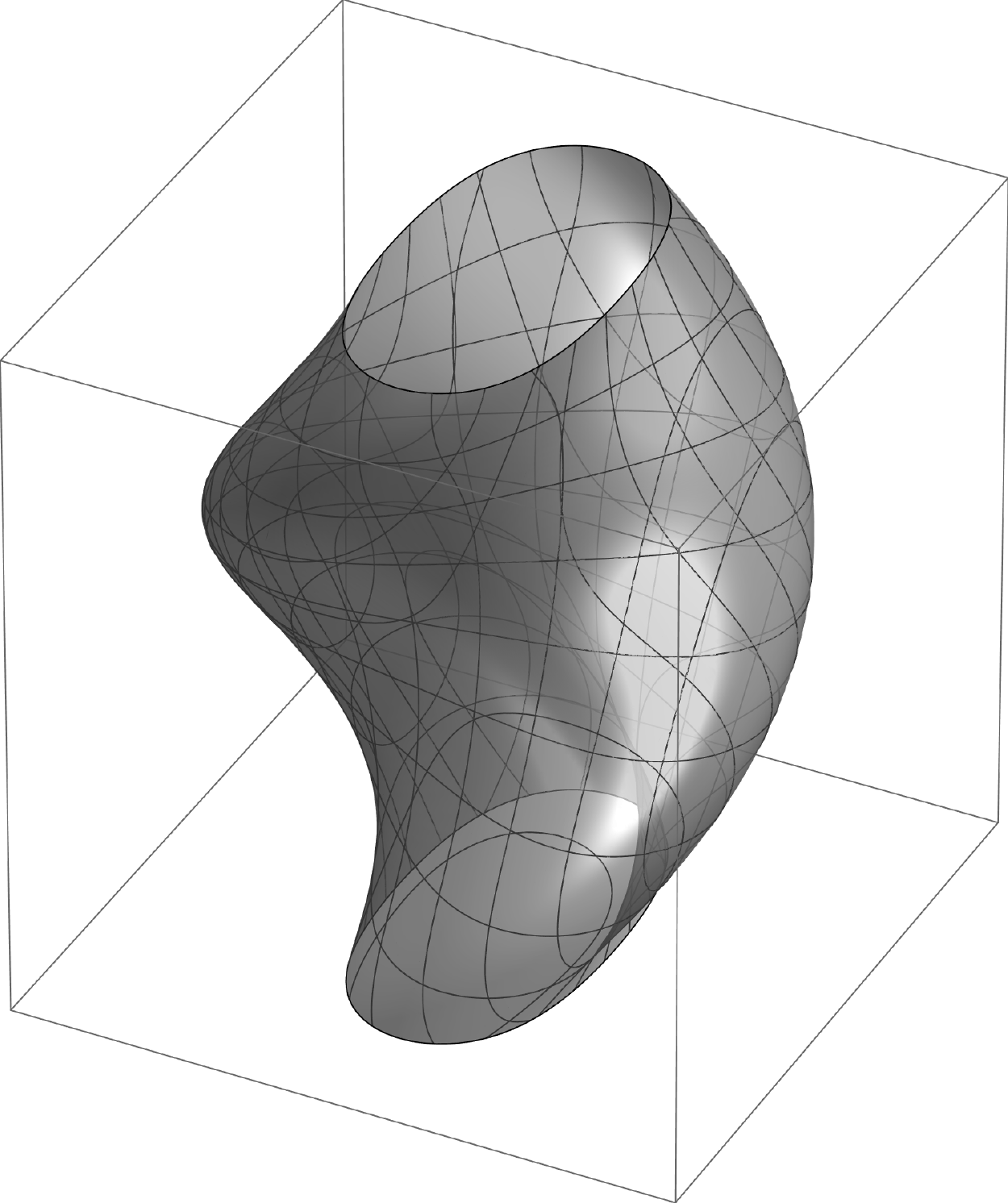}\\ 
         
        \includegraphics[width=\figwidth]{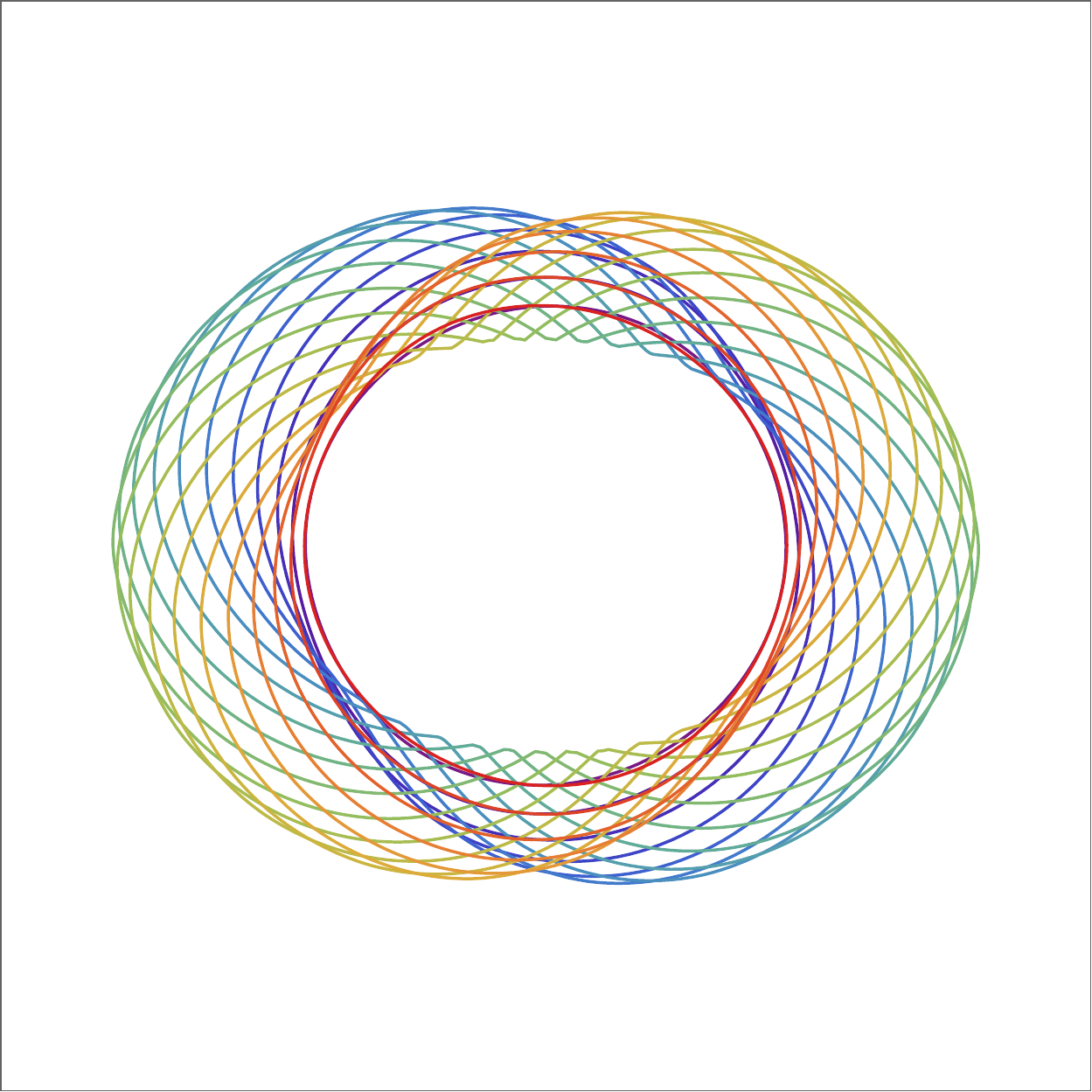} 
        & \includegraphics[width=\figwidth]{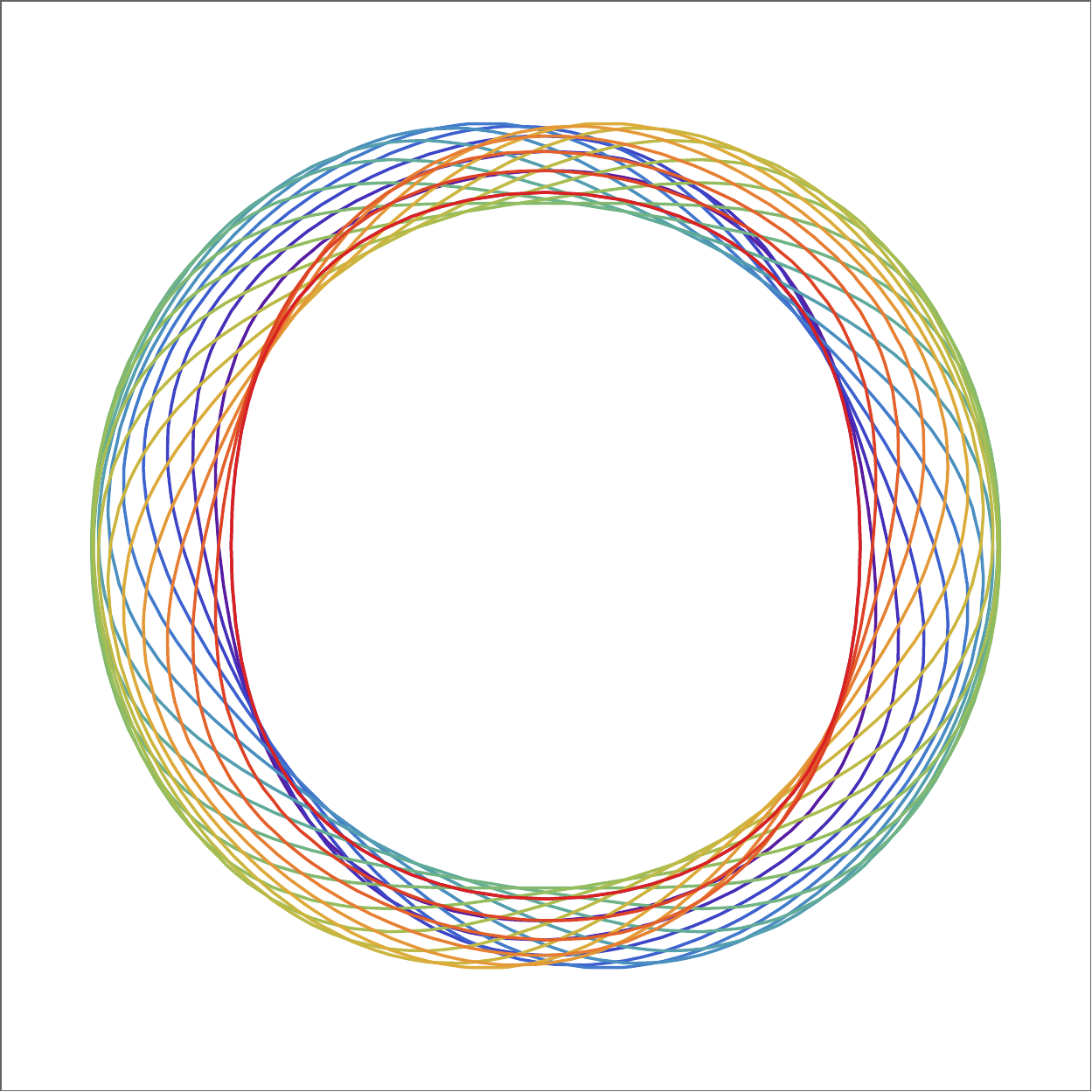} 
        & \includegraphics[width=\figwidth]{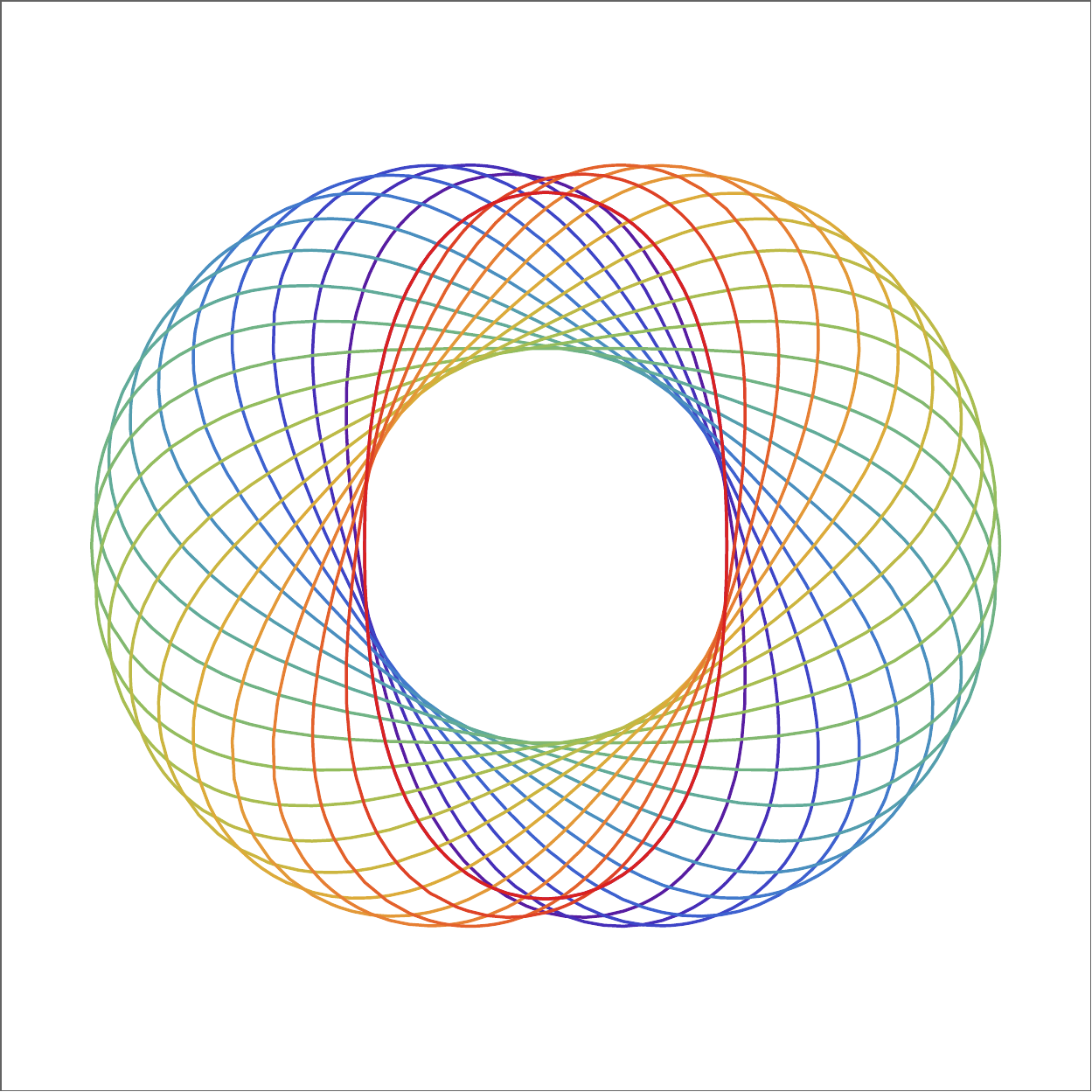}
    \end{tabular}
    \caption{Same as \Cref{tab:old_subriemannian_appropriate} but then with \(w_1 = 1, w_2 = 8, w_3 = 0.5\). We see that in this case that the old sub-Riemannian approximation \(\rho_{b,sr,old}\) \eqref{eq:approximative_sub_riemannian_distance_2} underestimates the true distance and becomes less appropriate. The new approximation \eqref{eq:new_sub_riemannian_distance} is also not perfect but qualitatively better. Decreasing \(w_3\) would exaggerate this effect even further.}
    \label{tab:old_subriemannian_not_appropriate}
\end{table*}

The Riemannian and sub-Riemannian approximations can be combined into the following newly proposed practical approximation:
\begin{equation} \label{eq:combined_distance_approximation}
    \rho_{c,com} := \max(l,\ \min(\rho_{c, sr}\ ,\ \rho_{c})),
\end{equation} 
where \(l : \bbM_2 \to \bbR\) is given by:
\begin{equation} \label{eq:lower_bound_first}
    l := \sqrt{ (w_1 x)^2 + (w_1 y)^2 + (w_3 \theta)^2 },
\end{equation}
for which will we show that it is a lower bound of the exact distance \(d\) in \Cref{res:global_bounds}. 

The most important property of the combined approximation is that is automatically switches between the Riemannian and sub-Riemannian approximations depending on the metric parameters. Namely, the Riemannian approximation is appropriate very close to the reference point \(\bp_0\), but tends to overestimate the true distance at a moderate distance from it. The sub-Riemannian approximation is appropriate at moderate distances from \(\bp_0\), but tends to overestimate very close to it, and underestimate far away. The combined approximation is such that we get rid of the weaknesses that the approximations have on their own.

On top of these approximative distances, we also define \(\rho_b\), \(\rho_{b,sr}\), and \(\rho_{b,com}\) by replacing the logarithmic coordinates \(c^i\) by their corresponding half-angle coordinates \(b^i\) defined by:
\begin{equation}\label{eq:half_angle_coordinates}
\begin{split}
    b^1 &= x \cos \tfrac{\theta}{2} + y \sin \tfrac{\theta}{2}, \\
    b^2 &= -x \sin \tfrac{\theta}{2} + y \cos \tfrac{\theta}{2}, \\
    b^3 &= \theta.
\end{split}
\end{equation}
So, for example, we define \(\rho_b\) as:
\begin{equation} \label{eq:def_rho_b}
    \rho_b := \sqrt{(w_1 b^1)^2 + (w_2 b^2)^2 + (w_3 b^3)^2}.    
\end{equation}
Why we use these coordinates will be explained in \Cref{sec:symmetry}. 

We can define approximative morphological kernels by replacing the exact distance in \eqref{eq:morphological_kernel} by any of the approximative distances in this section. To this end we, for example, define \(k_b\) by replacing the exact distance in the morphological kernel \(k\) by \(\rho_b\):
\begin{equation} \label{eq:half_angle_kernel}
    k_{b,t}^\a := \frac{t}{\beta} \Par{ \frac{\rho_b}{t} } ^\beta,
\end{equation} 
where we recall that
\(\frac{1}{\a} + \frac{1}{\beta} = 1\) and $\alpha>1$.

%% file: sections/analysis.tex
\section{Main Theorem and Analysis} \label{sec:analysis}

When the effect of erosion and dilation is calculated with an approximative morphological kernel an error is made.
We are are therefor interested in analyzing the behaviour of this error.
We do this by comparing the approximative morphological kernels with the exact kernel \(k_t^\a\) \eqref{eq:morphological_kernel}.
The result of our analysis is summarized in the following theorem.
Because there is no time \(t\) dependency in all the inequalities of our main result we use short notation \(k^\a := k_t^\a\), \(k_b^\a := k_{b,t}^\a\).
\begin{theorem} \textbf{\emph{(Quality of approximative morphological kernels)}} \\[7pt] 
    \label{res:main_results}
    Let \(\zeta := \frac{w_2}{w_1}\) denote the spatial anisotropy, and let \(\beta\) be such that \(\frac{1}{\a} + \frac{1}{\beta} = 1\), for some $\alpha>1$ fixed.
    We assess the quality of our approximative kernels in three ways:
    \begin{itemize}
        \item The exact and all approximative kernels have the same symmetries, see \Cref{tab:symmetries_half_angle}.
        \item Globally it holds that:
        \begin{equation} \label{eq:kernel_global_bounds}
            \zeta^{-\beta} k^\a \leq k_b^\a \leq \zeta^{\beta} k^\a,
        \end{equation}
        from which we see that in the case \(\zeta = 1\) we have that \(k^\a_b\) is exactly equal to \(k^\a\).
        \item Locally around \footnote{for a precise statement see \Cref{res:asymptotics} and Remark~\ref{rem:tolerance_bound}.} \(\bp_0\) we have:
        \begin{equation} \label{eq:kernel_local_bound}
            k_b^\a \leq (1 + \varepsilon)^{\beta/2} k^\a.
        \end{equation}
        where
        \begin{equation}
            \e := \frac{\zeta^2 - 1}{2 w_3^2} \zeta^4 \rho_b^2  + \cO(\Abs{\theta}^3).
        \end{equation}
    \end{itemize}
\end{theorem}
\begin{proof}
    The proof of the parts of the theorem will be discussed throughout the upcoming subsections.
    \begin{itemize}
        \item The symmetries are shown in \Cref{res:kernel_symmetries}.
        \item The global bound \eqref{eq:kernel_global_bounds} is shown in \Cref{res:global_bound_kernel}.
        \item The local bound \eqref{eq:kernel_local_bound} is shown in \Cref{res:local_bound_kernels}.
    \end{itemize}
\end{proof}

Clearly, as all approximative kernels are solely functions of the corresponding approximative distances, the analysis of the quality of an approximative kernel reduces to analysing the quality of the approximative distance that is used, and this is exactly what we will do.

In previous work on PDE-G-CNN's the bound \(d=d(\bp_0,\cdot) \leq \rho_c\) is proven \cite[Lem.20]{smets2022pdebased}. Furthermore, it is shown that around \(\bp_0\) one has:
\begin{equation} \label{eq:old_asymptotics}
    \rho_c^2 \leq d^2 + \cO(d^4),
\end{equation}
which has the corollary that there exist a constant \(C \geq 1\) such that
\begin{equation} \label{eq:old_bound}
    \rho_c \leq C d
\end{equation}
for any compact neighbourhood around \(\bp_0\). We improve on these results by;
\begin{itemize}
    \item Showing that the approximative distances have the same symmetries as the exact Riemannian distance; \Cref{res:distance_symmetries}.
    \item Finding simple global bounds on the exact distance \(d\) which can then be used to find global estimates of \(\rho_b\) by \(d\); \Cref{res:global_bounds}. This improves upon \eqref{eq:old_bound} by finding an expression for the constant \(C\).
    \item Estimating the leading term of the asymptotic expansion, and observing that our upper bound of the relative error between \(\rho_b\) and \(d\) explodes in the cases \(\zeta \to \infty\) and \(w_3 \to 0\); \Cref{res:asymptotics}. This improves upon \cref{eq:old_asymptotics}.
\end{itemize}
Note however that we are \textit{not} analysing \(\rho_c\): we will be analysing \(\rho_b\). This is mainly because the half-angle coordinates are easier to work with: they do not have the \(\sinc \tfrac{\theta}{2}\) factor the logarithmic coordinates have. Using that
\begin{equation} \label{eq:relation_c_and_b}
    b^1 = c^1 \sinc \tfrac{\theta}{2},\ b^2 = c^2 \sinc \tfrac{\theta}{2},\ b^3 = c^3,
\end{equation}
recall (\ref{eq:half_angle_coordinates}) and (\ref{eq:logarithm_coordinates}), we see that
\begin{equation*}
    \sinc \tfrac{\theta}{2}\ \rho_c \leq \rho_b \leq \rho_c,
\end{equation*}
and thus locally \(\rho_c\) and \(\rho_b\) do not differ much, and results on \(\rho_b\) can be easily transferred to (slightly weaker) results on \(\rho_c\).

\subsection{Symmetry Preservation} \label{sec:symmetry}
    
    Symmetries play a major role in the analysis of (sub-)Riemannian geodesics/distance in \(SE(2)\). They help to analyze symmetries in Hamiltonian flows \cite{moiseev2010maxwell} and corresponding symmetries in association field models \cite[Fig.11]{duits2013association}. There are together 8 of them and their relation with logarithmic coordinates \(c^i\) (\Cref{res:symmetries_b_coordinates}) shows they correspond to inversion of the Lie-algebra basis \(A_i \mapsto -A_i\). The symmetries for the sub-Riemannian setting are explicitly listed on \cite[Prop.4.3]{moiseev2010maxwell}. They can be algebraically generated by the (using the same labeling as \cite{moiseev2010maxwell}) following three symmetries:
    \begin{equation} \label{eq:reflectional_symmetries}
    \begin{array}{l}
        \scalemath{0.9}{\e^{2}(x,y,\theta) = (-x \cos \theta - y \sin \theta, -x \sin \theta + y \cos \theta, \theta),}\\
        \scalemath{0.9}{\e^{1}(x,y,\theta) = (x \cos \theta + y \sin \theta, x \sin \theta - y \cos \theta, \theta), \text{ and }} \\
        \scalemath{0.9}{\e^{6}(x,y,\theta) = (x \cos \theta + y \sin \theta, -x \sin \theta + y \cos \theta, -\theta).}
    \end{array}
    \end{equation}
    They generate the other 4 symmetries as follows:
    \begin{equation}\label{eq:symmetries_relations}
    \begin{array}{l}
        \e^{3}=\e^{2} \circ \e^1,\ 
        \e^{4}=\e^{2} \circ \e^6,\ 
        \e^{7}=\e^{1} \circ \e^6, \\
        \text{ and } \e^{5}= \e^2 \circ \e^{1} \circ \e^6.
    \end{array}
    \end{equation}
    and with \(\e^0 = \textrm{id}\). All symmetries are involutions: \(\e^i \circ \e^i = \textrm{id}\). Henceforth all eight symmetries will be called `fundamental symmetries'. How all fundamental symmetries relate to each other becomes clearer if we write them down in either logarithm or half-angle coordinates.
    \begin{lemma} 
        \label{res:symmetries_b_coordinates}
        \emph{\textbf{(8 fundamental symmetries)}} \\
        The 8 fundamental symmetries \(\e_i\), in either half-angle coordinates \(b^i\) or logarithmic coordinates \(c^i\), correspond to sign flips as laid out in \Cref{tab:symmetries_half_angle}.
    \end{lemma}
    \begin{proof}
        We will only show that \(\e^2\) flips \(b^1\). All other calculations are done analogously. Pick a point \(\bp = (x,y,\theta)\) and let \(\bq = \e^2(\bp)\). We now calculate \(b^1(\bq)\):
        \begin{equation*} 
        \begin{split}
            b^1(\bq)
            ={}&x(\bq) \cos \tfrac{\theta(\bq)}{2} + y(\bq) \sin \tfrac{\theta(\bq)}{2}\\
            = &- (x \cos \theta + y \sin \theta) \cos \tfrac{\theta}{2} \\
              &+ (-x \sin \theta + y \cos \theta) \sin \tfrac{\theta}{2}\\
            = &-x (\cos \theta \cos \tfrac{\theta}{2} + \sin \theta \sin \tfrac{\theta}{2} ) \\
              &- y(\sin \cos \tfrac{\theta}{2} - \cos \theta \sin \tfrac{\theta}{2})\\
            = &- x \cos \tfrac{\theta}{2} - y \sin \tfrac{\theta}{2}\\
            = &-b^1(\bp),
        \end{split} 
        \end{equation*}
        where we have used the trigonometric difference identities of cosine and sine in the second-to-last equality. From the relation between logarithmic and half-angle coordinates \eqref{eq:relation_c_and_b} we have that the logarithmic coordinates \(c^i\) flip in the same manner under the symmetries. 
    \end{proof}
    
    \begin{table}
        \centering
        \begingroup
            \setlength{\tabcolsep}{7pt} 
            \renewcommand{\arraystretch}{1.2} 
            \begin{tabular}{c|cccccccc}
                 & \(\e_0\) & \(\e_1\) & \(\e_2\) & \(\e_3\) & \(\e_4\) & \(\e_5\) & \(\e_6\) & \(\e_7\)\\
                \hline
                \(b^1, c^1\) & \(+\) & \(+\) & \(-\) & \(-\) & \(-\) & \(-\) & \(+\) & \(+\) \\
                \(b^2, c^2\) & \(+\) & \(-\) & \(+\) & \(-\) & \(+\) & \(-\) & \(+\) & \(-\) \\
                \(b^3, c^3\) & \(+\) & \(+\) & \(+\) & \(+\) & \(-\) & \(-\) & \(-\) & \(-\)
            \end{tabular}
        \endgroup
        \caption{}
        \label{tab:symmetries_half_angle}
    \end{table}
    
    The fixed points of the symmetries \(\e^2\), \(\e^1\), and \(\e^6\) have an interesting geometric interpretation. The logarithmic and half-angle coordinates, being so closely related to the fundamental symmetries, also carry the same interpretation. \Cref{def:cocircularity} introduces this geometric idea and \Cref{res:cocircular_equivalence} makes its relation to the fixed points of the symmetries precise. In \Cref{fig:fixed_points_symmetries_plot} the fixed points are visualized, and in \Cref{fig:cocircular_coradial_parallel} a visualization of these geometric ideas can be seen.
    
    \renewcommand{\figwidth}{0.30\linewidth}
    \renewcommand{\tikzfigwidth}{2.5\linewidth}
    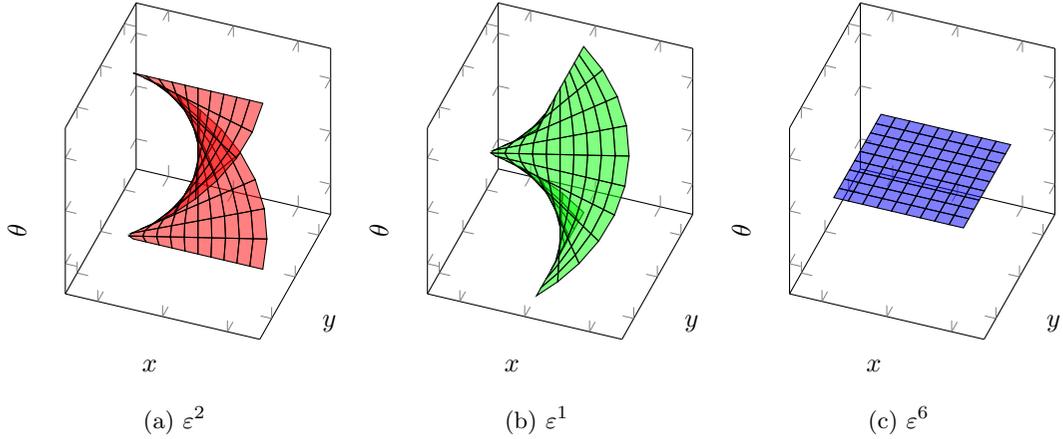
\begin{figure*}
        \centering%
        \begin{subfigure}{\figwidth}
            \centering
            \input{figures/epsilon2}
            \caption{\(\e^2\)}
        \end{subfigure}%
        \begin{subfigure}{\figwidth}
            \centering
            \input{figures/epsilon1}
            \caption{\(\e^1\)}
        \end{subfigure}%
        \begin{subfigure}{\figwidth}
            \centering
            \input{figures/epsilon6}
            \caption{\(\e^6\)}
        \end{subfigure}%
        \caption{The fixed points of the \(\e^2\), \(\e^1\), and \(\e^6\). For \(\e^2\) and \(\e^1\) only the points within the region \(x^2 + y^2 \leq 2^2\) are plotted. For \(\e^6\) only the points in the region \(\max(\Abs{ x},\Abs{ y}) \leq 2\). The fixed points of \(\e^2\), \(\e^1\), and \(\e^6\) correspond respectively to the points in \(\bbM_2\) that are coradial, cocircular, and parallel to the reference point \(\bp_0\).}
        \label{fig:fixed_points_symmetries_plot}
    \end{figure*}

    \begin{definition}\label{def:cocircularity}
        Two points $\bp_1=(\bx_1,\bn_1)$, $\bp_2=(\bx_{2},\bn_1)$ of \(\bbM_{2}\) are called \textbf{cocircular} if there exist a circle, of possibly infinite radius, passing through \(\bx_1\) and \(\bx_2\) such that the orientations \(\bn_1 \in S^1\) and \(\bn_{2} \in S^1\) are tangents to the circle, at respectively \(\bx_1\) and \(\bx_2\), in either both the clockwise or anti-clockwise direction. Similarly, the points are called \textbf{coradial} if the orientations are normal to the circle in either both the outward or inward direction. Finally, two points are called \textbf{parallel} if their orientations coincide.
    \end{definition}
    \renewcommand{\figwidth}{0.30\linewidth}
    \begin{figure*}
        \centering%
        \begin{subfigure}{\figwidth}
            \centering
            \input{figures/coradial}
            \caption{Coradial}
        \end{subfigure}%
        \begin{subfigure}{\figwidth}
            \centering
            \input{figures/cocircular}
            \caption{Cocircular}
        \end{subfigure}%
        \begin{subfigure}{\figwidth}
            \centering
            \input{figures/parallel}
            \caption{Parallel}
        \end{subfigure}%
        \caption{An example of points in \(\bbM_2\) that are coradial, cocircular, and parallel.}
        \label{fig:cocircular_coradial_parallel}
    \end{figure*}
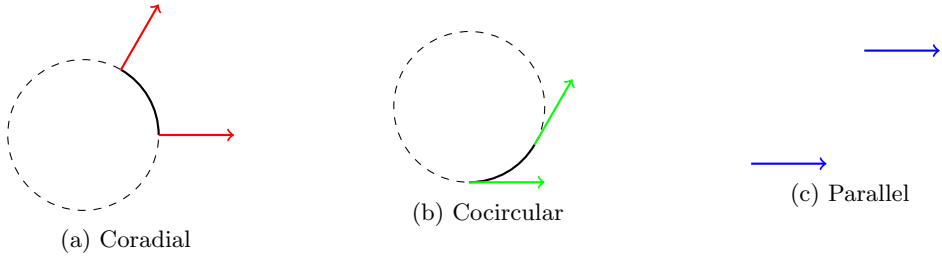
    
    Co-circularity has a well-known characterisation that is often used for line enhancement in image processing, such as tensor voting \cite{medioni2021tensor}.
    \begin{remark} \label{rem:doubleangle}
      Point $\bp=(r \cos \phi, r \sin \phi, \theta) \in \bbM_2$ is cocircular to the reference point \(\bp_0=(0,0,0)\) if and only if the double angle equality \(\theta \equiv 2 \phi \mod 2\pi\) holds.
    \end{remark}
    In fact all fixed points of the fundamental symmetries can be intuitively characterised:
    \begin{lemma} 
        \label{res:cocircular_equivalence}
        \emph{\textbf{(Fixed Points of Symmetries)}} \\
        Fix reference point \(\bp_0=(0,0,0) \in \mathbb{M}_2\). \\
        The point \(g \bp_0\in \mathbb{M}_2\) with $g \in SE(2)$ is respectively  \\
        -\ coradial to $\bp_0$ when
        \begin{equation}
            c^1(g) = 0 \iff \e_2(g) = g \iff g \in \exp(\Ang{A_2, A_3}),
        \end{equation}
        -\ cocircular to $\bp_0$ when
        \begin{equation}
            c^2(g) = 0 \iff \e_1(g) = g \iff g \in \exp(\Ang{A_1, A_3}),  \label{eq:cocircular_equivalence}
        \end{equation}
        -\ parallel to $\bp_0$ when
        \begin{equation}
            c^3(g) = 0 \iff \e_6(g) = g \iff g \in \exp(\Ang{A_1, A_2}).
        \end{equation}
    \end{lemma}
    \begin{proof}
        We will only show \eqref{eq:cocircular_equivalence}, the others are done analogously. We start by writing \(g=(r \cos \phi, r \sin \phi, \theta)\) and calculating that \(g \odot \bp_0 = (r \cos \phi, r \sin \phi, (\cos \theta, \sin \theta))\). Then by \Cref{rem:doubleangle} we known that \(g \bp_0\) is cocircular to \(\bp_0\) if and only if \(2\phi = \theta \omod 2\pi\). We can show this is equivalent to \(c^2(g)=0\):
        \begin{equation*} 
        \begin{split}
            c^2(g) = 0 
            &\iff b^2(g) = 0 \\
            &\iff -x \sin \tfrac{\theta}{2} +y \cos \tfrac{\theta}{2}=0\\
            &\iff -\cos \phi \sin \tfrac{\theta}{2}+\sin \phi \cos \tfrac{\theta}{2}=0\\
            &\iff \sin(\phi-\tfrac{\theta}{2})=0 
            \iff 2\phi = \theta \omod 2\pi.
        \end{split}
        \end{equation*}
        In logarithmic coordinates \(\e_1\) is equivalent to:
        \begin{equation*}
            \e_1(c^1, c^2, c^3) = (c^1, -c^2, c^3) 
        \end{equation*}
        from which we may deduce that \(\e_1(g) = g\) is equivalent to \(c^2(g) = 0\). If \(c^2(g) = 0\) then \(\log g \in \Ang{A_1, A_3}\) and thus \(g \in \exp(\Ang{A_1, A_3})\). As for the other way around, it holds by simple computation that:
        \begin{equation*}
            c^2(\exp(c^1A_1 + c^3A_3)) = 0
        \end{equation*}
        which shows that \(g \in \exp(\Ang{A_1, A_3}) \Rightarrow c^2(g) = 0\).
    \end{proof}

    In the important work \cite{moiseev2010maxwell} on sub-Riemannian geometry on $SE(2)$ by Sachkov and Moiseev, it is shown that the exact sub-Riemannian distance \(d_{sr}\) is invariant under the fundamental symmetries \(\e^i\). However, these same symmetries hold true for the Riemannian distance \(d\). Moreover, because the approximative distances use the logarithmic coordinates \(c^i\) and half-angle coordinates \(b^i\) they also carry the same symmetries. The following lemma makes this precise.
    
    \begin{lemma} 
        \label{res:distance_symmetries} 
        \emph{\textbf{(Symmetries of the exact distance and all proposed approximations)}}\\
        All exact and approximative (sub)-Riemannian distances (w.r.t. the reference point \(\bp_0\))
        are invariant under all the fundamental symmetries \(\e_i\).
    \end{lemma}
    \begin{proof}
        By \Cref{tab:symmetries_half_angle} one sees that \(\e^3, \e^4\), and \(\e^5\) also generate all symmetries. Therefore if we just show that all distances are invariant under these select three symmetries we also have shown that they are invariant under all symmetries. We will first show the exact distance, in either the Riemannian or sub-Riemannian case, is invariant w.r.t. these three symmetries, i.e. \(d(\bp) = d(\e^i(\bp))\) for $i \in \{3,4,5\}$. 
        
        By \eqref{eq:symmetries_relations} and \eqref{eq:reflectional_symmetries} one has \(\e^3(x,y,\theta)=(-x,-y,\theta)\) and \(\e^4(x,y,\theta) = (-x,y,-\theta)\). Now consider the push forward \(\e^3_*\). By direct computation (in \((x,y,\theta)\) coordinates) we have \(\e^3_* \At{\cA_i}_\bp = \pm \At{\cA_i}_{\e^3(\bp)}\). Because the metric tensor field \(\cG\) \eqref{eq:diagonal_metric} is diagonal w.r.t. to the \(\cA_i\) basis this means that \(\e^3\) is a isometry. Similarly, \(\e^4\) is an isometry. Being an isometry of the metric \(\cG\) we may directly deduce that \(\e^3\) and \(\e^4\) preserve distance. The \(\e^5\) symmetry flips all the signs of the \(c^i\) coordinates which amounts to Lie algebra inversion: \( -\log g = \log(\e^5(g)) \). Taking the exponential on both sides shows that \(g^{-1} = \e^5(g)\). By left-invariance of the metric we have \(d(g \bp_0, \bp_0)  = d(\bp_0, g^{-1} \bp_0)\), which holds in both the Riemannian and sub-Riemannian case, and thus \( d(g\bp_0) = d(\e^5(g\bp_0)) \). 
        
        That all approximative distances (both in the Riemannian and sub-Riemannian case) are also invariant under all the symmetries is not hard to see: every \(b^i\) and \(c^i\) term is either squared or the absolute value is taken. Flipping signs of these coordinates, recall \Cref{res:symmetries_b_coordinates}, has no effect on the approximative distance.
    \end{proof}
    
    \begin{corollary} 
        \label{res:kernel_symmetries}
        \emph{\textbf{(All kernels preserve symmetries)}} \\   
        The exact kernel and all approximative kernels have the same fundamental symmetries.
    \end{corollary}
    \begin{proof}
        The kernels are direct functions of the exact and approximative distances, recall for example \eqref{eq:morphological_kernel}, so from \Cref{res:distance_symmetries} we can immediately conclude that they also carry the 8 fundamental symmetries.
    \end{proof}
    
    In \Cref{fig:true_distance_plot_intro} the previous lemma can be seen. The two fundamental symmetries \(\e^2\) and \(\e^1\) correspond, respectively, to reflecting the isocontours (depicted in colors) along their short edge and long axis. The \(\e^6\) symmetry corresponds to mapping the positive \(\theta\) isocontours to their negative \(\theta\) counterparts. In \Cref{fig:symmetries_overlay_rho_b} one can see an isocontour of \(\rho_b\) together with the symmetry ``planes'' of \(\e_2\), \(\e_1\) and \(\e_6\).

    \begin{figure}
        \centering
        \includegraphics[width=0.8\linewidth]{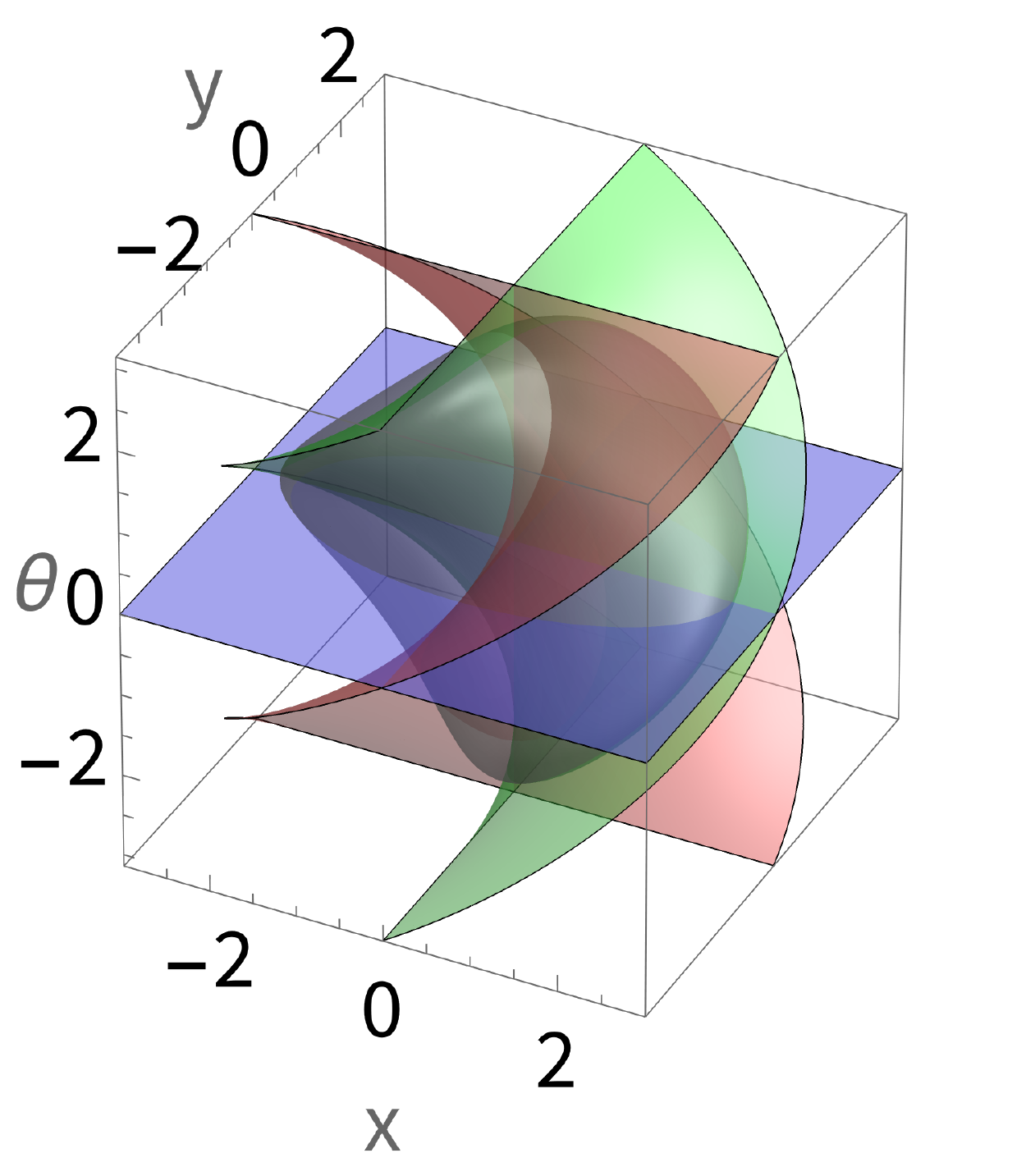}
        \caption{In grey the isocontour \(\rho_b=2.5\), together with the symmetry ``planes'' of \(\e_2\), \(\e_1\) and \(\e_6\), as also plotted in \Cref{fig:fixed_points_symmetries_plot}. The metric parameters are \((w_1,w_2,w_3)=(1,2,1)\). }
        \label{fig:symmetries_overlay_rho_b}
    \end{figure}

\subsection{Simple Global Bounds} \label{sec:global_bounds}

    Next we provide some basic global lower and upper bounds for the exact Riemannian distance \(d\) \eqref{eq:riemannian_distance}. Recall that the lower bound \(l\) plays an important role in the combined approximation \(\rho_{c,com}\) \eqref{eq:combined_distance_approximation} when far from the reference point \(\bp_0\).
    
    \begin{lemma} \label{res:global_bounds} 
    \emph{\textbf{(Global bounds on distance)}} \\
        The exact Riemannian distance \(d=d(\bp_0,\cdot)\) is greater than or equal to the following lower bound \(l: \bbM_2 \to \bbR\):
        \begin{equation*}
            l := \sqrt{ (w_1 x)^2 + (w_1 y)^2 + (w_3 \theta)^2 } \leq d
        \end{equation*}
        and less than or equal to the following upper bounds \(u_1, u_2 : \bbM_2 \to \bbR\):
        \begin{align*}
            d \leq u_1 &:= \sqrt{ (w_2 x)^2 + (w_2 y)^2 + (w_3 \theta)^2 }\\
            d \leq u_2 &:= \sqrt{ (w_1 x)^2 + (w_1 y)^2 } + w_3 \pi
        \end{align*} 
    \end{lemma}
    \begin{proof}
    
        We will first show \(l \leq d\). Consider the following spatially isotropic metric:
        \begin{equation*}
            \tilde \cG = w_1^2\ \omega^1 \otimes \omega^1 + w_1^2\ \omega^2 \otimes \omega^2 + w_3^2 \ \omega^3 \otimes \omega^3.
        \end{equation*}
        We assumed w.l.o.g. that \(w_1 \leq w_2\) so we have for any vector \(v \in T\bbM_2\) that \( \Vert v\Vert_{\tilde \cG} \leq \Vert v\Vert_{\cG} \). From this we can directly deduce that for any curve \(\g\) on \(\bbM_2\) we have that \(L_{\tilde \cG}(\g) \leq L_{\cG}(\g)\). Now consider a length-minimizing curve \(\g\) w.r.t. \(\cG\) between the reference point \(\bp_0\) and some end point \(\bp\). We then have the chain of (in)equalities:
        \begin{equation*}
            d_{\tilde \cG}(\bp) \leq L_{\tilde \cG}(\g) \leq L_{\cG}(\g) = d_{\cG}(\bp)
        \end{equation*}
        Furthermore, because the metric \(\tilde \cG\) is spatially isotropic it can be equivalently be written as:
        \begin{equation*}
            \tilde \cG = w_1^2\ dx \otimes dx + w_1^2\ dy \otimes dy + w_3^2 \ d\theta \otimes d\theta,
        \end{equation*}
        which is a constant metric on the coordinate covector fields, and thus:
        \begin{equation*}
             d_{\tilde \cG}(\bp) = \sqrt{ (w_1 x)^2 + (w_1 y)^2 + (w_3 \theta)^2 } = l.
        \end{equation*}
        Putting everything together gives the desired result of \(l \leq d\). To show that \(d \leq u_1\) can be done analogously.
        
        As for showing \(d \leq u_2\) we will construct a curve \(\g\) of which the length \(L(\g)\) w.r.t. \(\cG\) can be bounded from above with \(u_2\). This in turn shows that \(d \leq u_2\) by definition of the distance. Pick a destination position and orientation \(\bp = (\bx, \bn)\). The constructed curve \(\g\) will be as follows. We start by aligning our starting orientation \(\bn_0 = (1,0) \in S^1\) towards the destination position \(\bx\). This desired orientation towards \(\bx\) is \(\hat \bx := \frac{\bx}{r}\) where \(r = \Vert \bx \Vert = \sqrt{x^2 + y^2}\). This action will cost \(w_3 a\) for some \(a \geq 0\). Once we are aligned with \(\hat \bx\) we move towards \(\bx\). Because we are aligned this action will cost \(w_1 r\). Now that we are at \(\bx\) we align our orientation with the destination orientation \(\bn\), which will cost \(w_3b\) for some \(b \geq 0\). Altogether we have \(L(\g) = w_1 r + w_3 (a+b)\). In its current form the constructed curve actually doesn't have that \(a+b\leq \pi\) as desired. To fix this we realise that we did not necessarily had to align with \(\hat \bx\). We could have aligned with \(-\hat \bx\) and move backwards towards \(\bx\), which will also cost \(w_1 r\). One can show that one of the two methods (either moving forwards or backwards towards \(\bx\)) indeed has that \(a+b\leq \pi\) and thus \(d \leq u_2\).
    \end{proof}
    
    These bounds are simple but effective: they help us prove a multitude of insightful corollaries.
    
    \begin{corollary}
        \label{res:global_bounds_rho_b}
        \emph{\textbf{(Global error distance)}} \\
        Simple manipulations, together with the fact that \(x^2 + y^2 = (b^1)^2 + (b^2)^2\), give the following inequalities between \(l, u_1\) and \(\rho_b\):
        \begin{align*}
            l \leq \rho_b \leq u_1,\ 
            \frac{1}{\zeta} u_1 \leq \rho_b \leq \zeta l. 
        \end{align*}
        The second equation can be extended to inequalities between \(\rho_b\) and \(d\):
        \begin{equation} \label{eq:rho_b_and_d}
            \frac{1}{\zeta} d \leq \rho_b \leq \zeta d 
        \end{equation}
    \end{corollary} 
        
    \begin{remark} 
        \label{rem:rho_b_becomes_exact}
        If \(w_1 = w_2 \iff \zeta = 1\), i.e. the spatially isotropic case, then the lower and upper bound coincide, thus becoming exact. Because \(\rho_b\) is within the lower and upper bound it also becomes exact.
    \end{remark}
    
    \begin{corollary}
        \label{res:global_bound_kernel}
        \emph{\textbf{(Global error kernel)}} \\
        Globally the error is independent of time $t>0$ and is estimated by the spatial anisotropy $\zeta \geq 1$ \eqref{eq:spatial_anisotropy} as follows:
        \begin{equation*}
             \zeta^{-\beta} k^\a \leq k_b^\a \leq \zeta^{\beta} k^\a \ .
        \end{equation*}
        For $\zeta=1$ there is no error.
    \end{corollary}
    \begin{proof}
        We will only prove the second inequality, the first is done analogously.
        \begin{equation*}
        \begin{split}
            k_b^\a 
            &:= \frac{1}{\beta} (\rho_b/t)^\beta 
            \leq \frac{1}{\beta} \Par{\zeta d/t}^\beta \\
            &= \zeta^{\beta} \Par{\frac{1}{\beta} \Par{d/t}^\beta}
            = \zeta^{\beta} k^\a
        \end{split}
        \end{equation*}
    \end{proof}
    The previous result indicates that problems can arise if $\zeta \to \infty$, which indeed turns out to be the case: 
    \begin{corollary} 
        \label{res:rho_b_goes_certainly_bad}
        \emph{\textbf{(Observing the problem)}} \\
        If we restrict ourselves to \(x=\theta=0\) we have that \(u_1 = \rho_b = \rho_c = w_2\Abs{y}\). From this we can deduce that one can be certain that both \(\rho_b\) and \(\rho_c\) become bad approximations away from \(\bp_0\). Namely, when \(\zeta > 1 \iff w_2 > w_1\)  both approximations go above \(u_2\) if one looks far enough away from \(\bp_0\). How ``fast'' it goes bad is determined by all metric parameters. Namely, the intersection of the approximations \(\rho_b\) and \(\rho_c\), and \(u_2\) is at \(\Abs{ y} = \frac{w_3\pi}{w_2 - w_1}\), or equivalently at \(\rho = \frac{w_3\pi}{1 - \zeta^{-1}}\). This intersection is visible in \Cref{fig:exact_distance_and_bounds} in the higher anisotropy cases. From this expression of the intersection we see that in the cases \(w_3 \to 0\) and \(\zeta \to \infty\) the Riemannian distance approximations \(\rho_b\) and \(\rho_c\) quickly go bad. We will see exactly the same behaviour in \Cref{res:asymptotics} and Remark~\ref{rem:tolerance_bound}.
    \end{corollary}
    
    
    \Cref{res:global_bounds} is visualized in \Cref{fig:exact_distance_and_bounds,fig:exact_distance_and_bounds_2}. In \Cref{fig:exact_distance_and_bounds} figure we consider the behavior of the exact distance and bounds along the \(y\)-axis, that is at \(x=\theta=0\). We have chosen to inspect the \(y\)-axis because it consist of points that are hard to reach from the reference point \(\bp_0\) when the spatial anisotropy is large, which makes it interesting. In contrast, along the \(x\)-axis \(l,d,\rho_b,\rho_c, u_1\) and \(w_1\Abs{x}\) all coincide, and is therefor uninteresting. To provide more insight we also depict the bounds along the \(y=x\) axis, see \Cref{fig:exact_distance_and_bounds_2}. Observe that in both figures, the exact distance \(d\) is indeed always above the lower bound \(l\) and below the upper bounds \(u_1\) and \(u_2\).
    
    \renewcommand{\figwidth}{0.49\linewidth}
    \renewcommand{\tikzfigwidth}{1.25\linewidth}
    \begin{figure}
        \centering%
        \begin{subfigure}{\figwidth}
            \centering
            \input{figures/bounds_1d_at_x=theta=0_w=1,1,1}
            \caption{\(w_2 = 1\)}
        \end{subfigure}%
        \begin{subfigure}{\figwidth}
            \centering
            \input{figures/bounds_1d_at_x=theta=0_w=1,2,1}
            \caption{\(w_2 = 2\)}
        \end{subfigure}
        \begin{subfigure}{\figwidth}
            \centering
            \input{figures/bounds_1d_at_x=theta=0_w=1,3,1}
            \caption{\(w_2 = 3\)}
        \end{subfigure}%
        \begin{subfigure}{\figwidth}
            \centering
            \input{figures/bounds_1d_at_x=theta=0_w=1,4,1}
            \caption{\(w_2 = 4\)}
        \end{subfigure}%
        \caption{Exact distance and its lower and upper bounds (given in \Cref{res:global_bounds}) along the $y$-axis, i.e at \(x=\theta=0\), for increasing spatial anisotropy. We keep \(w_1=w_3=1\) and vary \(w_2\). The horizontal axis is \(y\) and the vertical axis the value of the distance/bound. Note how the exact distance \(d\) starts of linearly with a slope of \(w_2\), and ends linearly with a slope of \(w_1\).}
        \label{fig:exact_distance_and_bounds}
    \end{figure}
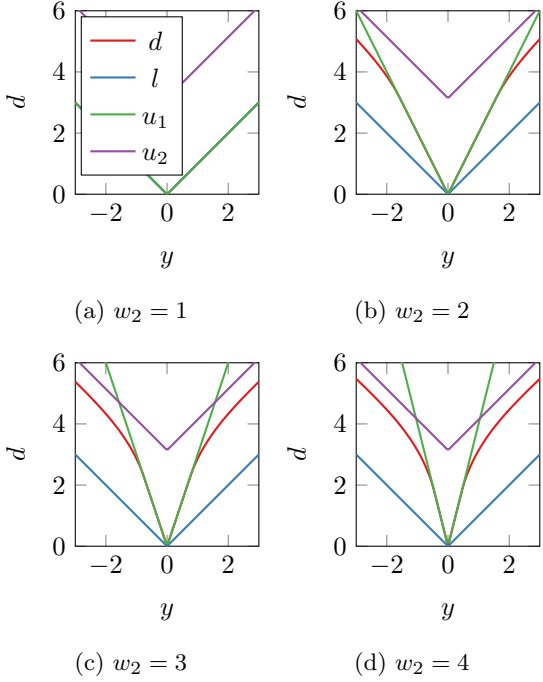
    
    \renewcommand{\figwidth}{0.49\linewidth}
    \renewcommand{\tikzfigwidth}{1.25\linewidth}
    \begin{figure}
        \centering%
        \begin{subfigure}{\figwidth}
            \centering
            \input{figures/bounds_1d_at_x=y_theta=0_w=1,1,1}
            \caption{\(w_2 = 1\)}
        \end{subfigure}%
        \begin{subfigure}{\figwidth}
            \centering
            \input{figures/bounds_1d_at_x=y_theta=0_w=1,2,1}
            \caption{\(w_2 = 2\)}
        \end{subfigure}
        \begin{subfigure}{\figwidth}
            \centering
            \input{figures/bounds_1d_at_x=y_theta=0_w=1,3,1}
            \caption{\(w_2 = 3\)}
        \end{subfigure}%
        \begin{subfigure}{\figwidth}
            \centering
            \input{figures/bounds_1d_at_x=y_theta=0_w=1,4,1}
            \caption{\(w_2 = 4\)}
        \end{subfigure}%
        \caption{Same setting as \Cref{fig:exact_distance_and_bounds} but at \(x=y, \theta=0\). The horizontal axis moves along the line \(x=y\).}
        \label{fig:exact_distance_and_bounds_2}
    \end{figure}
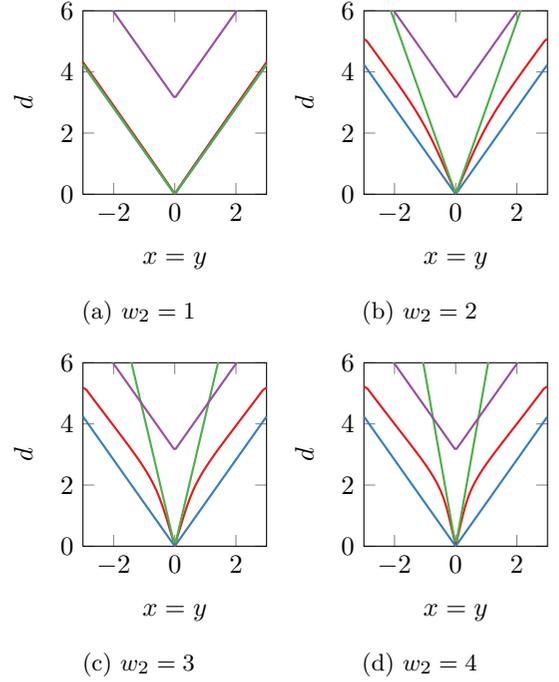

\subsection{Asymptotic Error Expansion} \label{sec:asymptotic_analysis}

     In this section we provide an asymptotic expansion of the error between the exact distance \(d\) and the half-angle distance approximation \(\rho_b\) (\Cref{res:asymptotics}). This error is then leveraged to an error between the exact morphological kernel \(k\) and the half-angle kernel \(k_b\) (\Cref{res:local_bound_kernels}). We also give a formula that determines a region for which the half-angle approximation \(\rho_b\) is appropriate given an a priori tolerance bound (\Cref{rem:tolerance_bound}).
    
    \begin{lemma} \label{res:fundamental_calculus}
        Let \(\g:[0,1] \to \bbM_2\) be a minimizing geodesic from \(\bp_0\) to \(\bp\). We have that:
        \begin{equation*}
            \rho_b(\bp) \leq d(\bp) \max_{t \in [0,1]} \Vert d\rho_b\vert_{\g(t)} \Vert.
        \end{equation*}
    \end{lemma}
    \begin{proof}
        The fundamental theorem of calculus tells us that:
        \begin{equation*}
            \int_0^1 (\rho_b \circ \g)'(t)\ dt = \rho_b(\g(1)) - \rho_b(\g(0)) = \rho_b(\bp),
        \end{equation*}
        but one can also bound this expression as follows:
        \begin{align*}
            \int_0^1 (\rho_b \circ \g)'(t)\ dt  
            &= \int_0^1 \Ang{d\rho_b\vert_{\g(t)}, \dot \g(t)}\ dt \\
            & \leq \int_0^1 \Norm{d\rho_b\vert_{\g(t)}} \Norm{\dot \g(t)}\ dt\\
            &\leq  \Par{\max_{t \in [0,1]} \Vert d\rho_b\vert_{\g(t)} \Vert} \int_0^1 \Norm{\dot \g(t)}\ dt \\
            &=  d(\bp) \max_{t \in [0,1]} \Vert d\rho_b\vert_{\g(t)} \Vert.
        \end{align*} 
        Putting the two together gives the desired result.
    \end{proof}
    
    \begin{lemma} \label{res:bound_on_differential_norm_rho_b}
        One can bound \(\Vert d\rho_b\Vert\) around \(\bp_0\) by:
        \begin{equation*}
            \Vert d \rho_b\Vert^2 \leq 1 + \frac{\zeta^2 - 1}{2w_3^2} \rho_b^2  + \cO(\theta^3).
        \end{equation*}
    \end{lemma}
    \begin{proof}
        The proof is deferred to \Cref{app:proof_bound_on_differential_norm_rho_b}
    \end{proof}
    
    By combining the simple \Cref{res:fundamental_calculus,res:bound_on_differential_norm_rho_b} one can find an expression for the asymptotic error between the exact distance \(d\) and the half-angle approximation \(\rho_b\).
    
    \begin{lemma} \label{res:asymptotics}
        Around any compact neighbourhood of \(\bp_0\) we have that
        \begin{equation} \label{eq:relative_error_1}
            \rho_b^2 \leq ( 1 + \e) d^2 , \text{ where } \e := \frac{\zeta^2 - 1}{2w_3^2} \zeta^4 \rho_b^2  + C \Abs{\theta}^3.
        \end{equation}
        for some $C \geq 0$.
    \end{lemma}
    
    \begin{proof}
        Let $\bp \in U$ be given, and let $\g: [0,1] \rightarrow \bbM_2$ be the geodesic from $\bp_0$ to $\bp$. For the distance we know that 
        \begin{equation*} 
            d(\g(s)) \leq d(\g(t)), \text{ for } s \leq t.
        \end{equation*}
        Making use of \eqref{eq:rho_b_and_d} we know that \(\frac{1}{\zeta} \rho_b \leq d \leq \zeta \rho_b\) so we can combine this with the previous equation to find:
        \begin{equation*}
            \rho_b(\g(s)) \leq \zeta^2 \rho_b(\g(t)), \text{ for } s \leq t.
        \end{equation*}
        from which we get that
        \begin{equation*}
            \max_{t \in [0,1]} \rho_b(\g(t)) \leq \zeta^2 \rho_b(\bp).
        \end{equation*}
        Combining this fact with the above two lemmas allows us to conclude \eqref{eq:relative_error_1}. 
    \end{proof}
    
    \begin{remark} 
        \label{rem:tolerance_bound}
        \textbf{(Region for approximation $\rho_b \approx d$)} \\
        Putting an \textit{a priori} tolerance bound \(\e_{tol}\) on the error \(\e\) (and neglecting the \(\cO(\theta^3)\) term) gives rise to a region $\Omega_0$ on which the local approximation \(\rho_b\) is appropriate:
        \begin{equation*}
            \Omega_0=\{ \bp \in \bbM_2 \mid \rho_b(\bp) < \frac{2 w_3^2}{(\zeta^2-1)\zeta^4} \e_{tol}\}.
        \end{equation*}
       Thereby we can not guarantee a large region of acceptable relative error when \(w_3 \to 0\) or \(\zeta \to \infty\). We solve this problem 
       by using $\rho_{b, com}$ given (\ref{eq:combined_distance_approximation}) instead of $\rho_b$.
    \end{remark}
    
    \begin{corollary} 
        \label{res:local_bound_kernels}
        \emph{\textbf{(Local error morphological kernel)}}\\
        Locally around \(\bp_0\) we have:
        \begin{equation*}
            k^\a_b < (1 + \e)^{\beta/2} k^\a.
        \end{equation*}
    \end{corollary}
    \begin{proof} 
        By Lemma~\ref{res:asymptotics} one has 
        \begin{equation*}
            k^\a_b 
            := \frac{1}{\beta} (\rho_b/t)^\beta
            \leq \frac{1}{\beta} ((1 + \e)d^2/t^2)^{\beta/2}
            = (1 + \e)^{\beta/2} k^\a.
        \end{equation*}
    \end{proof}
    
    

%% file: figures/epsilon2.tex
\begin{tikzpicture}
    \begin{axis}[
        width=\tikzfigwidth,
        view={20}{40},
        xmin = -3,
        xmax = 3,
        ymin = -3,
        ymax = 3,
        zmin = -pi,
        zmax = pi,
        xlabel = {\(x\)},
        ylabel = {\(y\)},
        zlabel = {\(\theta\)},
        unit vector ratio = 1 1 1,
        xticklabels={},
        yticklabels={},
        zticklabels={}
    ]
        \addplot3 [
            surf,
            fill opacity=0.5,
            fill = red,
            faceted color = black,
            samples = 11,
            domain = -2:2,
            y domain = -pi:pi,
            z buffer = sort,
        ] (
        {x * cos(deg(y)/2 + 90)},
        {x * sin(deg(y)/2 + 90)},
        y
        );  
    \end{axis}
\end{tikzpicture}

%% file: figures/epsilon1.tex
\begin{tikzpicture}
    \begin{axis}[
        width=\tikzfigwidth,
        view={20}{40},
        xmin = -3,
        xmax = 3,
        ymin = -3,
        ymax = 3,
        zmin = -pi,
        zmax = pi,
        xlabel = {\(x\)},
        ylabel = {\(y\)},
        zlabel = {\(\theta\)},
        unit vector ratio = 1 1 1,
        xticklabels={},
        yticklabels={},
        zticklabels={}
    ]
        \addplot3 [
            surf,
            fill opacity=0.5,
            fill = green,
            faceted color = black,
            samples = 11,
            domain = -2:2,
            y domain = -pi:pi,
            z buffer = sort,
        ] (
        {x * cos(deg(y)/2)},
        {x * sin(deg(y)/2)},
        y
        );  
    \end{axis}
\end{tikzpicture}

%% file: figures/epsilon6.tex
\begin{tikzpicture}
    \begin{axis}[
        width=\tikzfigwidth,
        view={20}{40},
        xmin = -3,
        xmax = 3,
        ymin = -3,
        ymax = 3,
        zmin = -pi,
        zmax = pi,
        xlabel = {\(x\)},
        ylabel = {\(y\)},
        zlabel = {\(\theta\)},
        unit vector ratio = 1 1 1,
        xticklabels={},
        yticklabels={},
        zticklabels={}
    ]
        \addplot3 [
            surf,
            fill opacity=0.5,
            fill = blue,
            faceted color = black,
            samples = 10,
            domain = -2:2,
            z buffer = sort,
        ] {0};
    \end{axis}
\end{tikzpicture}

%% file: figures/coradial.tex
\begin{tikzpicture}
    \draw[thick] (0,0) arc (0:60:1);
    \draw[dashed] (0,0) arc (0:360:1);
    \draw[thick,red] [-to] (0,0) -- (1,0);
    \draw[thick,red] [-to] (60:1) ++ (-1,0) -- ++(60:1);
\end{tikzpicture}

%% file: figures/cocircular.tex
\begin{tikzpicture}
    \draw[thick] (0,0) arc (-90:-30:1);
    \draw[dashed] (0,0) arc (-90:270:1);
    \draw[thick,green] [-to] (0,0) -- (1,0);
    \draw[thick,green] [-to] (30:1) -- ++(60:1);
\end{tikzpicture}

%% file: figures/parallel.tex
\begin{tikzpicture}
    \draw[thick,blue] [-to] (0,0) -- ++(1,0);
    \draw[thick,blue] [-to] (1.5,1.5) -- ++(1,0);
\end{tikzpicture}

%% file: figures/bounds_1d_at_x=theta=0_w=1,1,1.tex
\begin{tikzpicture}
    \begin{axis}[
        width = \tikzfigwidth,
        xlabel = {\(y\)},
        ylabel = {\(d\)},
        ymin = 0,
        ymax = 6,
        xmin = -3,
        xmax = 3,
        cycle list/Set1,
        legend pos = north west,
        unit vector ratio = 1 1
    ]
        \addplot+[thick, mark=none] table [x=y, y=d, col sep=comma] {data/x=theta=0/exact_data_x=theta=0_w=1,1,1.csv};
        \addlegendentry{\(d\)}
        
        \addplot+[thick, mark=none, domain=-3:3, samples=100]{ l(0,x,0,1,1,1) };
        \addlegendentry{\(l\)}
        
        \addplot+[thick, mark=none, domain=-3:3, samples=100]{ u1(0,x,0,1,1,1) };
        \addlegendentry{\(u_1\)}
        
        \addplot+[thick, mark=none, domain=-3:3, samples=100]{ u2(0,x,0,1,1,1) };
        \addlegendentry{\(u_2\)}
    \end{axis}
\end{tikzpicture}

%% file: figures/bounds_1d_at_x=theta=0_w=1,2,1.tex
\begin{tikzpicture}
    \begin{axis}[
        width = \tikzfigwidth,
        xlabel = {\(y\)},
        ylabel = {\(d\)},
        ymin = 0,
        ymax = 6,
        xmin = -3,
        xmax = 3,
        cycle list/Set1,
        legend pos = south west,
        unit vector ratio = 1 1 1
    ]
        \addplot+[thick, mark=none] table [x=y, y=d, col sep=comma] {data/x=theta=0/exact_data_x=theta=0_w=1,2,1.csv};
        \addlegendentry{\(d\)}
        
        \addplot+[thick, mark=none, domain=-3:3, samples=100]{ l(0,x,0,1,2,1) };
        \addlegendentry{\(l\)}
        
        \addplot+[thick, mark=none, domain=-3:3, samples=100]{ u1(0,x,0,1,2,1) };
        \addlegendentry{\(u_1\)}
        
        \addplot+[thick, mark=none, domain=-3:3, samples=100]{ u2(0,x,0,1,2,1) };
        \addlegendentry{\(u_2\)}
        \legend{}; 
    \end{axis}
\end{tikzpicture}

%% file: figures/bounds_1d_at_x=theta=0_w=1,3,1.tex
\begin{tikzpicture}
    \begin{axis}[
        width = \tikzfigwidth,
        xlabel = {\(y\)},
        ylabel = {\(d\)},
        ymin = 0,
        ymax = 6,
        xmin = -3,
        xmax = 3,
        cycle list/Set1,
        legend pos = south west,
        unit vector ratio = 1 1 1
    ]
        \addplot+[thick, mark=none] table [x=y, y=d, col sep=comma] {data/x=theta=0/exact_data_x=theta=0_w=1,3,1.csv};
        \addlegendentry{\(d\)}
        
        \addplot+[thick, mark=none, domain=-3:3, samples=100]{ l(0,x,0,1,3,1) };
        \addlegendentry{\(l\)}
        
        \addplot+[thick, mark=none, domain=-3:3, samples=100]{ u1(0,x,0,1,3,1) };
        \addlegendentry{\(u_1\)}
        
        \addplot+[thick, mark=none, domain=-3:3, samples=100]{ u2(0,x,0,1,3,1) };
        \addlegendentry{\(u_2\)}
        \legend{}; 
    \end{axis}
\end{tikzpicture}

%% file: figures/bounds_1d_at_x=theta=0_w=1,4,1.tex
\begin{tikzpicture}
    \begin{axis}[
        width = \tikzfigwidth,
        xlabel = {\(y\)},
        ylabel = {\(d\)},
        ymin = 0,
        ymax = 6,
        xmin = -3,
        xmax = 3,
        cycle list/Set1,
        legend pos = south west,
        unit vector ratio = 1 1 1
    ]
        \addplot+[thick, mark=none] table [x=y, y=d, col sep=comma] {data/x=theta=0/exact_data_x=theta=0_w=1,4,1.csv};
        \addlegendentry{\(d\)}
        
        \addplot+[thick, mark=none, domain=-3:3, samples=100]{ l(0,x,0,1,4,1) };
        \addlegendentry{\(l\)}
        
        \addplot+[thick, mark=none, domain=-3:3, samples=100]{ u1(0,x,0,1,4,1) };
        \addlegendentry{\(u_1\)}
        
        \addplot+[thick, mark=none, domain=-3:3, samples=100]{ u2(0,x,0,1,4,1) };
        \addlegendentry{\(u_2\)}
        \legend{}; 
    \end{axis}
\end{tikzpicture}

%% file: figures/bounds_1d_at_x=y_theta=0_w=1,1,1.tex
\begin{tikzpicture}
    \begin{axis}[
        width = \tikzfigwidth,
        xlabel = {\(x=y\)},
        ylabel = {\(d\)},
        ymin = 0,
        ymax = 6,
        xmin = -3,
        xmax = 3,
        cycle list/Set1,
        legend pos = north west,
        unit vector ratio = 1 1 1
    ]
        \addplot+[thick, mark=none] table [x=y, y=d, col sep=comma] {data/x=y_theta=0/exact_data_x=y_theta=0_w=1,1,1.csv};
        \addlegendentry{\(d\)}
        \addplot+[thick, mark=none, domain=-3:3, samples=100]{ l(x,x,0,1,1,1) };
        \addlegendentry{\(l\)}
        \addplot+[thick, mark=none, domain=-3:3, samples=100]{ u1(x,x,0,1,1,1) };
        \addlegendentry{\(u_1\)}
        \addplot+[thick, mark=none, domain=-3:3, samples=100]{ u2(x,x,0,1,1,1) };
        \addlegendentry{\(u_2\)}
        \legend{}
    \end{axis}
\end{tikzpicture}

%% file: figures/bounds_1d_at_x=y_theta=0_w=1,2,1.tex
\begin{tikzpicture}
    \begin{axis}[
        width = \tikzfigwidth,
        xlabel = {\(x=y\)},
        ylabel = {\(d\)},
        ymin = 0,
        ymax = 6,
        xmin = -3,
        xmax = 3,
        cycle list/Set1,
        legend pos = south west,
        unit vector ratio = 1 1 1
    ]
        \addplot+[thick, mark=none] table [x=y, y=d, col sep=comma] {data/x=y_theta=0/exact_data_x=y_theta=0_w=1,2,1.csv};
        \addlegendentry{\(d\)}
        \addplot+[thick, mark=none, domain=-3:3, samples=100]{ l(x,x,0,1,2,1) };
        \addlegendentry{\(l\)}
        \addplot+[thick, mark=none, domain=-3:3, samples=100]{ u1(x,x,0,1,2,1) };
        \addlegendentry{\(u_1\)}
        \addplot+[thick, mark=none, domain=-3:3, samples=100]{ u2(x,x,0,1,2,1) };
        \addlegendentry{\(u_2\)}
        \legend{}; 
    \end{axis}
\end{tikzpicture}

%% file: figures/bounds_1d_at_x=y_theta=0_w=1,3,1.tex
\begin{tikzpicture}
    \begin{axis}[
        width = \tikzfigwidth,
        xlabel = {\(x=y\)},
        ylabel = {\(d\)},
        ymin = 0,
        ymax = 6,
        xmin = -3,
        xmax = 3,
        cycle list/Set1,
        legend pos = south west,
        unit vector ratio = 1 1 1
    ]
        \addplot+[thick, mark=none] table [x=y, y=d, col sep=comma] {data/x=y_theta=0/exact_data_x=y_theta=0_w=1,3,1.csv};
        \addlegendentry{\(d\)}
        \addplot+[thick, mark=none, domain=-3:3, samples=100]{ l(x,x,0,1,3,1) };
        \addlegendentry{\(l\)}
        \addplot+[thick, mark=none, domain=-3:3, samples=100]{ u1(x,x,0,1,3,1) };
        \addlegendentry{\(u_1\)}
        \addplot+[thick, mark=none, domain=-3:3, samples=100]{ u2(x,x,0,1,3,1) };
        \addlegendentry{\(u_2\)}
        \legend{}; 
    \end{axis}
\end{tikzpicture}

%% file: figures/bounds_1d_at_x=y_theta=0_w=1,4,1.tex
\begin{tikzpicture}
    \begin{axis}[
        width = \tikzfigwidth,
        xlabel = {\(x=y\)},
        ylabel = {\(d\)},
        ymin = 0,
        ymax = 6,
        xmin = -3,
        xmax = 3,
        cycle list/Set1,
        legend pos = south west,
        unit vector ratio = 1 1 1
    ]
        \addplot+[thick, mark=none] table [x=y, y=d, col sep=comma] {data/x=y_theta=0/exact_data_x=y_theta=0_w=1,4,1.csv};
        \addlegendentry{\(d\)}
        \addplot+[thick, mark=none, domain=-3:3, samples=100]{ l(x,x,0,1,4,1) };
        \addlegendentry{\(l\)}
        \addplot+[thick, mark=none, domain=-3:3, samples=100]{ u1(x,x,0,1,4,1) };
        \addlegendentry{\(u_1\)}
        \addplot+[thick, mark=none, domain=-3:3, samples=100]{ u2(x,x,0,1,4,1) };
        \addlegendentry{\(u_2\)}
        \legend{}; 
    \end{axis}
\end{tikzpicture}

%% file: sections/experiments.tex
\section{Experiments} \label{sec:experiments}
\subsection{Error of Half Angle Approximation} \label{sec:quantitative_analysis_rho_b}
    We can quantitatively analyse the error between any distance approximation \(\rho\) and the exact Riemannian distance \(d\) as follows. We do this by first choosing a region \(\Omega \subseteq \bbM_2\) in which we will analyse the approximation. Just as in \Cref{tab:balls,tab:balls_high_anisotropy} we decided to inspect \(\Omega := [-3,3]\times[-3,3]\times[-\pi,\pi) \subseteq \bbM_2\). As for our exact measure of error \(\e\) we have decided on the \textit{mean relative error} defined as:
    \begin{equation}
        \e := \frac{1}{\mu(\Omega)} \int_{\Omega} \frac{\Abs{\rho_b(\bp) - d(\bp)}}{d(\bp)} d\mu(\bp)
    \end{equation}
    where \(\mu\) is the induced Riemannian measure determined by the Riemannian metric \(\cG\). We then discretized our domain \(\Omega\) into a grid of \(101 \times 101 \times 101\) equally spaced points \(\bp_i\ \in \Omega\) indexed by some index set \(i \in I\) and numerically solved for the exact distance \(d\) on this grid. This numerical scheme is of course not exact and we will refer to these values as \(\tilde d_i \approx d(\bp_i)\). We also calculate the value of the distance approximation \(\rho\) on the grid points \(\rho_i := \rho(\bp_i)\). Once we have these values we can approximate the true mean relative error \(\e\) by calculating the numerical error \(\tilde \e\) defined by:
    \begin{equation}
        \e \approx \tilde \e := \frac{1}{\Abs{I}} \sum_{i \in I} \frac{\Abs{\rho_i - \tilde d_i}}{\tilde d_i}
    \end{equation}
    
    In \Cref{tab:error_rho_b_and_d} the numerical mean relative error \(\tilde \e\) between the half-angle approximation \(\rho_b\) and the numerical Riemannian distance \(\tilde d\) can be seen for different spatial anisotropies \(\zeta\). We keep \(w_1=w_3=1\) constant and vary \(w_2\). We see that, as shown visually in \Cref{tab:balls,tab:balls_high_anisotropy}, that \(\rho_b\) gets worse and worse when we increase the spatial anisotropy \(\zeta\).
    
    There is an discrepancy in the table worth mentioning. We know from \Cref{rem:rho_b_becomes_exact} that when \(\zeta = 1\) then \(\rho_b = d\) and thus \(\e = 0\). But surprisingly we do not have \(\tilde \e = 0\) in the \(\zeta = 1\) case in \Cref{tab:error_rho_b_and_d}. This can be simply explained by the fact that the numerical solution \(\tilde d\) is not exactly equal to the true distance \(d\). We expect that \(\tilde \e\) will go to 0 in the \(\zeta = 1\) case if we discretize our region \(\Omega\) more and more finely.
    
    We can compare these numerical results to our theoretical results. Namely, we can deduce from \Cref{eq:rho_b_and_d} that:
    \begin{equation}
         \frac{\Abs{\rho_b - d}}{d} \leq \zeta - 1,
    \end{equation}
    which means
    \begin{equation} \label{eq:bound_on_mean_relative_error}
        \e \leq \zeta - 1.
    \end{equation}
    And so we expect this to also approximately hold for the numerical mean relative error \(\tilde \e\). Indeed, in \Cref{tab:error_rho_b_and_d} we can see that \( \tilde \e \lessapprox \zeta - 1\). 
    
    Interestingly, we see that \(\tilde \e\) is relatively small compared to our theoretical bound \eqref{eq:bound_on_mean_relative_error} even in the high anisotropy cases. However, this is only a consequence of relative smallness of \(\Omega\). If we make \(\Omega\) bigger and bigger we can be certain that \(\e\) converges to \(\zeta - 1\). This follows from an argument similar to the reasoning in \Cref{res:rho_b_goes_certainly_bad}.
    
    \begin{table}
        \centering
        \begin{tabular}{c|cccccccc}
            \(\zeta\) & \(1\) & \(1.5\) & \(2\) & \(3\) & \(4\) & \(6\) & \(8\)\\
            \(\tilde \e\) & 0.027 & 0.051 & 0.14 & 0.41 & 0.71 & 1.4 & 2.1 
        \end{tabular}
        \caption{Numerical mean relative error \(\tilde \e\) between \(\rho_b\) and \(d\) for multiple spatial anisotropies \(\zeta\).}
        \label{tab:error_rho_b_and_d}
    \end{table}

\subsection{DCA1}

    The DCA1 dataset is a publicly available database ``consisting of 130 X-ray coronary angiograms, and their corresponding ground-truth image outlined by an expert cardiologist'' \cite{sanchez2019segmentation}. One such angiogram and ground-truth can be seen in Figures \ref{fig:dca_sample_input} and \ref{fig:dca_sample_gt}.
    

    We have split the DCA1 dataset \cite{sanchez2019segmentation} into a training and test set consisting of 125 and 10 images respectively.
    
    To establish a baseline we ran a 3, 6, and 12 layer CNN, G-CNN and PDE-G-CNN on DCA1. The exact architectures are identical/analogous to the ones used in \cite[Fig.15]{smets2022pdebased}. For the baseline the logarithmic distance approximation \(\rho_c\) was used within the PDE-G-CNNs. This is the same approximation that was used in \cite{smets2022pdebased}. Every network was trained 10 times for 80 epochs. After every epoch the average Dice coefficient on the test set is stored. After every full training the maximum of the average Dice coefficients over all 80 epochs is calculated. The result is 10 maximum average Dice coefficients for every architecture. The result of this baseline can be seen in \Cref{fig:baseline_dca1_scatter}. The amount of parameters of the networks can be found in \Cref{tab:parameters_baseline_dca1}. We see that PDE-G-CNNs consistently perform equally well as, and sometimes outperform, G-CNNs and CNNs, all the while having the least amount of parameters of all architectures.

    \begin{table}
        \centering
        \def\arraystretch{1.25}
        \setlength{\tabcolsep}{0.5em} 
        \begin{tabular}{l|rrr}
            Parameters & 3 layers & 6 layers & 12 layers\\ \hline
            CNN        & 2814     & 25662    & 73614 \\
            G-CNN      & 2058     & 24632    & 72728 \\
            PDE-G-CNN  & 1264     & 2560     & 2698
        \end{tabular}
        \caption{The total amount of parameters in the networks that are used in \Cref{fig:baseline_dca1_scatter}. }
        \label{tab:parameters_baseline_dca1}
    \end{table}

    \renewcommand{\tikzfigwidth}{0.8\linewidth}
    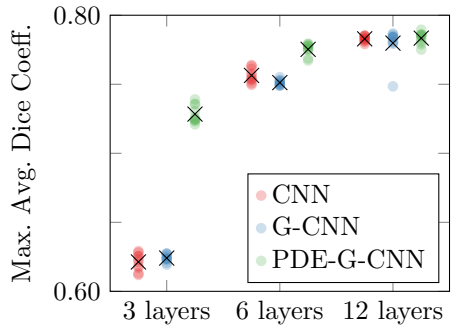
\begin{figure}
        \centering
        \input{figures/baseline_dca1_scatter}
        \caption{A scatterplot showing how a 3, 6, and 12 layer CNN, G-CNN, and PDE-G-CNN compare on the DCA1 dataset. The crosses indicate the mean. We see the PDE-G-CNNs provide equal or better results with respectively $2$, $10$ and $35$ times less parameters, see \Cref{tab:parameters_baseline_dca1}.}
        \label{fig:baseline_dca1_scatter}
    \end{figure}
    
    To compare the effect of using different approximative distances we decided to train the 6 layer PDE-G-CNN (with 2560 parameters) 10 times for 80 epochs using each distance approximation. The results can be found in \Cref{fig:dca1_distance_approximations_scatter,fig:example_outputs_dca1_distance_approximations}. We see that on DCA1 all distance approximations have a comparable performance. We notice a small dent in effectiveness when using \(\rho_{b,sr}\), and a small increase when using \(\rho_{b,com}\). 
    
    \renewcommand{\tikzfigwidth}{0.8\linewidth}
    \begin{figure}
        \centering
        \input{figures/different_approximations_dca1_scatter}
        \caption{A scatterplot showing how the use of different distance approximations affect the performance of the 6 layer PDE-G-CNN on the DCA1 dataset. The crosses indicate the mean.}
        \label{fig:dca1_distance_approximations_scatter}
    \end{figure}
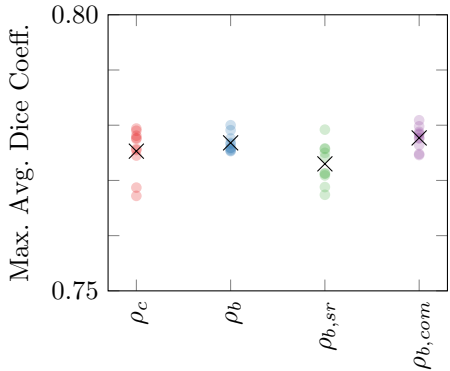

    \renewcommand{\figwidth}{0.32\linewidth}
    \begin{figure}
        \centering%
        \begin{subfigure}{\figwidth}
            \centering
            \includegraphics[width=\linewidth]{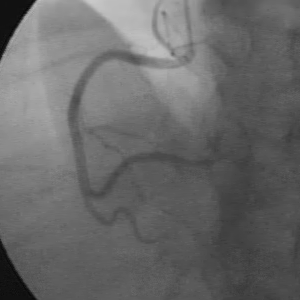}
            \caption{Input}
            \label{fig:dca_sample_input}
        \end{subfigure}
        \begin{subfigure}{\figwidth}
            \centering
            \includegraphics[width=\linewidth]{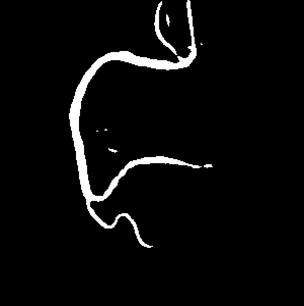}
            \caption{\(\rho_c\)}
        \end{subfigure}
        \begin{subfigure}{\figwidth}
            \centering
            \includegraphics[width=\linewidth]{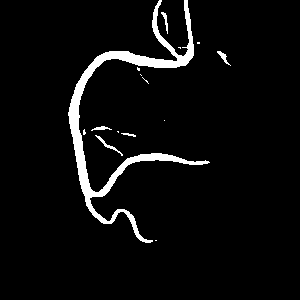}
            \caption{\(\rho_b\)}
        \end{subfigure}\\
        \begin{subfigure}{\figwidth}
            \centering
            \includegraphics[width=\linewidth]{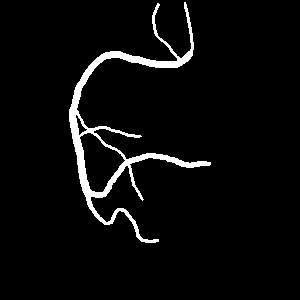}
            \caption{Truth}
            \label{fig:dca_sample_gt}
        \end{subfigure}
        \begin{subfigure}{\figwidth}
            \centering
            \includegraphics[width=\linewidth]{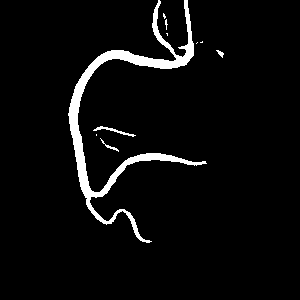}
            \caption{\(\rho_{b,sr}\)}
        \end{subfigure}
        \begin{subfigure}{\figwidth}
            \centering
            \includegraphics[width=\linewidth]{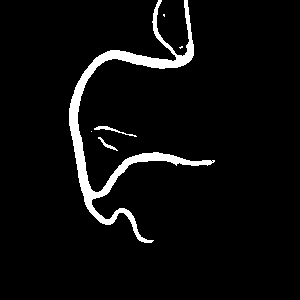}
            \caption{\(\rho_{b,com}\)}
        \end{subfigure}
        \caption{In \ref{fig:dca_sample_input} and \ref{fig:dca_sample_gt} we see one sample from the DCA1 dataset: a coronary angiogram together with the ground-truth segmentation. The other four pictures show the output of the 6 layer PDE-G-CNN, one for each distance approximation. 
        The networks that were used in this figure have an accuracy approximately equal to the mean accuracy in \Cref{fig:dca1_distance_approximations_scatter}. 
        }
        \label{fig:example_outputs_dca1_distance_approximations}
    \end{figure}


\subsection{Lines}

    For the line completion problem we created a dataset of 512 training images and 128 test images\footnote{The lines dataset is available from the authors on request.}. \Cref{fig:lines_sample_input,fig:lines_sample_gt} show one sample of the Lines dataset. 
    
    To establish a baseline we ran a 6 layer CNN, G-CNN and PDE-G-CNN. For this baseline we again used \(\rho_{c}\) within the the PDE-G-CNN, but changed the amount of channels to 30, and the kernel sizes to \([9,9,9]\), making the total amount of parameters 6018. By increasing the kernel size we anticipate that the difference in effectiveness of using the different distance approximations, if there is any, becomes more pronounced.
    Every network was trained 15 times for 60 epochs. The result of this baseline can be seen in \Cref{fig:baseline_lines_scatter}. The amount of parameters of the networks can be found in \Cref{tab:parameters_baseline_lines}. We again see that the PDE-G-CNN outperforms the G-CNN, which in turn outperforms the CNN, while having the least amount of parameters.
    
    \begin{table}
        \centering
        \def\arraystretch{1.25}
        \setlength{\tabcolsep}{0.5em} 
        \begin{tabular}{l|ccc}
                       & CNN   & G-CNN & PDE-G-CNN\\\hline
            Parameters & 25662 & 24632 & 6018
        \end{tabular}
        \caption{The total amount of parameters in the networks that are used in \Cref{fig:baseline_lines_scatter}. }
        \label{tab:parameters_baseline_lines}
    \end{table}
    
    \renewcommand{\tikzfigwidth}{0.8\linewidth}
    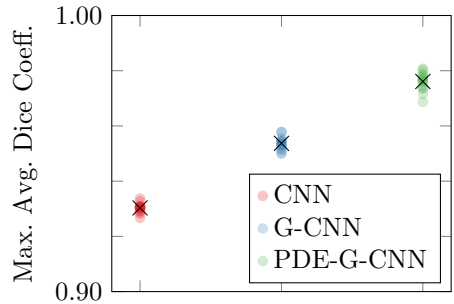
\begin{figure}
        \centering
        \input{figures/baseline_lines_scatter}
        \caption{A scatterplot showing how a 6 layer CNN, G-CNN (both with $\approx 25k$ parameters), and a PDE-G-CNN (with only 6k parameters) compare on the Lines dataset. The crosses indicate the mean. For the precise amount of parameters see \Cref{tab:parameters_baseline_lines}.}
        \label{fig:baseline_lines_scatter}
    \end{figure}

    We again test the effect of using different approximative distances by training the 6 layer PDE-G-CNN 15 times for 60 epochs for every approximation. The results can be found in \Cref{fig:lines_distance_approximations_scatter}. We see that on the Lines dataset all distance approximations again have a comparable performance. We again notice an increase in effectiveness when using \(\rho_{b,com}\), just as on the DCA1 dataset. Interestingly, using \(\rho_{b,sr}\) does not seem to hurt the performance on the Lines dataset, which is in contrast with DCA1. This is in line with what one would expect in view of the existing sub-Riemannian line-perception models in neurogeometry.

    \renewcommand{\tikzfigwidth}{0.8\linewidth}
    \begin{figure}
        \centering
        \input{figures/different_approximations_lines_scatter}
        \caption{A scatterplot showing how the use of different distance approximations affect the performance of the 6 layer PDE-G-CNN on the Lines dataset. The crosses indicate the mean.}
        \label{fig:lines_distance_approximations_scatter}
    \end{figure}
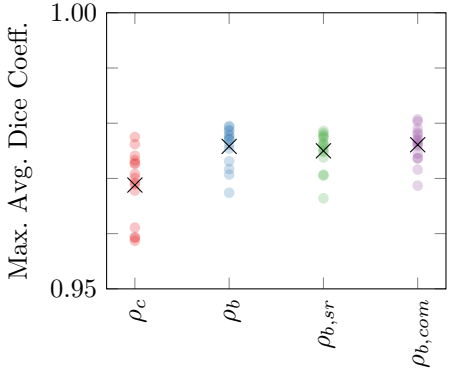

    \renewcommand{\figwidth}{0.32\linewidth}
    \begin{figure}
        \centering%
        \begin{subfigure}{\figwidth}
            \centering
            \includegraphics[width=\linewidth]{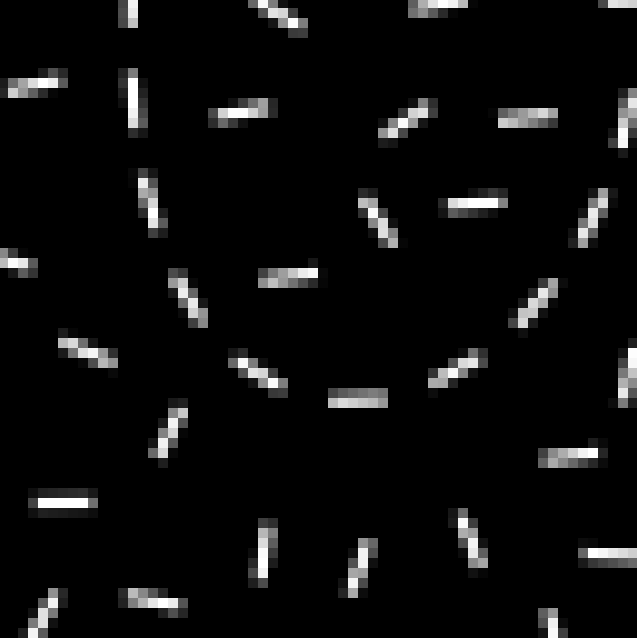}
            \caption{Input}
            \label{fig:lines_sample_input}
        \end{subfigure}
        \begin{subfigure}{\figwidth}
            \centering
            \includegraphics[width=\linewidth]{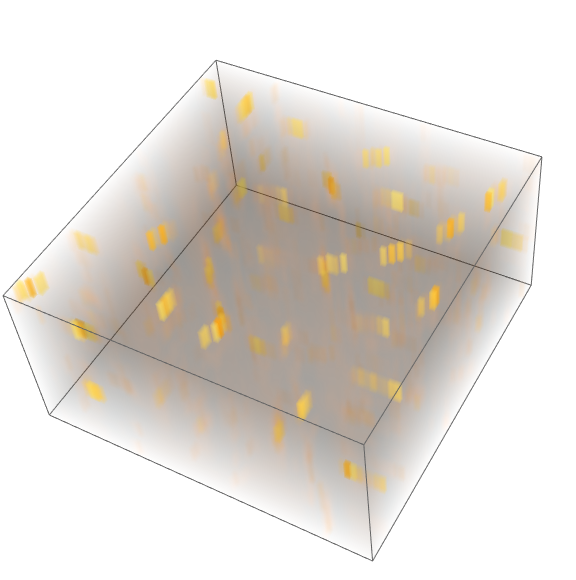}
            \caption{}
            \label{fig:lines_after_lifting}
        \end{subfigure}
        \begin{subfigure}{\figwidth}
            \centering
            \includegraphics[width=\linewidth]{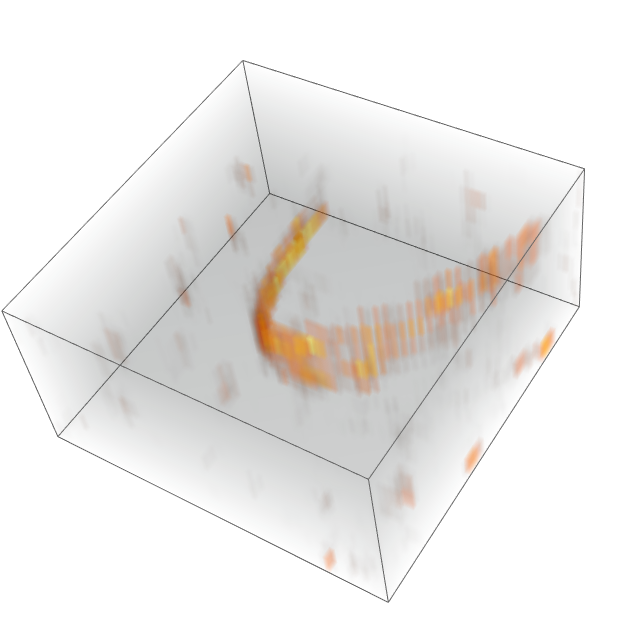}
            \caption{}
            \label{fig:lines_before_final_projection}
        \end{subfigure}\\
        \begin{subfigure}{\figwidth}
            \centering
            \includegraphics[width=\linewidth]{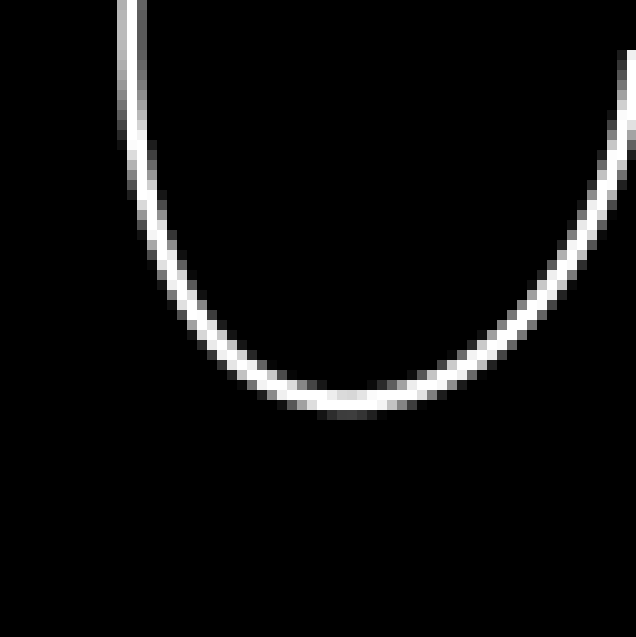}
            \caption{Truth}
            \label{fig:lines_sample_gt}
        \end{subfigure}
        \begin{subfigure}{\figwidth}
            \centering
            \includegraphics[width=\linewidth]{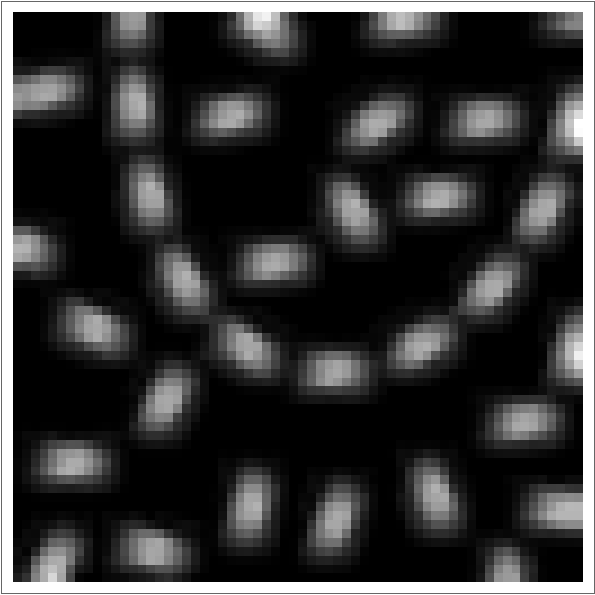}
            \caption{}
            \label{fig:lines_after_lifting_projected}
        \end{subfigure}
        \begin{subfigure}{\figwidth}
            \centering
            \includegraphics[width=\linewidth]{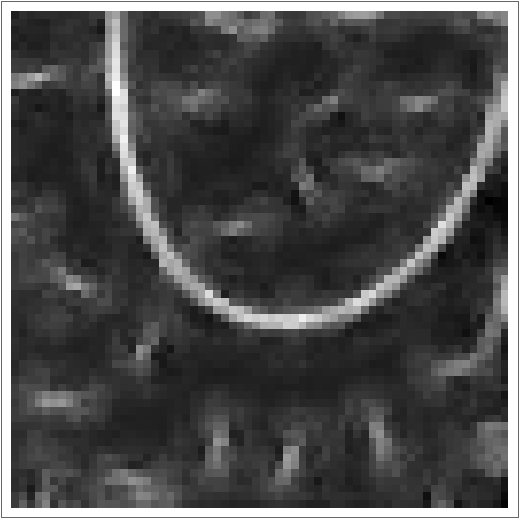}
            \caption{}
            \label{fig:lines_before_final_projection_projected}
        \end{subfigure}
        \caption{In \ref{fig:lines_sample_input} and \ref{fig:lines_sample_gt} we see one sample from the Lines dataset. The other four pictures are visualizations of feature maps of the 6 layer PDE-G-CNN. In \ref{fig:lines_after_lifting} and \ref{fig:lines_after_lifting_projected} we see a feature map of the lifting layer together with its max-projection over \(\theta\). In \ref{fig:lines_before_final_projection} and \ref{fig:lines_before_final_projection_projected} we see a feature map of the last PDE layer, just before the final projection layer.}
        \label{fig:lines_input_gt_volumes_and_projections}
    \end{figure}

%% file: figures/baseline_dca1_scatter.tex
\begin{tikzpicture}
    \begin{axis}[
        width = \tikzfigwidth,
        xtick = {2, 6, 10},
        xticklabels = { 3 layers, 6 layers, 12 layers},
        ymin = 0.60,
        ymax = 0.80,
        ytick =       {0.60, 0.65, 0.70, 0.75, 0.80},
        yticklabels = {0.60,     ,     ,     , 0.80},
        ylabel = {Max. Avg. Dice Coeff.},
        legend pos = south east,
        legend cell align={left},
    ]

        \addplot[
            only marks, 
            fill opacity = 0.25, 
            draw opacity = 0,
            fill = Set1-A
        ] coordinates {
            (1, 0.6256)
            (1, 0.6253)
            (1, 0.6230)
            (1, 0.6119)
            (1, 0.6130)
            (1, 0.6172)
            (1, 0.6183)
            (1, 0.6281)
            (1, 0.6291)
            (1, 0.6215)

            (5, 0.7545)
            (5, 0.7578)
            (5, 0.7639)
            (5, 0.7631)
            (5, 0.7612)
            (5, 0.7564)
            (5, 0.7508)
            (5, 0.7527)
            (5, 0.7496)
            (5, 0.7536)
            
            (9, 0.7843)
            (9, 0.7817)
            (9, 0.7853)
            (9, 0.7791)
            (9, 0.7830)
            (9, 0.7850)
            (9, 0.7830)
            (9, 0.7805)
            (9, 0.7818)
            (9, 0.7851)
        };
        \addlegendentry{CNN}
        
        \addplot[
            only marks, 
            fill opacity = 0.25, 
            draw opacity = 0,
            fill = Set1-B
        ] coordinates {
            (2, 0.6254)
            (2, 0.6275)
            (2, 0.6207)
            (2, 0.6273)
            (2, 0.6193)
            (2, 0.6268)
            (2, 0.6235)
            (2, 0.6237)
            (2, 0.6227)
            (2, 0.6238)
            
            (6, 0.7489)
            (6, 0.7532)
            (6, 0.7504)
            (6, 0.7514)
            (6, 0.7523)
            (6, 0.7497)
            (6, 0.7524)
            (6, 0.7492)
            (6, 0.7495)
            (6, 0.7552)
            
            (10, 0.7840)
            (10, 0.7868)
            (10, 0.7859)
            (10, 0.7485)
            (10, 0.7802)
            (10, 0.7809)
            (10, 0.7793)
            (10, 0.7840)
            (10, 0.7841)
            (10, 0.7846)
        };
        \addlegendentry{G-CNN}

        \addplot[
            only marks, 
            fill opacity = 0.25, 
            draw opacity = 0,
            fill = Set1-C
        ] coordinates {
            (3, 0.7354)
            (3, 0.7306)
            (3, 0.7246)
            (3, 0.7210)
            (3, 0.7231)
            (3, 0.7243)
            (3, 0.7358)
            (3, 0.7242)
            (3, 0.7259)
            (3, 0.7390)
            
            (7, 0.7777)
            (7, 0.7794)
            (7, 0.7781)
            (7, 0.7790)
            (7, 0.7757)
            (7, 0.7672)
            (7, 0.7775)
            (7, 0.7687)
            (7, 0.7755)
            (7, 0.7745)
            
            (11, 0.7823)
            (11, 0.7845)
            (11, 0.7895)
            (11, 0.7856)
            (11, 0.7854)
            (11, 0.7751)
            (11, 0.7789)
            (11, 0.7804)
            (11, 0.7858)
            (11, 0.7862)
        };
        \addlegendentry{PDE-G-CNN}
        
        \addplot[
            only marks, 
            color = black,
            mark = x, 
            dashed,
            mark options={solid},
            mark size=4
        ] coordinates {
            (1,0.6213) 
            (2,0.6241) 
            (3,0.7284) 
            
            (5,0.7564) 
            (6,0.7512) 
            (7,0.7753) 
            
            (9,0.7829) 
            (10,0.7798) 
            (11,0.7834) 
        };
    \end{axis}
\end{tikzpicture}

%% file: figures/different_approximations_dca1_scatter.tex
\begin{tikzpicture}
    \begin{axis}[
        width = \tikzfigwidth,
        xtick = {1,2,3,4},
        xticklabels = {\(\rho_c\), \(\rho_b\), \(\rho_{b,sr}\), \(\rho_{b,com}\)},
        ymin = 0.75,
        ymax = 0.80, 
        ytick =       {0.75, 0.76, 0.77, 0.78, 0.79, 0.80},
        yticklabels = {0.75,     ,     ,     ,     , 0.80},
        ylabel = {Max. Avg. Dice Coeff.},
        xticklabel style = {rotate = 90, anchor = east},
    ]

        \addplot[
            scatter, 
            only marks, 
            point meta=explicit symbolic,
            scatter/classes={
                a={fill=Set1-A, fill opacity=0.25, draw opacity=0},
                b={fill=Set1-B, fill opacity=0.25, draw opacity=0},
                c={fill=Set1-C, fill opacity=0.25, draw opacity=0},
                d={fill=Set1-D, fill opacity=0.25, draw opacity=0}
            },
        ] table [meta=label] {
            x y      label
            1 0.7777 a
            1 0.7794 a
            1 0.7781 a
            1 0.7790 a
            1 0.7757 a
            1 0.7672 a
            1 0.7775 a
            1 0.7687 a
            1 0.7755 a
            1 0.7745 a
            
            2 0.7756 b
            2 0.7776 b
            2 0.7800 b
            2 0.7759 b
            2 0.7753 b
            2 0.7768 b
            2 0.7766 b
            2 0.7791 b
            2 0.7754 b
            2 0.7759 b
            
            3 0.7688 c
            3 0.7757 c
            3 0.7712 c
            3 0.7750 c
            3 0.7742 c
            3 0.7758 c
            3 0.7674 c
            3 0.7713 c
            3 0.7709 c
            3 0.7792 c
            
            4 0.7788 d 
            4 0.7809 d
            4 0.7783 d
            4 0.7785 d
            4 0.7748 d
            4 0.7764 d
            4 0.7798 d
            4 0.7746 d
            4 0.7776 d
            4 0.7777 d
        };
        
        \addplot[
            only marks,
            color = black,
            mark = x, 
            dashed,
            mark options={solid},
            mark size=4
        ] coordinates {
            (1,0.7753) 
            (2,0.7768) 
            (3,0.7730) 
            (4,0.7777) 
        };
    \end{axis}
\end{tikzpicture}

%% file: figures/baseline_lines_scatter.tex
\begin{tikzpicture}
    \begin{axis}[
        width = \tikzfigwidth,
        xtick = {1,2,3},
        xticklabels = { , , },
        xticklabel style = {rotate = 90, anchor = east},
        ymin = 0.90,
        ymax = 1.00,
        ytick =       {0.90, 0.92, 0.94, 0.96, 0.98, 1.00},
        yticklabels = {0.90,     ,     ,     ,     , 1.00},
        ylabel = {Max. Avg. Dice Coeff.},
        legend pos = south east,
        legend cell align={left},
    ]

        \addplot[
            only marks, 
            fill opacity = 0.25, 
            draw opacity = 0,
            fill = Set1-A
        ] coordinates {
            (1, 0.9307)
            (1, 0.9292)
            (1, 0.9304)
            (1, 0.9283)
            (1, 0.9301)
            (1, 0.9327)
            (1, 0.9304)
            (1, 0.9289)
            (1, 0.9309)
            (1, 0.9337)
            (1, 0.9267)
            (1, 0.9308)
            (1, 0.9299)
            (1, 0.9305)
            (1, 0.9309)
        };
        \addlegendentry{CNN}
        
        \addplot[
            only marks, 
            fill opacity = 0.25, 
            draw opacity = 0,
            fill = Set1-B
        ] coordinates {
            (2, 0.9540)
            (2, 0.9544)
            (2, 0.9533)
            (2, 0.9534)
            (2, 0.9528)
            (2, 0.9544)
            (2, 0.9514)
            (2, 0.9555)
            (2, 0.9548)
            (2, 0.9531)
            (2, 0.9518)
            (2, 0.9578)
            (2, 0.9504)
            (2, 0.9499)
            (2, 0.9578)
        };
        \addlegendentry{G-CNN}
        
        \addplot[
            only marks, 
            fill opacity = 0.25, 
            draw opacity = 0,
            fill = Set1-C
        ] coordinates {
            (3, 0.9759)
            (3, 0.9807)
            (3, 0.9745)
            (3, 0.9783)
            (3, 0.9687)
            (3, 0.9803)
            (3, 0.9762)
            (3, 0.9765)
            (3, 0.9769)
            (3, 0.9790)
            (3, 0.9779)
            (3, 0.9737)
            (3, 0.9736)
            (3, 0.9716)
            (3, 0.9772)
        };
        \addlegendentry{PDE-G-CNN}
        
        \addplot[
            only marks,
            color = black,
            mark = x, 
            dashed,
            mark options={solid},
            mark size=4
        ] coordinates {
            (1,0.9303) 
            (2,0.9537) 
            (3,0.9761) 
        };
    \end{axis}
\end{tikzpicture}

%% file: figures/different_approximations_lines_scatter.tex
\begin{tikzpicture}
    \begin{axis}[
        width = \tikzfigwidth,
        xtick = {1,2,3,4},
        xticklabels = {\(\rho_c\), \(\rho_b\), \(\rho_{b,sr}\), \(\rho_{b,com}\)},
        ymin = 0.95,
        ymax = 1.00,
        ytick =       {0.95, 0.96, 0.97, 0.98, 0.99, 1.00},
        yticklabels = {0.95,     ,     ,     ,     , 1.00},
        ylabel = {Max. Avg. Dice Coeff.},
        xticklabel style = {rotate = 90, anchor = east},
    ]

        \addplot[
            scatter, 
            only marks, 
            point meta=explicit symbolic,
            scatter/classes={
                a={fill=Set1-A, fill opacity=0.25, draw opacity=0},
                b={fill=Set1-B, fill opacity=0.25, draw opacity=0},
                c={fill=Set1-C, fill opacity=0.25, draw opacity=0},
                d={fill=Set1-D, fill opacity=0.25, draw opacity=0}
            },
        ] table [meta=label] {
            x y      label
            1 0.9587 a
            1 0.9701 a
            1 0.9594 a
            1 0.9741 a
            1 0.9689 a
            1 0.9678 a
            1 0.9708 a
            1 0.9611 a
            1 0.9726 a
            1 0.9692 a
            1 0.9733 a
            1 0.9592 a
            1 0.9728 a
            1 0.9775 a
            1 0.9762 a
            
            2  0.9772 b
            2  0.9771 b
            2  0.9793 b
            2  0.9786 b
            2  0.9707 b
            2  0.9754 b
            2  0.9787 b
            2  0.9795 b
            2  0.9777 b
            2  0.9731 b
            2  0.9779 b
            2  0.9717 b
            2  0.9761 b
            2  0.9770 b
            2  0.9674 b
            
            3 0.9664 c
            3 0.9786 c
            3 0.9778 c
            3 0.9750 c
            3 0.9782 c
            3 0.9762 c
            3 0.9705 c
            3 0.9776 c
            3 0.9738 c
            3 0.9707 c
            3 0.9765 c
            3 0.9773 c
            3 0.9760 c
            3 0.9755 c
            3 0.9748 c
            
            4 0.9759 d
            4 0.9807 d
            4 0.9745 d
            4 0.9783 d
            4 0.9687 d
            4 0.9803 d
            4 0.9762 d
            4 0.9765 d
            4 0.9769 d
            4 0.9790 d
            4 0.9779 d
            4 0.9737 d
            4 0.9736 d
            4 0.9716 d
            4 0.9772 d
        };
        
        \addplot[
            only marks,
            color = black,
            mark = x, 
            dashed,
            mark options={solid},
            mark size=4
        ] coordinates {
            (1,0.9688) 
            (2,0.9758) 
            (3,0.9750) 
            (4,0.9761) 
        };
    \end{axis}
\end{tikzpicture}

%% file: sections/conclusion.tex
\section{Conclusion} \label{sec:conclusion}

In this article we have carefully analyzed how well the nonlinear erosion and dilation parts of PDE-G-CNNs are actually solved on the homogeneous space of 2D positions and orientations \(\bbM_2\). According to \Cref{res:kernels} the Hamilton-Jacobi equations are solved by morphological kernels that are functions of only the exact (sub)-Riemannian distance function. As a result, every approximation of the exact distance yields a corresponding approximative morphological kernel.

In \Cref{res:main_results} we use this to improve upon local and global approximations of the relative errors of the erosion and dilations kernels used in the papers \cite{smets2022pdebased,duits2021equivariant} where PDE-G-CNN are first proposed (and shown to outperform G-CNNs). Our new sharper estimates for distance on \(\bbM_2\) have bounds that explicitly depend on the metric tensor field coefficients. This allowed us to theoretically underpin the earlier worries expressed in \cite[Fig.10]{smets2022pdebased} that if spatial anisotropy becomes high the previous morphological kernel approximations \cite{smets2022pdebased} become less and less accurate.

Indeed, as we show qualitatively in \Cref{tab:balls_high_anisotropy} and quantitatively in \Cref{sec:quantitative_analysis_rho_b}, if the spatial anisotropy \(\zeta\) is high one must resort to sub-Riemannian approximations. Furthermore, we propose a single distance approximation \(\rho_{b,com}\) that works both for low and high spatial anisotropy. 

Apart from how well the kernels approximate the PDEs, there is the issue of how well each of the distance approximations perform in applications within the PDE-G-CNNs. In practice the analytic approximative kernels using \(\rho_b\), \(\rho_c\), \(\rho_{b,com}\) perform similarly. This is not surprising as our theoretical result \Cref{res:distance_symmetries,res:kernel_symmetries} reveals that all morphological kernel approximations carry the correct 8 fundamental symmetries of the PDE. Nevertheless, \Cref{fig:dca1_distance_approximations_scatter,fig:lines_distance_approximations_scatter} do reveal advantages of using the new kernel approximations (in particular \(\rho_{b,com}\)) over the previous kernel \(\rho_c\) in \cite{smets2022pdebased}. 

The experiments also show that the strictly sub-Riemannian distance approximation \(\rho_{b,sr}\) only performs well on applications where sub-Riemannian geometry really applies. For instance, as can be seen in \Cref{fig:dca1_distance_approximations_scatter,fig:lines_distance_approximations_scatter}, on the DCA1 dataset \(\rho_{b,sr}\) performs relatively poor, whereas on the Lines dataset,  \(\rho_{b,sr}\) performs well. This is what one would expect in view of sub-Riemannian models and findings in cortical line-perception \cite{citti2006cortical,petitot2003neurogeometry,baspinar2018geometric,baspinar2021cortical,franceschiello2019geometrical,petitot2017elements} in neurogeometry.

Besides better accuracy and better performance of the approximative kernels, there is the issue of geometric interpretability. In G-CNNs and CNNs geometric interpretability is absent, as they include ad-hoc nonlinearities like ReLUs. PDE-G-CNNs instead employ morphological convolutions with kernels that reflect association fields, as visualized in \Cref{fig:geodesics_association_field}. In \Cref{fig:feature_maps} we see that as network depth increases association fields visually merge in the feature maps of PDE-G-CNNs towards adaptive line detectors, whereas such merging/grouping of association fields is not visible in normal CNNs.

In all cases, the PDE-G-CNNs still outperform G-CNNs and CNNs on the DCA1 dataset and Lines dataset: they have a higher (or equal) performance, while having a huge reduction in network complexity, even when using 3 layers. Regardless, the choice of kernel \(\rho_c\), \(\rho_b\), \(\rho_{b,sr}\), \(\rho_{b,com}\) the advantage of PDE-G-CNNs towards \mbox{G-CNNs} and CNNs is significant, as can be clearly observed in \Cref{fig:baseline_dca1_scatter,fig:baseline_lines_scatter} and \Cref{tab:parameters_baseline_dca1,tab:parameters_baseline_lines}. This is in line with previous observations on other datasets \cite{smets2022pdebased}. 

Altogether, PDE-G-CNNs have better geometric reduction, performance, and geometric interpretation, than basic classical feed-forward (G)-CNN networks on various segmentation problems. 

Extensive investigations on training data reduction, memory reduction (via U-Net versions of PDE-G-CNNs), and a topological description of the merging of association fields are beyond the scope of this article, and are left for future work.

%% file: sections/appendix.tex
\crefalias{section}{appendix} 

\section{Proof of \texorpdfstring{\Cref{res:bound_on_differential_norm_rho_b}}{bound on norm derivative rhob}} \label{app:proof_bound_on_differential_norm_rho_b}

\begin{proof}
    We start by writing out the explicit form of $\Vert d \rho_b\Vert^2$ in the left-invariant frame:
    \begin{equation*}
        \Vert d \rho_b\Vert^2 = w_1^{-2} (\cA_1 \rho_b)^2 + w_2^{-2} (\cA_2 \rho_b)^2 + w_3^{-2} (\cA_3 \rho_b)^2.
    \end{equation*}
    By replacing the left-invariant derivatives with half-angle coordinates derivatives we can equivalently write this as: 
    {\small
    \begin{equation*}
    \begin{split}
        & w_2^{-2} \Par{ 
              \Abs{\pdv{\rho}{b^1}}^2
            + \Abs{\pdv{\rho}{b^2}}^2 
        } \\
        +&(w_1^{-2} - w_2^{-2}) 
        \Abs{
              \cos\Par{\frac{b^3}{2}}\pdv{\rho}{b^1}
            + \sin\Par{\frac{b^{3}}{2}} \pdv{\rho}{b^2}
        }^2 \\
        +& w_3^{-2} \Abs{ 
              \frac{1}{2}\pdv{\rho}{\psi}
            + \pdv{\rho}{b^3} 
        }^2,
    \end{split}
    \end{equation*}
    }
    where \(\psi = \arctantwo(b^2, b^1)\), \(\pd_\psi = b^2 \pd_{b^1} - b^1 \pd_{b^2} \), and we omitted the subscript \(b\) from \(\rho\) for conciseness. We are going to Taylor expand the sin and cosine in the second term up to the second order term. This becomes
    \begin{multline*}
        \Abs{\cos\Par{\frac{b^3}{2}} \pdv{\rho}{b^1}
        + 
        \sin\Par{\frac{b^{3}}{2}} \pdv{\rho}{b^2} }^2
        =  \Abs{\pdv{\rho}{b^1}}^2 \\
        + \theta \Par{ \pdv{\rho}{b^1} \pdv{\rho}{b^2}  } 
        + \frac{\theta^2}{4} \Par{ \Abs{\pdv{\rho}{b^2}}^2 - \Abs{\pdv{\rho}{b^1}}^2  }
        + O(\theta^3).
    \end{multline*}
    This allows us to write \( \Norm{d \rho_b}^2\) as 
    \begin{equation*}
        w_1^{-2} \Abs{\pdv{\rho}{b^1}}^2 + w_2^{-2} \Abs{\pdv{\rho}{b^2}}^2 + w_3^{-2} \Abs{ \pdv{\rho}{b^3}}^2 + \e.
    \end{equation*}
    Making use of the fact that the first part in this expression equals 1, we can thus write \( \Norm{d \rho_b}^2 = 1 + \e\). The exact form of \(\e\) is as follows 
    \begin{multline*}
        \varepsilon = \frac{w_1^2 - w_2^2}{4 w_1^2 \, w_2^2 \, w_3^2 \, \rho_b^2} \biggl( w_1^4 w_3^2 (b^1 b^3)^2
        - w_2^4 w_3^2 (b^2 b^3)^2 \\
        + w_1^2w_2^2(w_1^2 - w_2^2) (b^1 b^2)^2 \biggr) + O(\theta^3).
    \end{multline*}
    Using that \(w_i \vert b^i \vert \leq \rho_b\) we can bound the expression from above by
    {\small
    \begin{equation*}
        \e \leq \rho_b^2 \frac{\Abs{ w_1^2 - w_2^2}}{4 w_1^2 \, w_2^2 \, w_3^2 \,} \left( w_1^2 + w_2^2 + \Abs{ w_1^2 - w_2^2} \right) + O(\theta^3).
    \end{equation*}
    }
    Finally the lemma follows by algebraic manipulations and the fact that \(w_1 \leq w_2\).
    
\end{proof}

\section{Geometric Interpretation of 
PDE-G-CNN layers} \label{app:geometric}

In a PDE-G-CNN layer \cite{smets2022pdebased,duits2021equivariant} one first performs convection and then a morphological convolution (dilation/erosion). This has the interesting effect that we can interpret this equivalently as performing a morphological convolution with a shifted morphological kernel. To make this precise we first define what convection exactly is:
\begin{definition}[Convection]
    Let \(v \in T_{\bp_0} (\bbM_2)\) be a tangent vector at the reference point \(\bp_0\), and let \(c : \bbM_2 \to T(\bbM_2)\) be the corresponding left-invariant vector field obtained by pushing \(v\) forward with the left-action \(L_g(\bp) := g \bp\), i.e. $c(g \bp_0)=(L_g)_* v$. Convection is defined as:
    \begin{equation*}
    \begin{cases}
        \pdv{W}{t} &= -cW\\
        \At{W}_{t=0} &= U,
    \end{cases}
    \end{equation*}
    where both \(W\) and \(U\) are scalar differentiable functions on \(\bbM_2\).
\end{definition}
The solution of this left-invariant transport (`convection') is quite simple and we state it in the following proposition without proof:
\begin{proposition}
    The solution to the convection equation is 
    \begin{equation*}
        W(g \bp_0, t) = U(g \exp(-vt) \bp_0 ).
    \end{equation*}
    where we identified \(v\) as a tangent vector in \(T_eSE(2)\).
\end{proposition}
For the proof, and more details on how convection is implemented in practice within the PDE-G-CNN framework, we refer to \cite[Sec.5.1]{smets2022pdebased}. The general idea is that the characteristics of left-invariant flow are Lie group exponential curves acting on the reference point $\bp_0 \in \bbM_2$ in the homogeneous space. 

We can now show that first performing convection and then a morphological convolution is the same as doing a morphological convolution with a shifted kernel:
\begin{proposition} \label{res:shifted_kernel}
    Let \(k : \bbM_2 \to \bbR\) be any morphological kernel. We have:
    \begin{equation*}
    \begin{split}
        (k \mathbin{\square} W)(\bp) 
        &= \inf \limits_{g \in G} \left\{k(g^{-1} \bp) + W(g \bp_0, t)\right\}\\
        &= (\hat{k} \mathbin{\square} U)(\bp),
    \end{split}
    \end{equation*}
    with shifted kernel \(\hat{k}(\bp, t) := k(\exp(-t v) \bp)\). In particular for time dependent erosion PDE kernels
    \begin{equation}\label{eq:shifted_kernel_erosion}
        \hat{k}_t(\bp)= \frac{t}{\beta}\left(\frac{d_{\mathcal{G}}(\bp,e^{tv}\bp_0)}{t}\right)^{\beta}
    \end{equation}
\end{proposition}
\begin{proof}
    Indeed, by direct computations one has:
    \begin{equation*}
    \begin{split}
        (k \mathbin{\square} W)(\bp) 
        &= \inf \limits_{g \in G} \left\{k(g^{-1} \bp) + W(g \bp_0, t)\right\}\\
        &= \inf \limits_{g \in G} \left\{k(g^{-1} \bp) + U(g \exp(-vt) \bp_0 )\right\}\\
        &= \inf \limits_{g \in G} \left\{k(\exp(-vt) g^{-1} \bp) + U(g \bp_0 )\right\}\\
        &= \inf \limits_{g \in G} \left\{\hat{k}(g^{-1} \bp, t) + U(g \bp_0 )\right\}\\
        &= (\hat{k} \mathbin{\square} U)(\bp).
    \end{split}
    \end{equation*}
    When applying this to the erosion kernels \eqref{eq:morphological_kernel_intro} the result \eqref{eq:shifted_kernel_erosion} follows by left-invariance of the Riemannian metric: $d_\cG(e^{-tv}\bp,\bp_0)=d_\cG(\bp,e^{tv}\bp_0)$ and the identity $(e^{tv})^{-1}=e^{-tv}$.    
\end{proof}

\begin{figure}
    \centering
    \includegraphics[width=1.03\linewidth]{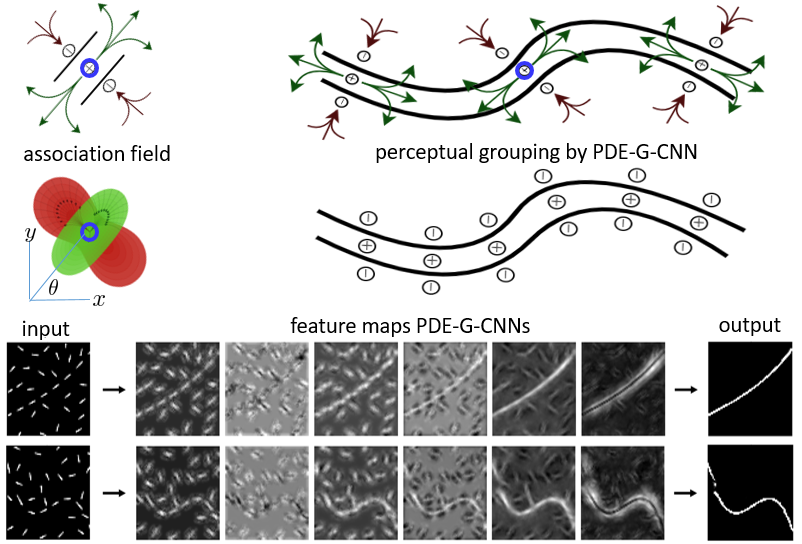}
    \caption{A PDE-G-CNN module trains left-invariant convection vector field $c$, a Riemannian ball over which we apply max-pooling (dilation for line excitation), and a Riemannian ball over which we apply min-pooling (erosion for inhibition/sharpening). Top: by \Cref{res:shifted_kernel} a PDE-G-CNN module trains: 1) a center point (blue), 2) an association field for excitation (green), and 3) an association field for inhibition (red). Bottom: As network-depth increases these association fields group together as visible in the feature maps.}
    \label{fig:geons}
\end{figure}

Recall the relation between (approximative) Riemannian balls and association fields, as visualised in \Cref{fig:association_field,fig:relation_distance_association_field,fig:relation_rho_c_association_field}.

The top left corner in \Cref{fig:geons} shows how a single PDE-G-CNN module (i.e. operator between two nodes in the network). The top-right shows the geometric rationale behind a PDE-G-CNNs that essentially performs perceptual grouping of association fields via training, and indeed the bottom two rows of \Cref{fig:geons} reveals how the grouping of association fields becomes visible in the feature maps of two input test images. 
In comparison this (for PDE-G-CNNs) typical geometric behavior is absent in feature maps of CNNs applied to the same images, recall \Cref{fig:feature_maps}.

%% file: main.bbl

\begin{thebibliography}{61}
\ifx \bisbn   \undefined \def \bisbn  #1{ISBN #1}\fi
\ifx \binits  \undefined \def \binits#1{#1}\fi
\ifx \bauthor  \undefined \def \bauthor#1{#1}\fi
\ifx \batitle  \undefined \def \batitle#1{#1}\fi
\ifx \bjtitle  \undefined \def \bjtitle#1{#1}\fi
\ifx \bvolume  \undefined \def \bvolume#1{\textbf{#1}}\fi
\ifx \byear  \undefined \def \byear#1{#1}\fi
\ifx \bissue  \undefined \def \bissue#1{#1}\fi
\ifx \bfpage  \undefined \def \bfpage#1{#1}\fi
\ifx \blpage  \undefined \def \blpage #1{#1}\fi
\ifx \burl  \undefined \def \burl#1{\textsf{#1}}\fi
\ifx \doiurl  \undefined \def \doiurl#1{\url{https://doi.org/#1}}\fi
\ifx \betal  \undefined \def \betal{\textit{et al.}}\fi
\ifx \binstitute  \undefined \def \binstitute#1{#1}\fi
\ifx \binstitutionaled  \undefined \def \binstitutionaled#1{#1}\fi
\ifx \bctitle  \undefined \def \bctitle#1{#1}\fi
\ifx \beditor  \undefined \def \beditor#1{#1}\fi
\ifx \bpublisher  \undefined \def \bpublisher#1{#1}\fi
\ifx \bbtitle  \undefined \def \bbtitle#1{#1}\fi
\ifx \bedition  \undefined \def \bedition#1{#1}\fi
\ifx \bseriesno  \undefined \def \bseriesno#1{#1}\fi
\ifx \blocation  \undefined \def \blocation#1{#1}\fi
\ifx \bsertitle  \undefined \def \bsertitle#1{#1}\fi
\ifx \bsnm \undefined \def \bsnm#1{#1}\fi
\ifx \bsuffix \undefined \def \bsuffix#1{#1}\fi
\ifx \bparticle \undefined \def \bparticle#1{#1}\fi
\ifx \barticle \undefined \def \barticle#1{#1}\fi
\bibcommenthead
\ifx \bconfdate \undefined \def \bconfdate #1{#1}\fi
\ifx \botherref \undefined \def \botherref #1{#1}\fi
\ifx \url \undefined \def \url#1{\textsf{#1}}\fi
\ifx \bchapter \undefined \def \bchapter#1{#1}\fi
\ifx \bbook \undefined \def \bbook#1{#1}\fi
\ifx \bcomment \undefined \def \bcomment#1{#1}\fi
\ifx \oauthor \undefined \def \oauthor#1{#1}\fi
\ifx \citeauthoryear \undefined \def \citeauthoryear#1{#1}\fi
\ifx \endbibitem  \undefined \def \endbibitem {}\fi
\ifx \bconflocation  \undefined \def \bconflocation#1{#1}\fi
\ifx \arxivurl  \undefined \def \arxivurl#1{\textsf{#1}}\fi
\csname PreBibitemsHook\endcsname

\bibitem{bekkers2018roto}
\begin{bchapter}
\bauthor{\bsnm{Bekkers}, \binits{E.J.}},
\bauthor{\bsnm{Lafarge}, \binits{M.W.}},
\bauthor{\bsnm{Veta}, \binits{M.}},
\bauthor{\bsnm{Eppenhof}, \binits{K.A.J.}},
\bauthor{\bsnm{Pluim}, \binits{J.P.W.}},
\bauthor{\bsnm{Duits}, \binits{R.}}:
\bctitle{Roto-translation covariant convolutional networks for medical image
  analysis}.
In: \bbtitle{International Conference on Medical Image Computing and
  Computer-Assisted Intervention},
pp. \bfpage{440}--\blpage{448}
(\byear{2018}).
\bcomment{Springer}.
\burl{https://arxiv.org/abs/1804.03393}
\end{bchapter}
\endbibitem

\bibitem{lecun1989backpropagation}
\begin{barticle}
\bauthor{\bsnm{LeCun}, \binits{Y.}},
\bauthor{\bsnm{Boser}, \binits{B.}},
\bauthor{\bsnm{Denker}, \binits{J.S.}},
\bauthor{\bsnm{Henderson}, \binits{D.}},
\bauthor{\bsnm{Howard}, \binits{R.E.}},
\bauthor{\bsnm{Hubbard}, \binits{W.}},
\bauthor{\bsnm{Jackel}, \binits{L.D.}}:
\batitle{Backpropagation applied to handwritten zip code recognition}.
\bjtitle{Neural computation}
\bvolume{1}(\bissue{4}),
\bfpage{541}--\blpage{551}
(\byear{1989})
\end{barticle}
\endbibitem

\bibitem{krizhevsky2012imagenet}
\begin{bchapter}
\bauthor{\bsnm{Krizhevsky}, \binits{A.}},
\bauthor{\bsnm{Sutskever}, \binits{I.}},
\bauthor{\bsnm{Hinton}, \binits{G.E.}}:
\bctitle{{I}mage{N}et classification with deep convolutional neural networks}.
In: \beditor{\bsnm{Pereira}, \binits{F.}},
\beditor{\bsnm{Burges}, \binits{C.J.}},
\beditor{\bsnm{Bottou}, \binits{L.}},
\beditor{\bsnm{Weinberger}, \binits{K.Q.}} (eds.)
\bbtitle{Advances in Neural Information Processing Systems},
vol. \bseriesno{25}.
\bpublisher{Curran Associates, Inc.},
\blocation{Red Hook, New York}
(\byear{2012}).
\burl{https://proceedings.neurips.cc/paper/2012/file/c399862d3b9d6b76c8436e924a68c45b-Paper.pdf}
\end{bchapter}
\endbibitem

\bibitem{litjens2017survey}
\begin{barticle}
\bauthor{\bsnm{Litjens}, \binits{G.}},
\bauthor{\bsnm{Bejnodri}, \binits{B.E.}},
\bauthor{\bsnm{Setio}, \binits{A.A.A.}},
\bauthor{\bsnm{Ciompi}, \binits{F.}},
\bauthor{\bsnm{Ghafoorian}, \binits{M.}},
\bauthor{\bparticle{van~der} \bsnm{Laak}, \binits{J.A.W.M.}},
\bauthor{\bparticle{van} \bsnm{Ginneken}, \binits{B.}},
\bauthor{\bsnm{S\'{a}nchez}, \binits{C.I.}}:
\batitle{A survey on deep learning in medical image analysis}.
\bjtitle{Medical Image Analysis}
\bvolume{42},
\bfpage{60}--\blpage{88}
(\byear{2017})
\end{barticle}
\endbibitem

\bibitem{cohen2016group}
\begin{barticle}
\bauthor{\bsnm{Cohen}, \binits{T.S.}},
\bauthor{\bsnm{Welling}, \binits{M.}}:
\batitle{Group equivariant convolutional networks}.
\bjtitle{Proc. of the 33rd Int. Conf. on Machine Learning}
\bvolume{48},
\bfpage{1}--\blpage{12}
(\byear{2016})
\end{barticle}
\endbibitem

\bibitem{dieleman2016exploiting}
\begin{botherref}
\oauthor{\bsnm{Dieleman}, \binits{S.}},
\oauthor{\bsnm{De~Fauw}, \binits{J.}},
\oauthor{\bsnm{Kavukcuoglu}, \binits{K.}}:
Exploiting cyclic symmetry in convolutional neural networks.
arXiv preprint arXiv:1602.02660
(2016)
\end{botherref}
\endbibitem

\bibitem{dieleman2015rotation}
\begin{barticle}
\bauthor{\bsnm{Dieleman}, \binits{S.}},
\bauthor{\bsnm{Willett}, \binits{K.W.}},
\bauthor{\bsnm{Dambre}, \binits{J.}}:
\batitle{Rotation-invariant convolutional neural networks for galaxy morphology
  prediction}.
\bjtitle{Monthly Notices of the Royal Astronomical Society}
\bvolume{450}(\bissue{2}),
\bfpage{1441}--\blpage{1459}
(\byear{2015})
\end{barticle}
\endbibitem

\bibitem{winkels20183d}
\begin{botherref}
\oauthor{\bsnm{Winkels}, \binits{M.}},
\oauthor{\bsnm{Cohen}, \binits{T.S.}}:
{3D} {G-CNN}s for pulmonary nodule detection.
MIDL,
1--11
(2018)
\end{botherref}
\endbibitem

\bibitem{worrall2018cubenet}
\begin{botherref}
\oauthor{\bsnm{Worrall}, \binits{D.}},
\oauthor{\bsnm{Brostow}, \binits{G.}}:
{CubeNet}: Equivariance to {3D} rotation and translation.
ECCV 2018,
585--602
(2018)
\end{botherref}
\endbibitem

\bibitem{oyallon2015deep}
\begin{bchapter}
\bauthor{\bsnm{Oyallon}, \binits{E.}},
\bauthor{\bsnm{Mallat}, \binits{S.}}:
\bctitle{Deep roto-translation scattering for object classification}.
In: \bbtitle{Proceedings of the IEEE Conference on Computer Vision and Pattern
  Recognition},
pp. \bfpage{2865}--\blpage{2873}
(\byear{2015})
\end{bchapter}
\endbibitem

\bibitem{weiler2018learning}
\begin{bchapter}
\bauthor{\bsnm{Weiler}, \binits{M.}},
\bauthor{\bsnm{Hamprecht}, \binits{F.A.}},
\bauthor{\bsnm{Storath}, \binits{M.}}:
\bctitle{Learning steerable filters for rotation equivariant {CNNs}}.
In: \bbtitle{Proceedings of the IEEE Conference on Computer Vision and Pattern
  Recognition},
pp. \bfpage{849}--\blpage{858}
(\byear{2018})
\end{bchapter}
\endbibitem

\bibitem{bekkers2019bspline}
\begin{botherref}
\oauthor{\bsnm{Bekkers}, \binits{E.J.}}:
{B}-spline {CNN}s on {L}ie groups.
CoRR
\textbf{abs/1909.12057}
(2019)
{\href{https://arxiv.org/abs/1909.12057}{{1909.12057}}}
\end{botherref}
\endbibitem

\bibitem{finzi2020generalizing}
\begin{bchapter}
\bauthor{\bsnm{Finzi}, \binits{M.}},
\bauthor{\bsnm{Stanton}, \binits{S.}},
\bauthor{\bsnm{Izmailov}, \binits{P.}},
\bauthor{\bsnm{Wilson}, \binits{A.G.}}:
\bctitle{Generalizing convolutional neural networks for equivariance to {L}ie
  groups on arbitrary continuous data}.
In: \beditor{\bsnm{III}, \binits{H.D.}},
\beditor{\bsnm{Singh}, \binits{A.}} (eds.)
\bbtitle{Proceedings of the 37th International Conference on Machine Learning}.
\bsertitle{Proceedings of Machine Learning Research},
vol. \bseriesno{119},
pp. \bfpage{3165}--\blpage{3176}.
\bpublisher{PMLR},
\blocation{Virtual}
(\byear{2020}).
\burl{http://proceedings.mlr.press/v119/finzi20a.html}
\end{bchapter}
\endbibitem

\bibitem{cohen2019general}
\begin{botherref}
\oauthor{\bsnm{Cohen}, \binits{T.S.}},
\oauthor{\bsnm{Geiger}, \binits{M.}},
\oauthor{\bsnm{Weiler}, \binits{M.}}:
A general theory of equivariant {CNN}s on homogeneous spaces.
Advances in Neural Information Processing Systems
\textbf{32}
(2019)
\end{botherref}
\endbibitem

\bibitem{worrall2017harmonic}
\begin{bchapter}
\bauthor{\bsnm{Worrall}, \binits{D.E.}},
\bauthor{\bsnm{Garbin}, \binits{S.J.}},
\bauthor{\bsnm{Turmukhambetov}, \binits{D.}},
\bauthor{\bsnm{Brostow}, \binits{G.J.}}:
\bctitle{Harmonic networks: Deep translation and rotation equivariance}.
In: \bbtitle{Proceedings of the IEEE Conference on Computer Vision and Pattern
  Recognition},
pp. \bfpage{5028}--\blpage{5037}
(\byear{2017})
\end{bchapter}
\endbibitem

\bibitem{kondor2018generalization}
\begin{bchapter}
\bauthor{\bsnm{Kondor}, \binits{R.}},
\bauthor{\bsnm{Trivedi}, \binits{S.}}:
\bctitle{On the generalization of equivariance and convolution in neural
  networks to the action of compact groups}.
In: \beditor{\bsnm{Dy}, \binits{J.}},
\beditor{\bsnm{Krause}, \binits{A.}} (eds.)
\bbtitle{Proceedings of the 35th {International} {Conference} on {Machine}
  {Learning}}.
\bsertitle{Proceedings of {Machine} {Learning} {Research}},
vol. \bseriesno{80},
pp. \bfpage{2747}--\blpage{2755}.
\bpublisher{PMLR},
\blocation{Stockholmsmässan, Stockholm Sweden}
(\byear{2018}).
\burl{http://proceedings.mlr.press/v80/kondor18a.html}
\end{bchapter}
\endbibitem

\bibitem{esteves2018learning}
\begin{bchapter}
\bauthor{\bsnm{Esteves}, \binits{C.}},
\bauthor{\bsnm{Allen-Blanchette}, \binits{C.}},
\bauthor{\bsnm{Makadia}, \binits{A.}},
\bauthor{\bsnm{Daniilidis}, \binits{K.}}:
\bctitle{Learning {SO(3)} equivariant representations with spherical {CNNs}}.
In: \bbtitle{Proceedings of the European Conference on Computer Vision (ECCV)},
pp. \bfpage{52}--\blpage{68}
(\byear{2018})
\end{bchapter}
\endbibitem

\bibitem{weiler2019general}
\begin{bchapter}
\bauthor{\bsnm{Weiler}, \binits{M.}},
\bauthor{\bsnm{Cesa}, \binits{G.}}:
\bctitle{General {E(2)}-equivariant steerable {CNNs}}.
In: \bbtitle{Advances in Neural Information Processing Systems},
pp. \bfpage{14334}--\blpage{14345}
(\byear{2019})
\end{bchapter}
\endbibitem

\bibitem{paoletti2020rotation}
\begin{barticle}
\bauthor{\bsnm{Paoletti}, \binits{M.E.}},
\bauthor{\bsnm{Haut}, \binits{J.M.}},
\bauthor{\bsnm{Roy}, \binits{S.K.}},
\bauthor{\bsnm{Hendrix}, \binits{E.M.T.}}:
\batitle{Rotation equivariant convolutional neural networks for hyperspectral
  image classification}.
\bjtitle{IEEE Access}
\bvolume{8},
\bfpage{179575}--\blpage{179591}
(\byear{2020}).
\doiurl{10.1109/ACCESS.2020.3027776}
\end{barticle}
\endbibitem

\bibitem{weiler2021coordinate}
\begin{botherref}
\oauthor{\bsnm{Weiler}, \binits{M.}},
\oauthor{\bsnm{Forré}, \binits{P.}},
\oauthor{\bsnm{Verlinde}, \binits{E.}},
\oauthor{\bsnm{Welling}, \binits{M.}}:
Coordinate Independent Convolutional Networks -- Isometry and Gauge Equivariant
  Convolutions on {R}iemannian Manifolds.
arXiv
(2021).
\doiurl{10.48550/ARXIV.2106.06020}.
\url{https://arxiv.org/abs/2106.06020}
\end{botherref}
\endbibitem

\bibitem{cohen2019gauge}
\begin{bchapter}
\bauthor{\bsnm{Cohen}, \binits{T.S.}},
\bauthor{\bsnm{Weiler}, \binits{M.}},
\bauthor{\bsnm{Kicanaoglu}, \binits{B.}},
\bauthor{\bsnm{Welling}, \binits{M.}}:
\bctitle{Gauge equivariant convolutional networks and the icosahedral {CNN}}.
In: \beditor{\bsnm{Chaudhuri}, \binits{K.}},
\beditor{\bsnm{Salakhutdinov}, \binits{R.}} (eds.)
\bbtitle{Proceedings of the 36th International Conference on Machine Learning}.
\bsertitle{Proceedings of Machine Learning Research},
vol. \bseriesno{97},
pp. \bfpage{1321}--\blpage{1330}.
\bpublisher{PMLR},
\blocation{Long Beach, California}
(\byear{2019}).
\burl{https://proceedings.mlr.press/v97/cohen19d.html}
\end{bchapter}
\endbibitem

\bibitem{bogatskiy2020lorentz}
\begin{botherref}
\oauthor{\bsnm{Bogatskiy}, \binits{A.}},
\oauthor{\bsnm{Anderson}, \binits{B.}},
\oauthor{\bsnm{Offermann}, \binits{J.T.}},
\oauthor{\bsnm{Roussi}, \binits{M.}},
\oauthor{\bsnm{Miller}, \binits{D.W.}},
\oauthor{\bsnm{Kondor}, \binits{R.}}:
Lorentz Group Equivariant Neural Network for Particle Physics.
arXiv
(2020).
\doiurl{10.48550/ARXIV.2006.04780}.
\url{https://arxiv.org/abs/2006.04780}
\end{botherref}
\endbibitem

\bibitem{sifre2013rotation}
\begin{bchapter}
\bauthor{\bsnm{Sifre}, \binits{L.}},
\bauthor{\bsnm{Mallat}, \binits{S.}}:
\bctitle{Rotation, scaling and deformation invariant scattering for texture
  discrimination}.
In: \bbtitle{2013 IEEE Conference on Computer Vision and Pattern Recognition},
pp. \bfpage{1233}--\blpage{1240}
(\byear{2013}).
\doiurl{10.1109/CVPR.2013.163}
\end{bchapter}
\endbibitem

\bibitem{bekkers2018template}
\begin{barticle}
\bauthor{\bsnm{Bekkers}, \binits{E.J.}},
\bauthor{\bsnm{Loog}, \binits{M.}},
\bauthor{\bparticle{ter} \bsnm{Haar~Romeny}, \binits{B.M.}},
\bauthor{\bsnm{Duits}, \binits{R.}}:
\batitle{Template matching via densities on the roto-translation group}.
\bjtitle{IEEE Transactions on Pattern Analysis and Machine Intelligence}
\bvolume{40}(\bissue{2}),
\bfpage{452}--\blpage{466}
(\byear{2018}).
\doiurl{10.1109/TPAMI.2017.2652452}
\end{barticle}
\endbibitem

\bibitem{worrall2019deep}
\begin{bchapter}
\bauthor{\bsnm{Worrall}, \binits{D.}},
\bauthor{\bsnm{Welling}, \binits{M.}}:
\bctitle{Deep scale-spaces: Equivariance over scale}.
In: \beditor{\bsnm{Wallach}, \binits{H.}},
\beditor{\bsnm{Larochelle}, \binits{H.}},
\beditor{\bsnm{Beygelzimer}, \binits{A.}},
\beditor{\bparticle{d\textquotesingle} \bsnm{Alch\'{e}-Buc}, \binits{F.}},
\beditor{\bsnm{Fox}, \binits{E.}},
\beditor{\bsnm{Garnett}, \binits{R.}} (eds.)
\bbtitle{Advances in Neural Information Processing Systems},
vol. \bseriesno{32}.
\bpublisher{Curran Associates, Inc.},
\blocation{Red Hook, New York}
(\byear{2019}).
\burl{https://proceedings.neurips.cc/paper/2019/file/f04cd7399b2b0128970efb6d20b5c551-Paper.pdf}
\end{bchapter}
\endbibitem

\bibitem{satorras2021equivariant}
\begin{bchapter}
\bauthor{\bsnm{Satorras}, \binits{V.G.}},
\bauthor{\bsnm{Hoogeboom}, \binits{E.}},
\bauthor{\bsnm{Welling}, \binits{M.}}:
\bctitle{{E(n)} equivariant graph neural networks}.
In: \beditor{\bsnm{Meila}, \binits{M.}},
\beditor{\bsnm{Zhang}, \binits{T.}} (eds.)
\bbtitle{Proceedings of the 38th International Conference on Machine Learning}.
\bsertitle{Proceedings of Machine Learning Research},
vol. \bseriesno{139},
pp. \bfpage{9323}--\blpage{9332}.
\bpublisher{PMLR},
\blocation{Virtual}
(\byear{2021}).
\burl{https://proceedings.mlr.press/v139/satorras21a.html}
\end{bchapter}
\endbibitem

\bibitem{bronstein2021geometric}
\begin{botherref}
\oauthor{\bsnm{Bronstein}, \binits{M.M.}},
\oauthor{\bsnm{Bruna}, \binits{J.}},
\oauthor{\bsnm{Cohen}, \binits{T.S.}},
\oauthor{\bsnm{Veličković}, \binits{P.}}:
Geometric Deep Learning: Grids, Groups, Graphs, Geodesics, and Gauges.
arXiv
(2021).
\doiurl{10.48550/ARXIV.2104.13478}.
\url{https://arxiv.org/abs/2104.13478}
\end{botherref}
\endbibitem

\bibitem{smets2022pdebased}
\begin{barticle}
\bauthor{\bsnm{Smets}, \binits{B.M.N.}},
\bauthor{\bsnm{Portegies}, \binits{J.W.}},
\bauthor{\bsnm{Bekkers}, \binits{E.J.}},
\bauthor{\bsnm{Duits}, \binits{R.}}:
\batitle{{PDE}-based group equivariant convolutional neural networks}.
\bjtitle{Journal of Mathematical Imaging and Vision}
(\byear{2022}).
\doiurl{10.1007/s10851-022-01114-x}
\end{barticle}
\endbibitem

\bibitem{duits2010leftinvariant}
\begin{barticle}
\bauthor{\bsnm{Duits}, \binits{R.}},
\bauthor{\bsnm{Franken}, \binits{E.M.}}:
\batitle{Left-invariant parabolic evolution equations on ${SE}(2)$ and contour
  enhancement via invertible orientation scores, part~{I}: Linear
  left-invariant diffusion equations on ${SE}(2)$}.
\bjtitle{QAM-AMS}
\bvolume{68},
\bfpage{255}--\blpage{292}
(\byear{2010})
\end{barticle}
\endbibitem

\bibitem{duits2005perceptual}
\begin{botherref}
\oauthor{\bsnm{Duits}, \binits{R.}}:
Perceptual organization in image analysis: a mathematical approach based on
  scale, orientation and curvature.
PhD thesis,
Eindhoven University of Technology
(2005)
\end{botherref}
\endbibitem

\bibitem{duits2013morphological}
\begin{barticle}
\bauthor{\bsnm{Duits}, \binits{R.}},
\bauthor{\bsnm{Dela~Haije}, \binits{T.C.J.}},
\bauthor{\bsnm{Creusen}, \binits{E.}},
\bauthor{\bsnm{Ghosh}, \binits{A.}}:
\batitle{Morphological and linear scale spaces for fiber enhancement in
  {DW-MRI}}.
\bjtitle{Journal of Mathematical Imaging and Vision}
\bvolume{46}(\bissue{3}),
\bfpage{326}--\blpage{368}
(\byear{2013})
\end{barticle}
\endbibitem

\bibitem{duits2007scale}
\begin{bchapter}
\bauthor{\bsnm{Duits}, \binits{R.}},
\bauthor{\bsnm{Burgeth}, \binits{B.}}:
\bctitle{Scale spaces on {L}ie groups}.
In: \bbtitle{International Conference on Scale Space and Variational Methods in
  Computer Vision},
pp. \bfpage{300}--\blpage{312}
(\byear{2007}).
\bcomment{Springer}
\end{bchapter}
\endbibitem

\bibitem{franken2008enhancement}
\begin{botherref}
\oauthor{\bsnm{Franken}, \binits{E.M.}}:
Enhancement of crossing elongated structures in images.
PhD thesis,
Eindhoven University of Technology
(2008)
\end{botherref}
\endbibitem

\bibitem{bekkers2017retinal}
\begin{botherref}
\oauthor{\bsnm{Bekkers}, \binits{E.J.}}:
Retinal image analysis using sub-{R}iemannian geometry in {SE(2)}.
PhD thesis,
Eindhoven University of Technology
(2017)
\end{botherref}
\endbibitem

\bibitem{hubel1959receptive}
\begin{barticle}
\bauthor{\bsnm{Hubel}, \binits{D.H.}},
\bauthor{\bsnm{Wiesel}, \binits{T.N.}}:
\batitle{Receptive fields of single neurons in the cat's striate cortex}.
\bjtitle{Journal of Physiology}
\bvolume{148},
\bfpage{574}--\blpage{591}
(\byear{1959})
\end{barticle}
\endbibitem

\bibitem{bosking1997orientation}
\begin{barticle}
\bauthor{\bsnm{Bosking}, \binits{W.H.}},
\bauthor{\bsnm{Zhang}, \binits{Y.}},
\bauthor{\bsnm{Schofield}, \binits{B.}},
\bauthor{\bsnm{Fitzpatrick}, \binits{D.}}:
\batitle{Orientation selectivity and the arrangement of horizontal connections
  in tree shrew striate cortex}.
\bjtitle{The Journal of Neuroscience}
\bvolume{17}(\bissue{6}),
\bfpage{2112}--\blpage{2127}
(\byear{1997})
\end{barticle}
\endbibitem

\bibitem{petitot2003neurogeometry}
\begin{barticle}
\bauthor{\bsnm{Petitot}, \binits{J.}}:
\batitle{The neurogeometry of pinwheels as a sub-{R}iemannian contact
  structure}.
\bjtitle{Journal of Physiology - Paris}
\bvolume{97},
\bfpage{265}--\blpage{309}
(\byear{2003})
\end{barticle}
\endbibitem

\bibitem{citti2006cortical}
\begin{barticle}
\bauthor{\bsnm{Citti}, \binits{G.}},
\bauthor{\bsnm{Sarti}, \binits{A.}}:
\batitle{A cortical based model of perceptional completion in the
  roto-translation space}.
\bjtitle{Journal of Mathematical Imaging and Vision}
\bvolume{24}(\bissue{3}),
\bfpage{307}--\blpage{326}
(\byear{2006})
\end{barticle}
\endbibitem

\bibitem{field1993contour}
\begin{barticle}
\bauthor{\bsnm{Field}, \binits{D.J.}},
\bauthor{\bsnm{Hayes}, \binits{A.}},
\bauthor{\bsnm{Hess}, \binits{R.F.}}:
\batitle{Contour integration by the human visual system: Evidence for a local
  “association field”}.
\bjtitle{Vision Research}
\bvolume{33}(\bissue{2}),
\bfpage{173}--\blpage{193}
(\byear{1993}).
\doiurl{10.1016/0042-6989(93)90156-Q}
\end{barticle}
\endbibitem

\bibitem{baspinar2021cortical}
\begin{barticle}
\bauthor{\bsnm{Baspinar}, \binits{E.}},
\bauthor{\bsnm{Calatroni}, \binits{L.}},
\bauthor{\bsnm{Franceschi}, \binits{V.}},
\bauthor{\bsnm{Prandi}, \binits{D.}}:
\batitle{A cortical-inspired sub-{R}iemannian model for {P}oggendorff-type
  visual illusions}.
\bjtitle{Journal of Imaging}
\bvolume{7},
\bfpage{41}
(\byear{2021}).
\doiurl{10.3390/jimaging7030041}
\end{barticle}
\endbibitem

\bibitem{franceschiello2019geometrical}
\begin{barticle}
\bauthor{\bsnm{Franceschiello}, \binits{B.}},
\bauthor{\bsnm{Mashtakov}, \binits{A.}},
\bauthor{\bsnm{Citti}, \binits{G.}},
\bauthor{\bsnm{Sarti}, \binits{A.}}:
\batitle{Geometrical optical illusion via sub-{R}iemannian geodesics in the
  roto-translation group}.
\bjtitle{Differential Geometry and its Applications}
\bvolume{65},
\bfpage{55}--\blpage{77}
(\byear{2019}).
\doiurl{10.1016/j.difgeo.2019.03.007}
\end{barticle}
\endbibitem

\bibitem{duits2013association}
\begin{barticle}
\bauthor{\bsnm{Duits}, \binits{R.}},
\bauthor{\bsnm{Boscain}, \binits{U.}},
\bauthor{\bsnm{Rossi}, \binits{F.}},
\bauthor{\bsnm{Sachkov}, \binits{Y.L.}}:
\batitle{Association fields via cuspless sub-{R}iemannian geodesics in
  {SE(2)}}.
\bjtitle{Journal of Mathematical Imaging and Vision}
\bvolume{49}(\bissue{2}),
\bfpage{384}--\blpage{417}
(\byear{2014}).
\doiurl{10.1007/s10851-013-0475-y}
\end{barticle}
\endbibitem

\bibitem{sachkov2011cutlocus}
\begin{barticle}
\bauthor{\bsnm{Sachkov}, \binits{Y.L.}}:
\batitle{Cut locus and optimal synthesis in the sub-{R}iemannian problem on the
  group of motions of a plane}.
\bjtitle{ESAIM: Control, Optimization and Calculus of Variations}
\bvolume{17},
\bfpage{293}--\blpage{321}
(\byear{2011})
\end{barticle}
\endbibitem

\bibitem{moiseev2010maxwell}
\begin{barticle}
\bauthor{\bsnm{Moiseev}, \binits{I.}},
\bauthor{\bsnm{Sachkov}, \binits{Y.L.}}:
\batitle{Maxwell strata in sub-{R}iemannian problem on the group of motions of
  a plane}.
\bjtitle{ESAIM: Control, Optimisation and Calculus of Variation}
\bvolume{16}(\bissue{2}),
\bfpage{380}--\blpage{399}
(\byear{2010}).
\doiurl{10.1051/cocv/2009004}
\end{barticle}
\endbibitem

\bibitem{duits2018optimal}
\begin{botherref}
\oauthor{\bsnm{Duits}, \binits{R.}},
\oauthor{\bsnm{Meesters}, \binits{S.P.L.}},
\oauthor{\bsnm{Mirebeau}, \binits{J.-M.}},
\oauthor{\bsnm{Portegies}, \binits{J.M.}}:
Optimal paths for variants of the {2D} and {3D} {Reeds}--{Shepp} car with
  applications in image analysis.
Journal of Mathematical Imaging and Vision,
1--33
(2018)
\end{botherref}
\endbibitem

\bibitem{petitot2017elements}
\begin{bbook}
\bauthor{\bsnm{Petitot}, \binits{J.}}:
\bbtitle{Elements of Neurogeometry}.
\bsertitle{Lecture Notes in Morphogenesis}.
\bpublisher{Springer},
\blocation{London}
(\byear{2017}).
\doiurl{10.1007/978-3-319-65591-8}
\end{bbook}
\endbibitem

\bibitem{bekkers2015pde}
\begin{barticle}
\bauthor{\bsnm{Bekkers}, \binits{E.J.}},
\bauthor{\bsnm{Duits}, \binits{R.}},
\bauthor{\bsnm{Mashtakov}, \binits{A.}},
\bauthor{\bsnm{Sanguinetti}, \binits{G.R.}}:
\batitle{A {PDE} approach to data-driven sub-{R}iemannian geodesics in
  {SE(2)}}.
\bjtitle{SIAM Journal on Imaging Sciences}
\bvolume{8}(\bissue{4}),
\bfpage{2740}--\blpage{2770}
(\byear{2015})
\end{barticle}
\endbibitem

\bibitem{bekkers2018nilpotent}
\begin{barticle}
\bauthor{\bsnm{Bekkers}, \binits{E.J.}},
\bauthor{\bsnm{Chen}, \binits{D.}},
\bauthor{\bsnm{Portegies}, \binits{J.M.}}:
\batitle{Nilpotent approximations of sub-{R}iemannian distances for fast
  perceptual grouping of blood vessels in {2D} and {3D}}.
\bjtitle{Journal of Mathematical Imaging and Vision}
\bvolume{60}(\bissue{6}),
\bfpage{882}--\blpage{899}
(\byear{2018}).
\doiurl{10.1007/s10851-018-0787-z}
\end{barticle}
\endbibitem

\bibitem{chirikjian2000engineering}
\begin{bbook}
\bauthor{\bsnm{Chirikjian}, \binits{G.S.}},
\bauthor{\bsnm{Kyatkin}, \binits{A.B.}}:
\bbtitle{Engineering Applications of Noncommutative Harmonic Analysis: With
  Emphasis on Rotation and Motion Groups}.
\bpublisher{CRC Press},
\blocation{Boca Raton, Florida}
(\byear{2000}).
\burl{https://books.google.nl/books?id = nfbLBQAAQBAJ}
\end{bbook}
\endbibitem

\bibitem{schmidt2016morphological}
\begin{barticle}
\bauthor{\bsnm{Schmidt}, \binits{M.}},
\bauthor{\bsnm{Weickert}, \binits{J.}}:
\batitle{Morphological counterparts of linear shift-invariant scale-spaces}.
\bjtitle{Journal of Mathematical Imaging and Vision}
\bvolume{56}(\bissue{2}),
\bfpage{352}--\blpage{366}
(\byear{2016})
\end{barticle}
\endbibitem

\bibitem{boomgaard1994morphological}
\begin{barticle}
\bauthor{\bparticle{van~den} \bsnm{Boomgaard}, \binits{R.}},
\bauthor{\bsnm{Smeulders}, \binits{A.}}:
\batitle{The morphological structure of images: the differential equations of
  morphological scale-space}.
\bjtitle{IEEE Transactions on Pattern Analysis and Machine Intelligence}
\bvolume{16}(\bissue{11}),
\bfpage{1101}--\blpage{1113}
(\byear{1994}).
\doiurl{10.1109/34.334389}
\end{barticle}
\endbibitem

\bibitem{evans2010partial}
\begin{bbook}
\bauthor{\bsnm{Evans}, \binits{L.C.}}:
\bbtitle{Partial Differential Equations}
vol. \bseriesno{19}.
\bpublisher{American Mathematical Society},
\blocation{Providence, Rhode Island}
(\byear{2010})
\end{bbook}
\endbibitem

\bibitem{diop2021extension}
\begin{bchapter}
\bauthor{\bsnm{Diop}, \binits{E.H.S.}},
\bauthor{\bsnm{Mbengue}, \binits{A.}},
\bauthor{\bsnm{Manga}, \binits{B.}},
\bauthor{\bsnm{Seck}, \binits{D.}}:
\bctitle{Extension of mathematical morphology in {R}iemannian spaces}.
In: \bbtitle{Scale Space and Variational Methods in Computer Vision},
pp. \bfpage{100}--\blpage{111}.
\bpublisher{Springer},
\blocation{Cham}
(\byear{2021})
\end{bchapter}
\endbibitem

\bibitem{fathi2007weakkam}
\begin{barticle}
\bauthor{\bsnm{Fathi}, \binits{A.}},
\bauthor{\bsnm{Maderna}, \binits{E.}}:
\batitle{Weak {KAM} theorem on non compact manifolds}.
\bjtitle{Nonlinear Differential Equations and Applications NoDEA}
\bvolume{14}(\bissue{1-2}),
\bfpage{1}--\blpage{27}
(\byear{2007}).
\doiurl{10.1007/s00030-007-2047-6}
\end{barticle}
\endbibitem

\bibitem{azagra2005nonsmooth}
\begin{barticle}
\bauthor{\bsnm{Azagra}, \binits{D.}},
\bauthor{\bsnm{Ferrera}, \binits{J.}},
\bauthor{\bsnm{L{\'o}pez-Mesas}, \binits{F.}}:
\batitle{Nonsmooth analysis and {H}amilton--{J}acobi equations on {R}iemannian
  manifolds}.
\bjtitle{Journal of Functional Analysis}
\bvolume{220}(\bissue{2}),
\bfpage{304}--\blpage{361}
(\byear{2005})
\end{barticle}
\endbibitem

\bibitem{lupi2021kernel}
\begin{botherref}
\oauthor{\bsnm{Lupi}, \binits{G.}}:
Kernel approximations in lie groups and application to group-invariant {CNN}.
Master thesis,
University of Bologna
(2021)
\end{botherref}
\endbibitem

\bibitem{terelst1998weighted}
\begin{barticle}
\bauthor{\bsnm{{ter Elst}}, \binits{A.F.M.}},
\bauthor{\bsnm{Robinson}, \binits{D.W.}}:
\batitle{Weighted subcoercive operators on {Lie} groups}.
\bjtitle{Journal of Functional Analysis}
\bvolume{157}(\bissue{1}),
\bfpage{88}--\blpage{163}
(\byear{1998}).
\doiurl{10.1006/jfan.1998.3259}
\end{barticle}
\endbibitem

\bibitem{medioni2021tensor}
\begin{bbook}
\bauthor{\bsnm{Mordohai}, \binits{P.}},
\bauthor{\bsnm{Medioni}, \binits{G.}}:
\bbtitle{Tensor Voting: A Perceptual Organization Approach to Computer Vision
  and Machine Learning}
vol. \bseriesno{2},
(\byear{2006}).
\doiurl{10.2200/S00049ED1V01Y200609IVM008}
\end{bbook}
\endbibitem

\bibitem{sanchez2019segmentation}
\begin{botherref}
\oauthor{\bsnm{Cervantes-Sanchez}, \binits{F.}},
\oauthor{\bsnm{Cruz-Aceves}, \binits{I.}},
\oauthor{\bsnm{Hernandez-Aguirre}, \binits{A.}},
\oauthor{\bsnm{Hernandez-Gonzalez}, \binits{M.A.}},
\oauthor{\bsnm{Solorio-Meza}, \binits{S.E.}}:
Automatic segmentation of coronary arteries in {X}-ray angiograms using
  multiscale analysis and artificial neural networks.
Applied Sciences
\textbf{9}(24)
(2019).
\doiurl{10.3390/app9245507}
\end{botherref}
\endbibitem

\bibitem{duits2021equivariant}
\begin{barticle}
\bauthor{\bsnm{Duits}, \binits{R.}},
\bauthor{\bsnm{Smets}, \binits{B.M.N.}},
\bauthor{\bsnm{Bekkers}, \binits{E.J.}},
\bauthor{\bsnm{Portegies}, \binits{J.W.}}:
\batitle{Equivariant deep learning via morphological and linear scale space
  {PDEs} on the space of positions and orientations}.
\bjtitle{LNCS}
\bvolume{12679},
\bfpage{27}--\blpage{39}
(\byear{2021})
\end{barticle}
\endbibitem

\bibitem{baspinar2018geometric}
\begin{barticle}
\bauthor{\bsnm{Baspinar}, \binits{E.}},
\bauthor{\bsnm{Citti}, \binits{G.}},
\bauthor{\bsnm{Sarti}, \binits{A.}}:
\batitle{A geometric model of multi-scale orientation preference maps via
  {G}abor functions}.
\bjtitle{Journal of Mathematical Imaging and Vision}
\bvolume{60}(\bissue{6}),
\bfpage{900}--\blpage{912}
(\byear{2018})
\end{barticle}
\endbibitem

\end{thebibliography}
